\documentclass[12pt,letterpaper]{article}

\def\arxivmode{1} %
\def\showauthornotes{0}
\def\showtableofcontents{1}
\def\showkeys{0}
\def\showdraftbox{0}
\def\showcolorlinks{1}
\def\usemicrotype{1}
\def\showfixme{0}

\usepackage{etex}

\usepackage[l2tabu, orthodox]{nag}

\usepackage{xspace,enumerate}

\usepackage[dvipsnames]{xcolor}

\usepackage[T1]{fontenc}
\usepackage[full]{textcomp}

\usepackage[american]{babel}

\usepackage{mathtools}

\usepackage{amsthm}

\newtheorem{theorem}{Theorem}[section]
\newtheorem*{theorem*}{Theorem}

\newtheorem*{proposition*}{Proposition}
\newtheorem{lemma}[theorem]{Lemma}
\newtheorem*{lemma*}{Lemma}
\newtheorem{corollary}[theorem]{Corollary}
\newtheorem*{conjecture*}{Conjecture}

\newtheorem*{fact*}{Fact}

\newtheorem*{hypothesis*}{Hypothesis}

\theoremstyle{definition}
\newtheorem{definition}[theorem]{Definition}

\newtheorem{example}[theorem]{Example}

\newtheorem{problem}[theorem]{Problem}

\theoremstyle{remark}

\newtheorem*{claim*}{Claim}
\newtheorem{remark}[theorem]{Remark}
\newtheorem*{remark*}{Remark}

\newtheorem*{observation*}{Observation}

\usepackage[
letterpaper,
top=0.7in,
bottom=0.9in,
left=1in,
right=1in]{geometry}

\ifnum\arxivmode=0
\usepackage{newpxtext} %
\usepackage{textcomp} %
\usepackage[varg,bigdelims]{newpxmath}
\usepackage[scr=rsfso]{mathalfa}%
\usepackage{bm} %
\let\mathbb\varmathbb
\fi

\ifnum\arxivmode=1
\usepackage[varg]{pxfonts} %
\usepackage{textcomp} %
\usepackage[scr=rsfso]{mathalfa}%
\usepackage{bm} %
\fi

\ifnum\showkeys=1
\usepackage[color]{showkeys}
\fi

\ifnum\showcolorlinks=1
\usepackage[
pagebackref,
colorlinks=true,
urlcolor=blue,
linkcolor=blue,
citecolor=OliveGreen,
]{hyperref}
\fi

\ifnum\showcolorlinks=0
\usepackage[
pagebackref,
colorlinks=false,
pdfborder={0 0 0}
]{hyperref}
\fi

\usepackage{prettyref}

\newcommand{\savehyperref}[2]{\texorpdfstring{\hyperref[#1]{#2}}{#2}}

\newrefformat{eq}{\savehyperref{#1}{\textup{(\ref*{#1})}}}
\newrefformat{lem}{\savehyperref{#1}{Lemma~\ref*{#1}}}
\newrefformat{def}{\savehyperref{#1}{Definition~\ref*{#1}}}
\newrefformat{thm}{\savehyperref{#1}{Theorem~\ref*{#1}}}
\newrefformat{cor}{\savehyperref{#1}{Corollary~\ref*{#1}}}
\newrefformat{cha}{\savehyperref{#1}{Chapter~\ref*{#1}}}
\newrefformat{sec}{\savehyperref{#1}{Section~\ref*{#1}}}
\newrefformat{app}{\savehyperref{#1}{Appendix~\ref*{#1}}}
\newrefformat{tab}{\savehyperref{#1}{Table~\ref*{#1}}}
\newrefformat{fig}{\savehyperref{#1}{Figure~\ref*{#1}}}
\newrefformat{hyp}{\savehyperref{#1}{Hypothesis~\ref*{#1}}}
\newrefformat{alg}{\savehyperref{#1}{Algorithm~\ref*{#1}}}
\newrefformat{rem}{\savehyperref{#1}{Remark~\ref*{#1}}}
\newrefformat{item}{\savehyperref{#1}{Item~\ref*{#1}}}
\newrefformat{step}{\savehyperref{#1}{step~\ref*{#1}}}
\newrefformat{conj}{\savehyperref{#1}{Conjecture~\ref*{#1}}}
\newrefformat{fact}{\savehyperref{#1}{Fact~\ref*{#1}}}
\newrefformat{prop}{\savehyperref{#1}{Proposition~\ref*{#1}}}
\newrefformat{prob}{\savehyperref{#1}{Problem~\ref*{#1}}}
\newrefformat{claim}{\savehyperref{#1}{Claim~\ref*{#1}}}
\newrefformat{relax}{\savehyperref{#1}{Relaxation~\ref*{#1}}}
\newrefformat{red}{\savehyperref{#1}{Reduction~\ref*{#1}}}
\newrefformat{part}{\savehyperref{#1}{Part~\ref*{#1}}}

\newcommand{\Sref}[1]{\hyperref[#1]{\S\ref*{#1}}}

\usepackage{nicefrac}

\ifnum\usemicrotype=1
\usepackage{microtype}
\fi

\ifnum\showauthornotes=1
\newcommand{\Authornote}[2]{{\sffamily\small\color{red}{[#1: #2]}}}
\newcommand{\Authornotecolored}[3]{{\sffamily\small\color{#1}{[#2: #3]}}}
\newcommand{\Authorcomment}[2]{{\sffamily\small\color{gray}{[#1: #2]}}}
\newcommand{\Authorstartcomment}[1]{\sffamily\small\color{gray}[#1: }

\newcommand{\Authorfnote}[2]{\footnote{\color{red}{#1: #2}}}
\newcommand{\Authorfixme}[1]{\Authornote{#1}{\textbf{??}}}
\newcommand{\Authormarginmark}[1]{\marginpar{\textcolor{red}{\fbox{\Large #1:!}}}}
\else
\newcommand{\Authornote}[2]{}
\newcommand{\Authornotecolored}[3]{}
\newcommand{\Authorcomment}[2]{}
\newcommand{\Authorstartcomment}[1]{}

\newcommand{\Authorfnote}[2]{}
\newcommand{\Authorfixme}[1]{}
\newcommand{\Authormarginmark}[1]{}
\fi

\newcommand{\Dnote}{\Authornote{D}}

\newcommand{\Snote}{\Authornote{S}}

\ifnum\showfixme=0

\fi

\usepackage{boxedminipage}

\newcommand{\Paren}[1]{\left(#1\right)}

\newcommand{\Brac}[1]{\left[#1\right]}

\newcommand{\abs}[1]{\lvert#1\rvert}

\newcommand{\Norm}[1]{\left\lVert#1\right\rVert}

\newcommand{\iprod}[1]{\langle#1\rangle}

\newcommand{\Esymb}{\mathbb{E}}
\newcommand{\Psymb}{\mathbb{P}}

\DeclareMathOperator*{\E}{\Esymb}

\DeclareMathOperator*{\ProbOp}{\Psymb}

\renewcommand{\Pr}{\ProbOp}

\newcommand{\tensor}{\otimes}

\newcommand{\textparen}[1]{\text{(#1)}}

\ifx\because\undefined
\newcommand{\because}[1]{\textparen{because #1}}
\else
\renewcommand{\because}[1]{\textparen{because #1}}
\fi

\newcommand{\lmin}{\lambda_{\min}}

\newcommand{\vbig}{\vphantom{\bigoplus}}

\newcommand{\defeq}{\stackrel{\mathrm{def}}=}

\newcommand{\seteq}{\mathrel{\mathop:}=}

\newcommand{\from}{\colon}

\newcommand{\mper}{\,.}
\newcommand{\mcom}{\,,}

\newcommand\bdot\bullet

\ifx\mathds\undefined %

\else

\fi

\DeclareMathOperator{\Tr}{Tr}

\DeclareMathOperator{\poly}{poly}

\DeclareMathOperator{\polylog}{polylog}

\newcommand\naive{na\"{\i}ve\xspace}

\newcommand\naively{na\"{\i}vely\xspace}

\newcommand{\N}{\mathbb N}
\newcommand{\R}{\mathbb R}

\newcommand{\cA}{\mathcal A}

\newcommand{\cC}{\mathcal C}
\newcommand{\cD}{\mathcal D}

\newcommand{\cK}{\mathcal K}
\newcommand{\cL}{\mathcal L}
\newcommand{\cM}{\mathcal M}
\newcommand{\cN}{\mathcal N}
\newcommand{\cO}{\mathcal O}

\newcommand{\cS}{\mathcal S}

\newcommand{\bbN}{\mathbb N}

\renewcommand{\leq}{\leqslant}
\renewcommand{\le}{\leqslant}
\renewcommand{\geq}{\geqslant}
\renewcommand{\ge}{\geqslant}

\ifnum\showdraftbox=1
\newcommand{\draftbox}{\begin{center}
  \fbox{%
    \begin{minipage}{2in}%
      \begin{center}%
          \Large\textsc{Working Draft}\\%
        Please do not distribute%
      \end{center}%
    \end{minipage}%
  }%
\end{center}
\vspace{0.2cm}}
\else
\newcommand{\draftbox}{}
\fi

\let\epsilon=\varepsilon

\numberwithin{equation}{section}

\newcommand\MYcurrentlabel{xxx}

\newcommand{\MYstore}[2]{%
  \global\expandafter \def \csname MYMEMORY #1 \endcsname{#2}%
}

\newcommand{\MYload}[1]{%
  \csname MYMEMORY #1 \endcsname%
}

\newcommand{\MYnewlabel}[1]{%
  \renewcommand\MYcurrentlabel{#1}%
  \MYoldlabel{#1}%
}

\newcommand{\MYdummylabel}[1]{}

\newcommand{\torestate}[1]{%
  \let\MYoldlabel\label%
  \let\label\MYnewlabel%
  #1%
  \MYstore{\MYcurrentlabel}{#1}%
  \let\label\MYoldlabel%
}

\newcommand{\restatetheorem}[1]{%
  \let\MYoldlabel\label
  \let\label\MYdummylabel
  \begin{theorem*}[Restatement of \prettyref{#1}]
    \MYload{#1}
  \end{theorem*}
  \let\label\MYoldlabel
}

\newcommand{\restatelemma}[1]{%
  \let\MYoldlabel\label
  \let\label\MYdummylabel
  \begin{lemma*}[Restatement of \prettyref{#1}]
    \MYload{#1}
  \end{lemma*}
  \let\label\MYoldlabel
}

\newcommand{\restateprop}[1]{%
  \let\MYoldlabel\label
  \let\label\MYdummylabel
  \begin{proposition*}[Restatement of \prettyref{#1}]
    \MYload{#1}
  \end{proposition*}
  \let\label\MYoldlabel
}

\newcommand{\restatefact}[1]{%
  \let\MYoldlabel\label
  \let\label\MYdummylabel
  \begin{fact*}[Restatement of \prettyref{#1}]
    \MYload{#1}
  \end{fact*}
  \let\label\MYoldlabel
}

\newcommand{\restate}[1]{%
  \let\MYoldlabel\label
  \let\label\MYdummylabel
  \MYload{#1}
  \let\label\MYoldlabel
}

\newcommand{\addreferencesection}{
  \phantomsection
  \addcontentsline{toc}{section}{References}
}

\newcommand{\e}{\epsilon}

\let\origparagraph\paragraph
\renewcommand{\paragraph}[1]{\origparagraph{#1.}}

\allowdisplaybreaks

\usepackage{paralist}

\usepackage{comment}

\let\citet\cite

\theoremstyle{definition}
\newtheorem{algo}[theorem]{Algorithm}

\usepackage{relsize}
\usepackage[font=footnotesize]{caption}

\let\pref=\prettyref

\newcommand*{\ot}{\otimes}

\newcommand{\symId}{\Id^{\mathrm{sym}}}

\newcommand{\bT}{\mathbf{T}}
\newcommand{\bA}{\mathbf{A}}

\newcommand{\Id}{\mathop{\mathrm{Id}}\!\mathinner{}}

\newcommand{\Pu}{P_{u^{\perp}}}

\DeclareMathOperator{\pE}{\tilde {\mathbb E}}
\DeclareMathOperator{\pD}{\tilde {\mathbb D}}

\DeclareMathOperator{\tuples}{\#tuples}
\let\lmin\relax %
\DeclareMathOperator{\lmin}{\lambda_{min}}
\let\cL\relax
\DeclareMathOperator{\cL}{\mathcal L}
\newcommand{\bB}{{\mathbf B}}
\newcommand{\bS}{{\mathbf S}}
\newcommand{\bU}{{\mathbf U}}

\DeclareMathOperator{\Sym}{Sym}

\DeclareUrlCommand\email{}

\newcommand{\tO}{\tilde O}
\newcommand{\tOmega}{\tilde \Omega}

\newlength{\ninv}

\title{Tensor principal component analysis via sum-of-squares proofs}

\author{%
Samuel B. Hopkins\thanks{Department of Computer Science, Cornell University.
\protect\email{samhop@cs.cornell.edu}, \protect\email{jshi@cs.cornell.edu}, \protect\email{dsteurer@cs.cornell.edu}.
Please direct all communication to D.S.}
\and Jonathan Shi\footnotemark[1]
\and David Steurer\footnotemark[1]
}

\begin{document}

\maketitle
\draftbox
\thispagestyle{empty}

\begin{abstract}

We study a statistical model for the \emph{tensor principal component analysis problem} introduced by Montanari and Richard:
Given a order-$3$ tensor $T$ of the form $T = \tau \cdot v_0^{\otimes 3} + A$, where $\tau \geq 0$ is a signal-to-noise ratio, $v_0$ is a unit vector, and $A$ is a random noise tensor, the goal is to recover the planted vector $v_0$.
For the case that $A$ has iid standard Gaussian entries, we give an efficient algorithm to recover $v_0$ whenever $\tau \geq \omega(n^{3/4} \log(n)^{1/4})$, and certify that the recovered vector is close to a maximum likelihood estimator, all with high probability over the random choice of $A$.
The previous best algorithms with provable guarantees required $\tau \geq \Omega(n)$.

In the regime $\tau \leq o(n)$, natural tensor-unfolding-based spectral relaxations for the underlying optimization problem break down (in the sense that their integrality gap is large).
To go beyond this barrier, we use convex relaxations based on the sum-of-squares method.
Our recovery algorithm proceeds by rounding a degree-$4$ sum-of-squares relaxations of the maximum-likelihood-estimation problem for the statistical model.
To complement our algorithmic results, we show that degree-$4$ sum-of-squares relaxations break down for $\tau \leq O(n^{3/4}/\log(n)^{1/4})$, which demonstrates that improving our current guarantees (by more than logarithmic factors) would require new techniques or might even be intractable.

Finally, we show how to exploit additional problem structure in order to solve our sum-of-squares relaxations, up to some approximation, very efficiently.
Our fastest algorithm runs in nearly-linear time using shifted (matrix) power iteration and has similar guarantees as above.
The analysis of this algorithm also confirms a variant of a conjecture of Montanari and Richard about singular vectors of tensor unfoldings.

\end{abstract}


\clearpage

\setcounter{tocdepth}{1}

\ifnum\showtableofcontents=1
{
\tableofcontents
\thispagestyle{empty}
 }
\fi

\clearpage

\setcounter{page}{1}

\section{Introduction}
\label{sec:introduction}

\newcommand{\pca}{\textsc{pca}\xspace}
\newcommand{\csp}{\textsc{csp}\xspace}
\newcommand{\spca}{\textsc{sparse pca}\xspace}
\newcommand{\tpca}{\textsc{tensor pca}\xspace}

\emph{Principal component analysis} (\pca), the process of identifying a direction of largest possible variance from a matrix of pairwise correlations, is among the most basic tools for data analysis in a wide range of disciplines.
In recent years, variants of \pca have been proposed that promise to give better statistical guarantees for many applications.
These variants include restricting directions to the nonnegative orthant (nonnegative matrix factorization) or to directions that are sparse linear combinations of a fixed basis (\spca).
Often we have access to not only pairwise but also higher-order correlations.
In this case, an analog of \pca is to find a direction with largest possible third moment or other higher-order moment (higher-order \pca or \tpca).

All of these variants of \pca share that the underlying optimization problem is NP-hard for general instances (often even if we allow approximation), whereas vanilla \pca boils down to an efficient eigenvector computation for the input matrix.
However, these hardness result are not predictive in statistical settings where inputs are drawn from particular families of distributions.
Here efficient algorithm can often achieve much stronger guarantees than for general instances.
Understanding the power and limitations of efficient algorithms for statistical models of NP-hard optimization problems is typically very challenging: it is not clear what kind of algorithms can exploit the additional structure afforded by statistical instances, but, at the same time, there are very few tools for reasoning about the computational complexity of statistical / average-case problems.
(See \cite{DBLP:conf/colt/BerthetR13} and \cite{DBLP:conf/innovations/BarakKS13} for discussions about the computational complexity of statistical models for \spca and random constraint satisfaction problems.)

We study a statistical model for the \emph{tensor principal component analysis problem} introduced by \citet{richard2014tensorpca} through the lens of a meta-algorithm called the sum-of-squares method, based on semidefinite programming.
This method can capture a wide range of algorithmic techniques including linear programming and spectral algorithms.
We show that this method can exploit the structure of statistical \tpca instances in non-trivial ways and achieves guarantees that improve over the previous ones.
On the other hand, we show that those guarantees are nearly tight if we restrict the complexity of the sum-of-squares meta-algorithm at a particular level.
This result rules out better guarantees for a fairly wide range of potential algorithms.
Finally, we develop techniques to turn algorithms based on the sum-of-squares meta-algorithm into algorithms that are truly efficient (and even easy to implement).

\medskip
\noindent
Montanari and Richard propose the following statistical model\footnote{Montanari and Richard use a different normalization for the signal-to-noise ratio.
Using their notation, $\beta\approx \tau/\sqrt n$.} for \tpca.

\begin{problem}[Spiked Tensor Model for \tpca, Asymmetric]
\label{prob:spiked-tensor}
  Given an input tensor $\bT = \tau \cdot v^{\tensor 3} + \bA$, where $v \in \R^n$ is an arbitrary unit vector, $\tau \geq 0$ is the signal-to-noise ratio, and $\bA$ is a random noise tensor with iid standard Gaussian entries, recover the signal $v$ approximately.
\end{problem}

Montanari and Richard show that when $\tau \leq o(\sqrt n)$ \pref{prob:spiked-tensor} becomes information-theoretically unsolvable, while for $\tau \geq \omega(\sqrt n)$ the maximum likelihood estimator (MLE) recovers $v'$ with $\iprod{v,v'} \geq 1 - o(1)$.

The maximum-likelihood-estimator (MLE) problem for \pref{prob:spiked-tensor} is an instance of the following meta-problem for $k=3$ and $f\from x \mapsto \sum_{ijk} \bT_{ijk} x_i x_j x_k$ \cite{richard2014tensorpca}.
\begin{problem}
  \label{prob:homog-poly-opt}
  Given a homogeneous, degree-$k$ function $f$ on $\R^n$, find a unit vector $v\in \R^n$ so as to maximize $f(v)$ approximately.
\end{problem}
For $k=2$, this problem is just an eigenvector computation.
Already for $k=3$, it is NP-hard.
Our algorithms proceed by relaxing \pref{prob:homog-poly-opt} to a convex problem.
The latter can be solved either exactly or approximately (as will be the case of our faster algorithms).
Under the Gaussian assumption on the noise in \pref{prob:spiked-tensor}, we show that for $\tau \geq \omega(n^{3/4} \log(n)^{1/4})$ the relaxation does not substantially change the global optimum.

\paragraph{Noise Symmetry}
Montanari and Richard actually consider two variants of this model.
The first we have already described.
In the second, the noise is symmetrized, (to match the symmetry of potential signal tensors $v^{\otimes 3}$).
\begin{problem}[Spiked Tensor Model for \tpca, Symmetric]
\label{prob:spiked-tensor-symm}
  Given an input tensor $\bT = \tau \cdot v^{\tensor 3} + \bA$, where $v \in \R^n$ is an arbitrary unit vector, $\tau \geq 0$ is the signal-to-noise ratio, and $\bA$ is a random \emph{symmetric} noise tensor---that is, $A_{ijk} = A_{\pi(i) \pi(j) \pi(k)}$ for any permutation $\pi$---with otherwise iid standard Gaussian entries, recover the signal $v$ approximately.
\end{problem}

It turns out that for our algorithms based on the sum-of-squares method, this kind of symmetrization is already built-in.
Hence there is no difference between \pref{prob:spiked-tensor} and \pref{prob:spiked-tensor-symm} for those algorithms.
For our faster algorithms, such symmetrization is not built in.
Nonetheless, we show that a variant of our nearly-linear-time algorithm for \pref{prob:spiked-tensor} also solves \pref{prob:spiked-tensor-symm} with matching guarantees.

\subsection{Results}
\label{sec:results}

\paragraph{Sum-of-squares relaxation}

We consider the degree-$4$ sum-of-squares relaxation for the MLE problem.
(See \pref{sec:techniques} for a brief discussion about sum-of-squares.
All necessary definitions are in \pref{sec:preliminaries}.
See \cite{DBLP:journals/corr/BarakS14} for more detailed discussion.)
Note that the planted vector $v$ has objective value $(1-o(1))\tau$ for the MLE problem with high probability (assuming $\tau=\Omega(\sqrt n)$ which will always be the case for us).

\begin{theorem}
  \label{thm:sdp-main-informal}
  There exists a polynomial-time algorithm based on the degree-$4$ sum-of-squares relaxation for the MLE problem that given an instance of \pref{prob:spiked-tensor} or \pref{prob:spiked-tensor-symm} with $\tau \geq n^{3/4} (\log n)^{1/4} / \epsilon$ outputs a unit vector $v'$ with $\langle  v,v' \rangle\ge 1- O(\epsilon)$ with probability $1 - O(n^{-10})$ over the randomness in the input.
  Furthermore, the algorithm works by rounding any solution to the relaxation with objective value at least $(1-o(1))\tau$.
  Finally, the algorithm also certifies that all unit vectors bounded away from $v'$ have objective value significantly smaller than $\tau$ for the MLE problem \pref{prob:homog-poly-opt}.
\end{theorem}
We complement the above algorithmic result by the following lower bound.
\begin{theorem}[Informal Version]
  \label{thm:lb-intro-informal}
  There is $\tau : \N \rightarrow \R$ with $\tau \leq O(n^{3/4}/ \log(n)^{1/4})$
so that when
  $\bT$ is an instance of \pref{prob:spiked-tensor} with signal-to-noise ratio $\tau$,
  with probability $1 - O(n^{-10})$, there exists a solution to the degree-$4$ sum-of-squares relaxation for the MLE problem with objective value at least $\tau$ that does not depend on the planted vector $v$.
  In particular, no algorithm can reliably recover from this solution a vector $v'$ that is significantly correlated with $v$.
\end{theorem}

\paragraph{Faster algorithms}

We interpret a tensor-unfolding algorithm studied by Montanari and Richard as a spectral relaxation of the degree-4 sum-of-squares program for the MLE problem.
This interpretation leads to an analysis that gives better guarantees in terms of signal-to-noise ratio $\tau$ and also informs a more efficient implementation based on shifted matrix power iteration.\Snote{}

\begin{theorem}
  \label{thm:linear-intro}
  There exists an algorithm with running time $\tO(n^3)$, which is linear in the size of the input, that given an instance of \pref{prob:spiked-tensor} or \pref{prob:spiked-tensor-symm} with $\tau \geq  n^{3/4} / \epsilon$ outputs with probability $1-O(n^{-10})$ a unit vector $v'$ with $\langle  v,v' \rangle\ge 1-O(\epsilon)$.
\end{theorem}

We remark that unlike the previous polynomial-time algorithm this linear time algorithm does not come with a certification guarantee.
In \pref{sec:semi-random-tensor}, we show that small adversarial perturbations can cause this algorithm to fail, whereas the previous algorithm is robust against such perturbations.
We also devise an algorithm with the certification property and running time $\tilde O(n^{4})$ (which is subquadratic in the size $ n^3$ of the input).

\begin{theorem}
\label{cor:tpca-main}
   \label{thm:optimize-certify-fast}
  There exists an algorithm with running time $\tilde O(n^4)$ (for inputs of size $n^3$) that given an instance of \pref{prob:spiked-tensor} with $\tau \geq  n^{3/4} (\log n)^{1/4} / \epsilon$ for some $\epsilon$, outputs with probability $1-O(n^{-10})$ a unit vector $v'$ with $\langle  v,v' \rangle\ge 1-O(\e)$ and certifies that all vectors bounded away from $v'$ have MLE objective value significantly less than $\tau$.
\end{theorem}

\paragraph{Higher-order tensors}

Our algorithmic results also extend in a straightforward way to tensors of order higher than $3$.
(See \pref{sec:higher-order-tensors} for some details.)
For simplicity we give some of these results only for the higher-order analogue of \pref{prob:spiked-tensor};
we conjecture however that all our results for \pref{prob:spiked-tensor-symm} generalize in similar fashion.

\begin{theorem}
\label{thm:recovery-certify-general}
Let $k$ be an odd integer, $v_0 \in \R^n$ a unit vector, $\tau \geq n^{k/4} \log(n)^{1/4} / \epsilon$,
and $\bA$ an order-$k$ tensor with independent unit Gaussian entries.
Let $\bT(x) = \tau \cdot \iprod{v_0,x}^k + \bA(x)$.

\begin{enumerate}
\item
There is a polynomial-time algorithm, based on semidefinite programming,
which on input $\bT(x) = \tau \cdot \iprod{v_0,x}^k + \bA(x)$
returns a unit vector $v$ with $\iprod{v_0,v} \geq 1 - O(\epsilon)$
with probability $1 - O(n^{-10})$ over random choice of $\bA$.

\item
There is a polynomial-time algorithm, based on semidefinite programming,
which on input $\bT(x) = \tau \cdot \iprod{v_0,x}^k + \bA(x)$
certifies that $\bT(x) \leq \tau \cdot \iprod{v,x}^k + O(n^{k/4} \log(n)^{1/4})$
for some unit $v$
with probability $1 - O(n^{-10})$ over random choice of $\bA$.
This guarantees in particular that $v$ is close to a maximum likelihood estimator for the problem of recovering the signal $v_0$ from the input $\tau \cdot v_0^{\tensor k} + \bA$.

\item
By solving the semidefinite relaxation approximately,
both algorithms can be implemented in time $\tilde O(m^{1 + 1/k})$, where $m = n^k$ is the input size.
\end{enumerate}
For even $k$, the above all hold, except now we recover $v$ with $\iprod{v_0,v}^2 \geq 1 - O(\epsilon)$,
and the algorithms can be implemented in nearly linear time.
\end{theorem}
\begin{remark}
When $\bA$ is a \emph{symmetric} noise tensor (the higher-order analogue of \pref{prob:spiked-tensor-symm}), (1--2) above hold.
We conjecture that (3) does as well.
\end{remark}

The last theorem, the higher-order generalization of \pref{thm:linear-intro}, almost completely resolves a conjecture of Montanari and Richard regarding tensor unfolding algorithms for odd $k$.
We are able to prove their conjectured signal-to-noise ratio $\tau$ for an algorithm that works mainly by using an unfolding of the input tensor, but our algorithm includes an extra random-rotation step to handle sparse signals.
We conjecture but cannot prove that the necessity of this step is an artifact of the analysis.

\begin{theorem}
\label{thm:faster-general}
Let $k$ be an odd integer, $v_0 \in \R^n$ a unit vector, $\tau \geq n^{k/4} / \epsilon$,
and $\bA$ an order-$k$ tensor with independent unit Gaussian entries.
There is a nearly-linear-time algorithm, based on tensor unfolding, which,
with probability $1 - O(n^{-10})$ over random choice of $\bA$,
recovers a vector $v$ with $\iprod{v,v_0}^2 \geq 1 - O(\epsilon)$.
This continues to hold when $\bA$ is replaced by a \emph{symmetric} noise tensor (the higher-order analogue of \pref{prob:spiked-tensor-symm}).
\end{theorem}

\subsection{Techniques}
\label{sec:techniques}

We arrive at our results via an analysis of \pref{prob:homog-poly-opt} for the function $f(x)=\sum_{ijk} \bT_{ijk}x_i x_j x_k$, where $\bT$ is an instance of \pref{prob:spiked-tensor}.
The function $f$ decomposes as $f = g + h$ for a signal $g(x) = \tau \cdot \iprod{v,x}^3$ and noise $h(x) = \sum_{ijk} a_{ijk} x_i x_j x_j$ where $\{a_{ijk}\}$ are iid standard Gaussians.
The signal $g$ is maximized at $x = v$, where it takes the value $\tau$.
The noise part, $h$, is with high probability at most $\tilde O(\sqrt n)$ over the unit sphere.
We have insisted that $\tau$ be much greater than $\sqrt n$, so $f$ has a unique global maximum, dominated by the signal $g$.
The main problem is to find it.

To maximize $g$, we apply the \emph{Sum-of-Squares meta-algorithm} (SoS).
SoS provides a hierarchy of strong convex relaxations of \pref{prob:homog-poly-opt}.
Using convex duality, we can recast the optimization problem as one of efficiently certifying the upper bound on $h$ which shows that optima of $g$ are dominated by the signal.
SoS efficiently finds boundedness certificates for $h$ of the form
\begin{align*}
c - h(x) = s_1(x)^2 + \cdots + s_k(x)^2
\end{align*}
where ``$=$'' denotes equality in the ring $\R[x]/(\|x\|^2 - 1)$ and where $s_1,\ldots,s_k$ have bounded degree,
when such certificates exist.
(The polynomials $\{s_i\}$ and $\{t_j\}$ certify that $h(x) \leq c$. Otherwise $c - h(x)$ would be negative, but this is impossible by the nonnegativity of squared polynomials.)

Our main technical contribution is an almost-complete characterization of certificates like these for such degree-$3$ random polynomials $h$ when the polynomials $\{ s_i \}$ have degree at most four.
In particular, we show that with high probability in the random case a degree-$4$ certificate exists for $c = \tilde O(n^{3/4})$,
and that with high probability, no significantly better degree-four certificate exists.

\paragraph{Algorithms} We apply this characterization in three ways to obtain three different algorithms.
The first application is a polynomial-time based on semidefinite programming algorithm that maximizes $f$ when $\tau \geq \tilde \Omega(n^{3/4})$
(and thus solves TPCA in the spiked tensor model for $\tau \geq \tilde \Omega(n^{3/4})$.)
This first algorithm involves solving a large semidefinite program associated to the SoS relaxation.
As a second application of this characterization, we avoid solving the semidefinite program.
Instead, we give an algorithm running in time $\tilde O(n^4)$ which quickly constructs only a small portion of an almost-optimal SoS boundedness certificate;
in the random case this turns out to be enough to find the signal $v$ and certify the boundedness of $g$.
(Note that this running time is only a factor of $n \polylog n$ greater than the input size $n^3$.)

Finally, we analyze a third algorithm for TPCA
which simply computes the highest singular vector of a matrix unfolding of the input tensor.
This algorithm was considered in depth by Montanari and Richard, who fully characterized its behavior in the case of even-order tensors (corresponding to $k = 4, 6, 8, \ldots$ in \pref{prob:homog-poly-opt}).
They conjectured that this algorithm successfully recovers the signal $v$ at the signal-to-noise ratio $\tau$ of \pref{thm:optimize-certify-fast} for \pref{prob:spiked-tensor} and \pref{prob:spiked-tensor-symm}.
Up to an extra random rotations step before the tensor unfolding in the case that the input comes from \pref{prob:spiked-tensor-symm} (and up to logarithmic factors in $\tau$) we confirm their conjecture.
We observe that their algorithm can be viewed as a method of rounding a non-optimal solution to the SoS relaxation to find the signal.
We show, also, that for $k = 4$, the degree-$4$ SoS relaxation does no better than the simpler tensor unfolding algorithm as far as signal-to-noise ratio is concerned.
However, for odd-order tensors this unfolding algorithm does not certify its own success in the way our other algorithms do.

\paragraph{Lower Bounds}
In \pref{thm:lb-intro-informal}, we show that degree-$4$ SoS cannot certify that the noise polynomial $\bA(x) = \sum_{ijk} a_{ijk} x_i x_j x_k$ for $a_{ijk}$ iid standard Gaussians satisfies $\bA(x) \leq o(n^{3/4})$.

To show that SoS certificates do \emph{not} exist we construct a corresponding dual object.
Here the dual object is a degree-$4$ \emph{pseudo-expectation:} a linear map $\pE : \R[x]_{\leq 4} \rightarrow \R$ pretending to give the expected value of polynomials of degree at most $4$ under some distribution on the unit sphere.
``Pretending'' here means that, just like an actual distribution, $\pE p(x)^2 \geq 0$ for any $p$ of degree at most $4$.
In other words, $\pE$ is positive semidefinite on degree $4$ polynomials.
While for any \emph{actual} distribution over the unit sphere $\E \bA(x) \leq \tO(\sqrt n)$, we will give $\pE$ so that $\pE \bA(x) \geq \tOmega(n^{3/4})$.

To ensure that $\pE \bA(x) \geq \tOmega(n^{3/4})$, for monomials $x_i x_j x_k$ of degree $3$ we take $\pE x_i x_j x_k \approx \tfrac{n^{3/4}}{\iprod{\bA, \bA}} a_{ijk}$.
For polynomials $p$ of degree at most $2$ it turns out to be enough to set $\pE p(x) \approx \E^\mu p(x)$ where $\E^\mu$ denotes the expectation under the uniform distribution on the unit sphere.

Having guessed these degree $1,2$ and $3$ pseudo-moments, we need to define $\pE x_i x_j x_k x_\ell$ so that $\pE$ is PSD.
Representing $\pE$ as a large block matrix, the Schur complement criterion for PSDness can be viewed as a method for turning candidate degree $1$--$3$ moments (which here lie on upper-left and off-diagonal blocks) into a candidate matrix $M \in \R^{n^2 \times n^2}$ of degree-$4$ pseudo-expectation values which, if used to fill out the degee-$4$ part of $\pE$, would make it PSD.

We would thus like to set $\pE x_i x_j x_k x_\ell = M[(i,j),(k,l)]$.
Unfortunately, these candidate degree-$4$ moments $M[(i,j),(k,l)]$ do not satisfy commutativity; that is, we might have $M[(i,j),(k,l)] \neq M[(i,k),(j,\ell)]$ (for example).
But a valid pseudo-expectation must satisfy $\pE x_i x_j x_k x_\ell = \pE x_i x_k x_j x_\ell$.
To fix this, we average out the noncommutativity by setting $\pD x_i x_j x_k x_\ell = \tfrac{1}{|\cS_4|} \sum_{\pi \in \cS_4} M[(\pi(i), \pi(j)),(\pi(k),\pi(\ell))]$, where $\cS_4$ is the symmetric group on $4$ elements.

This ensures that the candidate degree-$4$ pseudo-expectation $\pD$ satisfies commutativity, but it introduces a new problem.
While the matrix $M$ from the Schur complement was guaranteed to be PSD and even to make $\pE$ PSD when used as its degree-$4$ part, some of the permutations $\pi \cdot M$ given by $(\pi \cdot M)[(i,j),(k,\ell)] = M[(\pi(i), \pi(j)),(\pi(k),\pi(\ell))]$ need not even be PSD themselves.
This means that, while $\pD$ avoids having large negative eigenvalues (since it is correlated with $M$ from Schur complement), it will have some small negative eigenvalues; i.e. $\pD p(x)^2 < 0$ for some $p$.

For each permutation $\pi \cdot M$ we track the most negative eigenvalue $\lambda_{min} (\pi \cdot M)$ using matrix concentration inequalities.
After averaging the permutations together to form $\pD$ and adding this to $\pE$ to give a linear functional $\pE + \pD$ on polynomials of degree at most $4$, our final task is to remove these small negative eigenvalues.
For this we mix $\pE + \pD$ with $\mu$, the uniform distribution on the unit sphere.
Since $\E^\mu$ has eigenvalues bounded away from zero, our final pseudo-expectation
\[
  \pE'p(x) \defeq \underbrace{\epsilon \cdot \pE p(x)}_{\text{degree $1$-$3$ pseudo-expectations}} + \underbrace{\epsilon \cdot \pD p(x)}_{\text{degree $4$ pseudo-expectations}} + \underbrace{(1 - \epsilon) \cdot \E\nolimits^\mu p(x)}_{\text{fix negative eigenvalues}}
  \]
 is PSD for $\epsilon$ small enough.
Having tracked the magnitude of the negative eigenvalues of $\pD$, we are able to show that $\epsilon$ here can be taken \emph{large} enough to get $\pE'\bA(x) = \tOmega(n^{3/4})$, which will prove \pref{thm:lb-intro-informal}.

\subsection{Related Work}
There is a vast literature on tensor analogues of linear algebra problems---too vast to attempt any survey here.
Tensor methods for machine learning, in particular for learning latent variable models, have garnered recent attention, e.g., with works of Anandkumar et al. \cite{DBLP:journals/jmlr/AnandkumarGHKT14, DBLP:journals/corr/abs-1302-2684}.
These approaches generally involve decomposing a tensor which captures some aggregate statistics of input data into rank-one components.
A recent series of papers analyzes the tensor power method, a direct analogue of the matrix power method, as a way to find rank-one components of random-case tensors \cite{DBLP:journals/corr/AnandkumarGJ14, DBLP:journals/corr/AnandkumarGJ14b}.

Another recent line of work applies the Sum of Squares (a.k.a. Lasserre or Lasserre/Parrilo) hierarchy of convex relaxations to learning problems.
See the survey of Barak and Steurer for references and discussion of these relaxations \cite{DBLP:journals/corr/BarakS14}.
Barak, Kelner, and Steurer show how to use SoS to efficiently find sparse vectors planted in random linear subspaces, and the same authors give an algorithm for dictionary learning with strong provable statistical guarantees \cite{DBLP:conf/stoc/BarakKS14, DBLP:journals/corr/BarakKS14}.
These algorithms, too, proceed by decomposition of an underlying random tensor; they exploit the strong (in many cases, the strongest-known) algorithmic guarantees offered by SoS for this problem in a variety of average-case settings.

Concurrently and independently of us, and also inspired by the recently-discovered applicability of tensor and sum-of-squares methods to machine learning, Barak and Moitra use SoS techniques formally related to ours to address the \emph{tensor prediction} problem:
given a low-rank tensor (perhaps measured with noise) only a subset of whose entries are revealed, predict the rest of the tensor entries \cite{DBLP:journals/corr/BarakM15}.
They work with worst-case noise and study the number of revealed entries necessary for the SoS hierarchy to successfully predict the tensor.
By constrast, in our setting, the entire tensor is revealed, and we study the signal-to-noise threshold necessary for SoS to recover its principal component under distributional assumptions on the noise that allow us to avoid worst-case hardness behavior.

Since Barak and Moitra work in a setting where few tensor entries are revealed, they are able to use algorithmic techniques and lower bounds from the study of sparse random constraint satisfaction problems (CSPs), in particular random 3XOR \cite{DBLP:conf/stacs/GoerdtK01, DBLP:journals/siamcomp/FriedmanGK05, DBLP:journals/toc/FeigeO07, DBLP:conf/focs/FeigeKO06}.
The tensors we study are much denser.
In spite of the density (and even though our setting is real-valued), our algorithmic techniques are related to the same spectral refutations of random CSPs.
However our lower bound techniques do not seem to be related to the proof-complexity techniques that go into sum-of-squares lower bound results for random CSPs.

The analysis of tractable tensor decomposition in the rank one plus noise model that we consider here (the \emph{spiked tensor model}) was initiated by Montanari and Richard, whose work inspired the current paper \cite{richard2014tensorpca}.
They analyze a number of natural algorithms and find that tensor unfolding algorithms, which use the spectrum of a matrix unfolding of the input tensor, are most robust to noise.
Here we consider more powerful convex relaxations, and in the process we tighten Montanari and Richard's analysis of tensor unfolding in the case of odd-order tensors.
In concurrent and independent work, Zheng and Tomioka also give a tight analysis of tensor unfolding for the asymmetric version of the spiked model of \tpca (\pref{prob:spiked-tensor}) \cite[Theorem 1]{2015arXiv150305479Z}.

Related to our lower bound, Montanari, Reichman, and Zeitouni (MRZ) prove strong impossibility results for the problem of detecting rank-one perturbations of Gaussian matrices and tensors using \emph{any} eigenvalue of the matrix or unfolded tensor; they are able to characterize the precise threshold below which the entire spectrum of a perturbed noise matrix or unfolded tensor becomes indistinguishable from pure noise \cite{montanari2014limitation}.
This lower bound is incomparable to our lower bound for the degree-4 SoS relaxation.
The MRZ lower bound considers fine-grained information about the spectrum of a single matrix associated with the detection problem.
Our lower bound considers coarser information (just the top eigenvalue) but it applies to a wide range of matrices associated with the problem (all matrices generated via the degree-4 sum-of-squares proof system).

\section{Preliminaries}
\label{sec:preliminaries}

\subsection{Notation}
We use $x=(x_1,\ldots,x_n)$ to denote a vector of indeterminates.
The letters $u,v,w$ are generally reserved for real vectors.
The letters $\alpha,\beta$ are reserved for multi-indices; that is, for tuples $(i_1,\ldots,i_k)$ of indices.
For $f,g : \N \rightarrow \R$ we write $f \precsim g$ for $f = O(g)$ and $f \succsim g$ for $f = \Omega(g)$.
We write $f = \tilde O(g)$ if $f(n) \leq g(n) \cdot \polylog n$, and $f = \tilde \Omega(g)$ if $f \geq g(n) / \polylog n$.

We employ the usual Loewner (a.k.a. positive semi-definite) ordering $\succeq$ on Hermitian matrices.

We will be heavily concerned with tensors and matrix flattenings thereof.
In general, boldface capital letters $\bT$ denote tensors and ordinary capital letters denote matrices $A$.
We adopt the convention that unless otherwise noted for a tensor $\bT$ the matrix $T$ is the squarest-possible unfolding of $\bT$.
If $\bT$ has even order $k$ then $T$ has dimensions $n^{k/2} \times n^{k/2}$.
For odd $k$ it has dimensions $n^{\lfloor k/2 \rfloor} \times n^{\lceil k/2 \rceil}$.
All tensors, matrices, vectors, and scalars in this paper are real.

We use $\iprod{\cdot,\cdot}$ to denote the usual entrywise inner product of vectors, matrices, and tensors.
For a vector $v$, we use $\| v \|$ to denote its $\ell_2$ norm. For a matrix $A$, we use $\| A \|$ to denote its operator norm (also known as the spectral or $\ell_2$-to-$\ell_2$ norm).

For a $k$-tensor $\bT$, we write ${\bT}(v)$ for $\iprod{v^{\tensor k},\bT}$.
Thus, $\bT(x)$ is a homogeneous real polynomial of degree $k$.

We use $\cS_k$ to denote the symmetric group on $k$ elements.
For a $k$-tensor $\bT$ and $\pi \in \cS_k$, we denote by ${\bT}^\pi$ the $k$-tensor with indices permuted according to $\pi$, so that ${\bT}^\pi_{\alpha} = {\bT}_{\pi^{-1}(\alpha)}$.
A tensor $\bT$ is symmetric if for all $\pi \in \cS_k$ it is the case that ${\bT}^\pi = \bT$.
(Such tensors are sometimes called ``supersymmetric.'')

For clarity, most of our presentation focuses on $3$-tensors.
For an $n \times n$ $3$-tensor $\bT$, we use $T_i$ to denote its $n \times n$ matrix slices along the first mode, i.e., $(T_i)_{j,k} = \bT_{i,j,k}$.

We often say that an sequence $\{ E_n \}_{n \in \bbN}$ of events occurs with high probability, which for us means that $\Pr(E_n \mbox{ fails}) = O(n^{-10})$.
(Any other $n^{-c}$ would do, with appropriate modifications of constants elsewhere.)

\subsection{Polynomials and Matrices}
\label{sec:matrix-prelims}
Let $\R[x]_{\leq d}$ be the vector space of polynomials with real coefficients in variables $x = (x_1,\ldots,x_n)$, of degree at most $d$.
We can represent a homogeneous even-degree polynomial $p \in \R[x]_{d}$ by an $n^{d/2} \times n^{d/2}$ matrix:
a matrix $M$ is a \emph{matrix representation} for $p$ if $p(x)=\iprod{x^{\tensor d/2}, M x^{\tensor d/2}}$.
If $p$ has a matrix representation $M \succeq 0$, then $p = \sum_i p_i(x)^2$ for some polynomials $p_i$.

\subsection{The Sum of Squares (SoS) Algorithm}

\begin{definition}
  Let $\cL: \R[x]_{\leq d} \rightarrow \R$ be a linear functional on polynomials of degree at most $d$ for some $d$ even.
  Suppose that
  \begin{itemize}
  \item $\cL 1 = 1$.
  \item $\cL p(x)^2 \geq 0$ for all $p \in \R[x]_{\leq d/2}$.
  \end{itemize}
  Then $\cL$ is a degree-$d$ pseudo-expectation.
  We often use the suggestive notation $\pE$ for such a functional, and think of $\pE p(x)$ as giving the expectation of the polynomial $p(x)$ under a \emph{pseudo-distribution} over $\{ x \}$.

For $p \in \R[x]_{\leq d}$ we say that the pseudo-distribution $\{ x \}$ (or, equivalently, the functional $\pE$) satisfies $\{ p(x) = 0 \}$ if $\pE p(x)q(x) = 0$ for all $q(x)$ such that $p(x)q(x) \in \R[x]_{\leq d}$.
\end{definition}

Pseudo-distributions were first introduced in \cite{DBLP:conf/stoc/BarakBHKSZ12} and are surveyed in \cite{DBLP:journals/corr/BarakS14}.

We employ the standard result that, up to negligible issues of numerical accuracy, if there exists a degree-$d$ pseudo-distribution satisfying constraints $\{ p_0(x) = 0, \ldots, p_m(x) = 0 \}$, then it can be found in time $n^{O(d)}$ by solving a semidefinite program of size $n^{O(d)}$.
(See \cite{DBLP:journals/corr/BarakS14} for references.)

\section{Certifying Bounds on Random Polynomials}
\label{sec:sos-proofs}

Let $f \in \R[x]_d$ be a homogeneous degree-$d$ polynomial.
When $d$ is even, $f$ has square matrix representations of dimension
$n^{d/2} \times n^{d/2}$. The maximal eigenvalue of a matrix representation
$M$ of $f$ provides a natural certifiable upper bound on $\max_{\|v\| = 1} f(v)$,
as
\[
  f(v) = \iprod{v^{\tensor d/2}, Mv^{\tensor d/2}} \leq \max_{w \in \R^{n^{d/2}}}
          \frac{\iprod{w,Mw}}{\iprod{w,w}} = \|M\|\mper
\]
When $f(x) = \bA(x)$ for an even-order tensor $\bA$ with independent random entries,
the quality of this certificate is well characterized by random matrix theory.
In the case where the entries of $\bA$ are standard Gaussians, for instance,
$\|M\| = \|A + A^T\| \le \tilde O(n^{d/4})$ with high probability, thus certifying that
$\max_{\|v\| = 1} f(v) \le \tilde O(n^{d/4})$.

A similar story applies to $f$ of odd degree with random coefficients, but with a catch: the certificates are not as good.
For example, we expect a degree-$3$ random polynomial to be a smaller and simpler object than one of degree-$4$,
and so we should be able to certify a tighter upper bound on $\max_{\|v\| = 1} f(v)$.
The matrix representations of $f$ are now rectangular $n^2 \times n$ matrices whose top singular values are certifiable upper bounds on $\max_{\|v\| = 1} f(v)$.
But in random matrix theory, this maximum singular value depends (to a first approximation) only on the longer dimension $n^2$, which is the same here as in the degree-4 case.
Again when $f(x) = \bA(x)$, this time where $\bA$ is an order-3 tensor of independent standard Gaussian entries, $\|M\| = \sqrt{\|AA^T\|} \ge \tilde \Omega(n)$, so that this method cannot certify better than $\max_{\|v\| = 1} f(v) \le \tilde O(n)$.
Thus, the natural spectral certificates are unable to exploit the decrease in degree from $4$ to $3$ to improve the certified bounds.

To better exploit the benefits of square matrices, we bound the maxima of degree-$3$ homogeneous $f$ by a degree-$4$ polynomial.
In the case that $f$ is multi-linear, we have the polynomial identity $f(x) = \frac 1 3 \iprod{x,\nabla f(x)}$.
Using Cauchy-Schwarz, we then get $f(x) \leq \frac 1 3 \|x\| \|\nabla f(x)\|$.
This inequality suggests using the degree-$4$ polynomial $\|\nabla f(x)\|^2$ as a bound on $f$.
Note that local optima of $f$ on the sphere occur where $\nabla f(v) \propto v$, and so this bound is tight at local maxima.
Given a random homogeneous $f$, we will associate a degree-$4$ polynomial related to $\|\nabla f\|^2$ and show that this polynomial yields the best possible degree-$4$ SoS-certifiable bound on $\max_{\|v\| = 1} f(v)$.

\begin{definition}
  Let $f\in\R[x]_3$ be a homogeneous degree-$3$ polynomial with indeterminates $x=(x_1,\ldots,x_n)$.
  Suppose $A_1,\ldots,A_n$ are matrices such that $f=\sum_i x_i \langle x, A_ix \rangle$. We say that $f$ is
  $\lambda$-bounded if there are matrices $A_1,\ldots,A_n$ as above and a matrix representation $M$ of $\|x\|^4$ so that $\sum_i A_i\otimes A_i \preceq \lambda^2 \cdot M$.
\end{definition}

We observe that for $f$ multi-linear in the coordinates $x_i$ of $x$, up to a constant factor we may take the matrices $A_i$ to be matrix representations of $\partial_i f$, so that $\sum_i A_i \tensor A_i$ is a matrix representation of the polynomial $\|\nabla f\|^2$.
This choice of $A_i$ may not, however, yield the optimal spectral bound $\lambda^2$.

The following theorem is the reason for our definition of $\lambda$-boundedness.
\begin{theorem}
  \label{thm:lambda-bddness}
  Let $f \in \R[x]_3$ be $\lambda$-bounded.
  Then $\max_{\|v\| = 1} f(v) \leq \lambda$, and the degree-$4$ SoS algorithm certifies this.
  In particular, every degree-$4$ pseudo-distribution $\{x\}$ over $\R^n$ satisfies
  \begin{displaymath}
    \pE f  \le \lambda \cdot \Paren{\pE \lVert  x \rVert^4}^{3/4}\mper
  \end{displaymath}
\end{theorem}

\begin{proof}
  By Cauchy--Schwarz for pseudo-expectations, the pseudo-distribution satisfies
  $\Paren{\pE \lVert  x \rVert^2}^2 \le \pE \lVert  x \rVert^4$ and $\Paren{\pE \sum_i x_i \langle x,A_i x \rangle}^2 \le \Paren{\pE \sum_i x_i^2}\cdot \Paren{\sum_i \langle  x,A_i x \rangle^2}$.
  Therefore,
  \begin{align*}
    \pE f
    & = \pE \sum\nolimits_i x_i \cdot \langle  x, A_i x \rangle \\
    & \le \Paren{\pE \sum\nolimits_i x_i^2}^{1/2} \cdot \Paren{\pE \sum\nolimits_i \langle  x,A_i x \rangle^2}^{1/2} \\
    & = \Paren{\pE \lVert  x \rVert^2}^{1/2} \cdot \Paren{\pE \langle  x^{\otimes 2}, \Paren{\sum\nolimits_i A_i\otimes A_i }x^{\otimes 2}  \rangle}^{1/2} \\
    & \le  \Paren{\pE \lVert  x \rVert^4}^{1/4} \cdot \Paren{\pE \langle  x^{\otimes 2}, \lambda^2 \cdot M x^{\otimes 2}  \rangle}^{1/2} \\
    & = \lambda \cdot \Paren{\pE \lVert  x \rVert^4}^{3/4}\mper
  \end{align*}
  The last inequality also uses the premise $\Paren{\sum\nolimits_i A_i\otimes A_i }\preceq \lambda^2 \cdot M$ for some matrix representation $M$ of $\|x\|^4$, in the following way.
  Since $M' \seteq \lambda^2 \cdot M - \Paren{\sum\nolimits_i A_i\otimes A_i }\succeq 0$, the polynomial $\iprod{x^{\tensor 2}, M' x^{\tensor 2}}$ is a sum of squared polynomials.
  Thus, $\pE \iprod{x^{\tensor 2}, M' x^{\tensor 2}} \geq 0$ and the desired inequality follows.
\end{proof}

We now state the degree-$3$ case of a general $\lambda$-boundedness fact for homogeneous polynomials with random coefficients.
The SoS-certifiable bound for a random degree-$3$ polynomial this provides is the backbone of our SoS algorithm for tensor PCA in the spiked tensor model.
\begin{theorem}
  \label{thm:concentration-3}
  Let $\bA$ be a $3$-tensor with independent entries from $\cN(0,1)$.
  Then $\bA(x)$ is $\lambda$-bounded with $\lambda = O(n^{3/4}\log(n)^{1/4})$, with high probability.
\end{theorem}
The full statement and proof of this theorem, generalized to arbitrary-degree homogeneous polynomials,
may be found as \pref{thm:concentration}; we prove the statement above as a corollary in \pref{sec:concentration-bounds}.
Here provide a proof sketch.
\begin{proof}[Proof sketch]
  We first note that the matrix slices $A_i$ of $\bA$ satisfy $\bA(x) = \sum_i x_i \iprod{ x,A_i x}$.
  Using the matrix Bernstein inequality, we show that
  $\sum_i A_i \tensor A_i - \E \sum_i A_i \tensor A_i \preceq O(n^{3/2} (\log n)^{1/2})\cdot \Id$
  with high probability.
  At the same time, a straightforward computation shows that $\tfrac 1 n \E \sum_i A_i \tensor A_i$ is a matrix representation of $\|x\|^4$.
  Since $\Id$ is as well, we get that $\sum_i A_i \tensor A_i \preceq \lambda^2 \cdot M\mcom$
  where $M$ is some matrix representation of $\|x\|^4$ which combines $\Id$ and $\E \sum_i A_i \tensor A_i$,
  and $\lambda = O(n^{3/4} (\log n)^{1/4})$.
\end{proof}

\begin{corollary}
  \label{cor:sos-cert-3}
  Let $\bA$ be a $3$-tensor with independent entries from $\cN(0,1)$.
  Then, with high probability, the degree-$4$ SoS algorithm certifies that $\max_{\|v\| = 1} \bA(v) \leq O(n^{3/4} (\log n)^{1/4})$.
  Furthermore, also with high probability, every pseudo-distribution $\{x\}$ over $\R^n$ satisfies
  \begin{displaymath}
    \pE \bA(x) \leq O(n^{3/4} (\log n)^{1/4}) (\pE \|x\|^4)^{3/4}\mper
  \end{displaymath}
\end{corollary}
\begin{proof}
  Immediate by combining \pref{thm:concentration-3} with \pref{thm:lambda-bddness}.
\end{proof}

\section{Polynomial-Time Recovery via Sum of Squares}
\label{sec:sos-algorithm}
Here we give our first algorithm for tensor PCA: we analyze the quality of the natural SoS relaxation of tensor PCA using our previous discussion of boundedness certificates for random polynomials,
and we show how to round this relaxation.
We discuss also the robustness of the SoS-based algorithm to some amount of additional \emph{worst-case} noise in the input.
For now, to obtain a solution to the SoS relaxation we will solve a large semidefinite program.
Thus, the algorithm discussed here is not yet enough to prove \pref{thm:optimize-certify-fast} and \pref{cor:tpca-main}:
the running time, while still polynomial, is somewhat greater than $\tilde O(n^4)$.
\Dnote{}

\begin{center}
\fbox{\begin{minipage}{\textwidth}
\begin{center}\textbf{Tensor PCA with Semidefinite Programming}\end{center}
\noindent\textbf{Input:} $\bT = \tau \cdot v_0^{\otimes 3} + \bA$, where $v \in \R^n$ and $\bA$ is some order-$3$ tensor. \\
\noindent\textbf{Goal:} Find $v \in \R^n$ with $\abs{\iprod{v,v_0}} \geq 1 - o(1)$.
\begin{algo}[Recovery]
\label{alg:sdp-recovery}
Using semidefinite programming, find the degree-$4$ pseudo-distribution $\{ x \}$ satisfying $\{ \|x\|^2 = 1 \}$ which maximizes $\pE \bT(x)$.
Output $\pE x / \| \pE x \|$.
\end{algo}
\begin{algo}[Certification]
\label{alg:sdp-certify}
Run \pref{alg:sdp-recovery} to obtain $v$.
Using semidefinite programming, find the degree-$4$ pseudo-distribution $\{ x \}$ satisfying $\{ \|x\| = 1 \}$ which maximizes $\pE \bT(x) - \tau \cdot \iprod{v,x}^3$.
If $\pE \bT(x) - \tau \cdot \iprod{v,x}^3 \leq O(n^{3/4} \log(n)^{1/4})$, output {\scshape certify}.
Otherwise, output {\scshape fail}.
\end{algo}
\end{minipage}}
\end{center}
The following theorem characterizes the success of \pref{alg:sdp-recovery} and \pref{alg:sdp-certify}
\begin{theorem}[Formal version of \pref{thm:sdp-main-informal}]
  \label{thm:sos-tpca}
  Let $\bT = \tau \cdot v_0^{\tensor 3} + \bA$, where $v_0 \in \R^n$ and $\bA$ has independent entries from $\cN(0,1)$. Let $\tau \succsim n^{3/4} \log(n)^{1/4} / \epsilon$.
  Then with high probability over random choice of $\bA$, on input $\bT$ or $\bT' \seteq \tau \cdot v_0^{\tensor 3} + \frac 1 {|\cS_3|} \sum_{\pi \in \cS_3} \bA^\pi$,
  \pref{alg:sdp-recovery} outputs a vector $v$ with $\iprod{v,v_0} \geq 1 - O(\epsilon)$.
  In other words, for this $\tau$, \pref{alg:sdp-recovery} solves both \pref{prob:spiked-tensor} and \pref{prob:spiked-tensor-symm}.

  For any unit $v_0 \in \R^n$ and $\bA$, if \pref{alg:sdp-certify} outputs \textup{\scshape certify} then $\bT(x) \leq \tau \cdot \iprod{v,x}^3 + O(n^{3/4} \log(n)^{1/4})$.
  For $\bA$ as described in either \pref{prob:spiked-tensor} or \pref{prob:spiked-tensor-symm} and $\tau \succsim n^{3/4} \log(n)^{1/4}/\epsilon$,
  \pref{alg:sdp-certify} outputs \textup{\scshape certify} with high probability.
\end{theorem}

The analysis has two parts.
We show that
\begin{enumerate}
  \item \label{item:if-bound-exists} if there exists a sufficiently good upper bound on $\bA(x)$ (or in the case of the symmetric noise input,
    on $\bA^\pi(x)$ for every $\pi \in \cS_3$) which is degree-4 SoS certifiable, then the vector recovered by the algorithm will be very close to $v$, and that
   \item \label{item:bound-exists} in the case of $\bA$ with independent entries from $\cN(0,1)$, such a bound exists with high probability.
\end{enumerate}
Conveniently, \pref{item:bound-exists} is precisely the content of \pref{cor:sos-cert-3}.
The following lemma expresses \pref{item:if-bound-exists}.

\begin{lemma}
\label{lem:sos-master-analysis}
  Suppose $\bA(x) \in \R[x]_3$ is such that $| \pE \bA(x)| \leq \epsilon \tau \cdot (\pE \| x \|^4)^{3/4}$ for any degree-$4$ pseudo-distribution $\{ x \}$.
  Then on input $\tau \cdot v_0^{\tensor 3} + \bA$, \pref{alg:sdp-recovery} outputs a unit vector $v$ with $\iprod{v,v_0} \geq 1 - O(\epsilon)$.
\end{lemma}
\begin{proof}
  \pref{alg:sdp-recovery} outputs $v = \pE x / \|\pE x\|$ for the pseudo-distribution that it finds,
so we'd like to show $\iprod{v_0, \pE x/\|\pE x\|} \geq 1 - O(\epsilon)$.
By pseudo-Cauchy-Schwarz (\pref{lem:pseudo-cs}), $\|\pE x\|^2 \le \pE \|x\|^2  = 1$, so it will suffice to prove just that $\iprod{v_0, \pE x } \geq 1 - O(\epsilon)$.

If $\pE \iprod{v_0,x}^3 \geq 1 - O(\epsilon)$, then by \pref{lem:deg-3-to-1-correlation} (and linearity of pseudo-expectation) we would have
\[
   \iprod{v_0, \pE x }
    \;=\; \pE \iprod{v_0, x}
    \;\ge\; 1 - O(2\epsilon) = 1 - O(\epsilon)
\]
So it suffices to show that $\pE \iprod{v_0, x}^3$ is close to $1$.

Recall that \pref{alg:sdp-recovery} finds a pseudo-distribution that maximizes $\pE \bT(x)$.
We split $\pE \bT(x)$ into the signal $\pE \iprod{v_0^{\tensor 3},x^{\tensor 3}}$ and noise $\pE\bA(x)$ components and use our hypothesized SoS upper bound on the noise.
  \begin{align*}
    \pE \bT(x)
     \;=\; \tau \cdot (\pE \iprod{v_0^{\tensor 3}, x^{\tensor 3}} ) + \pE \bA(x)
     \;\leq\; \tau \cdot ( \pE \iprod{v_0^{\tensor 3}, x^{\tensor 3}}) + \epsilon \tau \mper
  \end{align*}
Rewriting $\iprod{v_0^{\tensor 3}, x^{\tensor 3}}$ as $\iprod{v_0,x}^3$, we obtain
\begin{displaymath}
  \pE \iprod{v_0, x}^3  \;\geq\; \frac 1 \tau \cdot \pE \bT(x) - \epsilon\mper
\end{displaymath}

Finally, there exists a pseudo-distribution that achieves $\pE \bT(x) \ge \tau - \epsilon\tau$.
Indeed, the trivial distribution giving probability $1$ to $v_0$ is such a pseudo-distribution:
\[
  \bT(v_0) \;=\; \tau + \bA(v_0) \;\geq\; \tau - \epsilon \tau.
\]
Putting it together,
\[
  \pE \iprod{v_0, x}^3
  \;\ge\; \frac{1}{\tau} \cdot \pE \bT(x)  - \epsilon
  \; \geq \; \frac {(1 - \epsilon) \tau}{\tau} - \epsilon = 1 - O(\epsilon)\mper \qedhere
\]
\end{proof}

\begin{proof}[Proof of \pref{thm:sos-tpca}]
  We first address \pref{alg:sdp-recovery}. Let $\tau, \bT, \bT'$ be as in the theorem statement.
  By \pref{lem:sos-master-analysis}, it will be enough to show that with high probability
  every degree-$4$ pseudo-distribution $\{ x \}$ has $\pE \bA(x) \leq \epsilon' \tau \cdot (\pE \|x\|^4)^{3/4}$ and $\frac 1 {\cS_3} \pE \bA^\pi(x) \leq \epsilon' \tau \cdot (\pE \|x\|^4)^{3/4}$ for some $\epsilon' = \Theta(\epsilon)$.
  By \pref{cor:sos-cert-3} and our assumptions on $\tau$ this happens for each permutation $\bA^\pi$ individually with high probability,
  so a union bound over $\bA^\pi$ for $\pi \in \cS_3$ completes the proof.

  Turning to \pref{alg:sdp-certify}, the simple fact that SoS only certifies true upper bounds implies that the algorithm is never wrong when it outputs {\scshape certify}.
  It is not hard to see that whenever \pref{alg:sdp-recovery} has succeeded in recovering $v$ because $\pE \bA(x)$ is bounded, which as above happens with high probability, \pref{alg:sdp-certify} will output {\scshape certify}.
\end{proof}

\subsection{Semi-Random Tensor PCA}
\label{sec:semi-random-tensor}
We discuss here a modified TPCA model, which will illustrate the qualitative differences between the new tensor PCA algorithms we propose in this paper and previously-known algorithms.
The model is semi-random and semi-adversarial.
Such models are often used in average-case complexity theory to distinguish between algorithms which work by solving robust maximum-likelihood-style problems and those which work by exploiting some more fragile property of a particular choice of input distribution.

\begin{problem}[Tensor PCA in the Semi-Random Model]
  \label{prob:tpca-semirandom}
  Let $\bT = \tau \cdot v_0^{\tensor 3} + \bA$, where $v_0 \in \R^n$ and $\bA$ has independent entries from $\cN(0,1)$.
  Let $Q \in \R^{n \times n}$ with $\|\Id - Q \| \leq O(n^{-1/4})$, chosen adversarially depending on $\bT$.
  Let $\bT'$ be the $3$-tensor whose $n^2 \times n$ matrix flattening is $TQ$.
  (That is, each row of $T$ has been multiplied by a matrix which is close to identity.)
  On input $\bT'$, recover $v$.
\end{problem}

Here we show that \pref{alg:sdp-recovery} succeeds in recovering $v$ in the semi-random model.
\begin{theorem}
  Let $\bT'$ be the semi-random-model tensor PCA input, with $\tau \geq n^{3/4} \log(n)^{1/4}/\epsilon$.
  With high probability over randomness in $\bT'$, \pref{alg:sdp-recovery} outputs a vector $v$ with $\iprod{v,v_0} \geq 1 - O(\epsilon)$.
\end{theorem}
\begin{proof}
  By \pref{lem:sos-master-analysis}, it will suffice to show that $\bB \seteq (\bT' - \tau \cdot v_0^{\tensor 3})$ has  $\pE \bB(x) \leq \epsilon' \tau \cdot (\pE \|x\|^4)^{3/4}$ for any degree-$4$ pseudo-distribution $\{ x \}$, for some $\epsilon' = \Theta(\epsilon)$.
  We rewrite $\bB$ as
  \[
    \bB = (A + \tau \cdot v_0 (v_0 \tensor v_0)^T)(Q - \Id)+ A
  \]
  where $\bA$ has independent entries from $\cN(0,1)$.
  Let $\{ x \}$ be a degree-$4$ pseudo-distribution.
  Let $f(x) = \iprod{x^{\tensor 2}, (A + \tau \cdot v_0 (v_0 \tensor v_0)^T)(Q - \Id) x}$.
  By \pref{cor:sos-cert-3}, $\pE \bB(x) = \pE f(x) + O(n^{3/4} \log(n)^{1/4})(\pE \|x\|^4)^{3/4}$ with high probability.
  By triangle inequality and sub-multiplicativity of the operator norm, we get that with high probability
  \[
    \| (A + \tau \cdot v_0 (v_0 \tensor v_0))(Q - \Id)\| \leq (\|A \| + \tau) \|Q - \Id \| \leq O(n^{3/4})\mcom
  \]
  where we have also used \pref{lem:rectangular-concentration} to bound $\|A\| \leq O(n)$ with high probability and our assumptions on $\tau$ and $\|Q - \Id\|$.
  By an argument similar to that in the proof of \pref{thm:lambda-bddness} (which may be found in \pref{lem:sos-bound-from-op-norm}),
  this yields $\pE f(x) \leq O(n^{3/4})(\pE \|x\|^4)^{3/4}$ as desired.
\end{proof}

\section{Linear Time Recovery via Further Relaxation}
We now attack the problem of speeding up the algorithm from the preceding section.
We would like to avoid solving a large semidefinite program to optimality:
our goal is to instead use much faster linear-algebraic computations---in particular, we will recover the tensor PCA signal vector by performing a single singular vector computation on a relatively small matrix.
This will complete the proofs of \pref{thm:optimize-certify-fast} and \pref{thm:linear-intro}, yielding the desired running time.

Our SoS algorithm in the preceding section turned on the existence of the $\lambda$-boundedness certificate $\sum_i A_i \tensor A_i$,
where $A_i$ are the slices of a random tensor $\bA$.
Let $\bT = \tau \cdot v_0^{\tensor 3} + \bA$ be the spiked-tensor input to tensor PCA.
We could look at the matrix $\sum_i T_i \tensor T_i$ as a candidate $\lambda$-boundedness certificate for $\bT(x)$.
The spectrum of this matrix must not admit the spectral bound that $\sum_i A_i \tensor A_i$ does, because $\bT(x)$ is not globally bounded:
it has a large global maximum near the signal $v$.
This maximum plants a single large singular value in the spectrum of $\sum_i T_i \tensor T_i$.
The associated singular vector is readily decoded to recover the signal.

Before stating and analyzing this fast linear-algebraic algorithm, we situate it more firmly in the SoS framework.
In the following, we discuss \emph{spectral SoS}, a convex relaxation of \pref{prob:homog-poly-opt} obtained by weakening the full-power SoS relaxation.
We show that the spectrum of the aforementioned $\sum_i T_i  \tensor T_i$ can be viewed as approximately solving the spectral SoS relaxation.
This gives the fast, certifying algorithm of \pref{thm:optimize-certify-fast}.
We also interpret the tensor unfolding algorithm given by
Montanari and Richard for TPCA in the spiked tensor model as giving a more subtle approximate solution to the spectral SoS relaxation.
We prove a conjecture by those authors that the algorithm successfully recovers the TPCA signal at the same signal-to-noise ratio as our other algorithms, up to a small pre-processing step in the algorithm; this proves \pref{thm:linear-intro} \cite{richard2014tensorpca}.
This last algorithm, however, succeeds for somewhat different reasons than the others,
and we will show that it consequently fails to certify its own success and that it is not robust to a certain kind of semi-adversarial choice of noise.

\subsection{The Spectral SoS Relaxation}
\subsubsection{The SoS Algorithm: Matrix View}
\label{sec:sos-matrix-view}
To obtain spectral SoS, the convex relaxation of \pref{prob:homog-poly-opt} which we will be able to (approximately) solve quickly in the random case,
we first need to return to the full-strength SoS relaxation and examine it from a more linear-algebraic standpoint.

We have seen in \pref{sec:matrix-prelims} that a homogeneous $p \in \R[x]_{2d}$ may be represented as an $n^d \times n^d$ matrix whose entries correspond to coefficients of $p$.
A similar fact is true for non-homogeneous $p$.
Let $\tuples(d) = 1 + n + n^2 + \cdots + n^{d/2}$. Let $x^{\tensor \leq d/2} \seteq (x^{\tensor 0},x,x^{\tensor 2},\ldots,x^{\tensor d/2})$.
Then $p \in \R[x]_{\leq d}$ can be represented as an ${\tuples(d)} \times {\tuples(d)}$ matrix;
we say a matrix $M$ of these dimensions is a matrix representation of $p$ if $\iprod{x^{\leq \tensor d/2}, M x^{\leq \tensor d/2}} = p(x)$.
For this section, we let $\cM_p$ denote the set of all such matrix representation of $p$.

A degree-$d$ pseudo-distribution $\{x\}$ can similarly be represented as an
$\R^{\tuples(d) \times \tuples(d)}$ matrix.
We say that $M$ is a matrix representation for $\{ x \}$ if $M[\alpha,\beta] = \pE x^{\alpha}x^{\beta}$ whenever$\alpha$ and $\beta$ are multi-indices with $|\alpha|, |\beta| \leq d$.

Formulated this way, if $M_{ \{x\}}$ is the matrix representation of $\{ x \}$
and $M_p \in \cM_p$ for some $p \in \R[x]_{\leq 2d}$,
then $\pE p(x) = \iprod{ M_{\{x\}}, M_p}$.
In this sense, pseudo-distributions and polynomials, each represented as matrices, are dual under the trace inner product on matrices.

We are interested in optimization of polynomials over the sphere, and we have been looking at pseudo-distribution $\{ x \}$ satisfying $\{ \|x\|^2 - 1 = 0 \}$.
From this matrix point of view, the polynomial $\|x\|^2 - 1$ corresponds to a vector $w \in \R^{\tuples(d)}$
(in particular, the vector $w$ so that $ww^T$ is a matrix representation of $(\|x\|^2 - 1)^2$),
and a degree-4 pseudo-distribution $\{x\}$ satisfies $\{ \|x\|^2 - 1 = 0\}$ if and only if $w \in \ker M_{\{x\}}$.

A polynomial may have many matrix representations, but a pseudo-distribution
has just one: a matrix representation of a pseudo-distribution must obey strong
symmetry conditions in order to assign the same pseudo-expectation to every
representation of the same polynomial.
We will have much more to say about constructing matrices satisfying these symmetry conditions when we state and prove our lower bounds,
but here we will in fact profit from relaxing these symmetry constraints.

Let $p \in \R[x]_{\leq 2d}$.
In the matrix view, the SoS relaxation of the problem $\max_{\|x\|^2 = 1} p(x)$ is the following convex program.
\begin{align}
  \label{eq:sos-matrix}
  \max_{\substack{M : w \in \ker M \\ M \succeq 0 \\ \iprod{M, \cM_{1}} = 1}} \min_{M_p \in \cM_p} \iprod{M, M_p}\mper
\end{align}
It may not be immediately obvious why this program optimizes only over $M$ which are matrix representations of pseudo-distributions.
If, however, some $M$ does not obey the requisite symmetries, then $\min_{M_p \in \cM_p} \iprod{M, M_p}= -\infty$,
since the asymmetry may be exploited by careful choice of $M_p \in \cM_p$.
Thus, at optimality this program yields $M$ which is the matrix representation of a pseudo-distribution $\{ x\}$ satisfying $\{\|x\|^2 - 1 = 0 \}$.

\subsubsection{Relaxing to the Degree-$4$ Dual}
We now formulate spectral SoS.
In our analysis of full-power SoS for tensor PCA we have primarily considered pseudo-expectations of homogeneous degree-$4$ polynomials;
our first step in further relaxing SoS is to project from $\R[x]_{\leq 4}$ to $\R[x]_4$.
Thus, now our matrices $M,M'$ will be in $\R^{n^2 \times n^2}$ rather than $\R^{\tuples(2) \times \tuples(2)}$.
The projection of the constraint on the kernel in the non-homogeneous case implies $\Tr M = 1$ in the homogeneous case.
The projected program is
\begin{align*}
  \max_{\substack{\Tr M = 1 \\ M \succeq 0}} \min_{M_p \in \cM_p} \iprod{M, M'}\mper
\end{align*}
We modify this a bit to make explicit that the relaxation is allowed to add and subtract arbitrary matrix representations of the zero polynomial; in particular $M_{\|x\|^4} - \Id$ for any $M_{\|x\|^4} \in \cM_{\|x\|^4}$.
This program is the same as the one which precedes it.
\begin{align}
  \label{eq:sos-matrix-deg-4}
  \max_{\substack{\Tr M = 1 \\ M \succeq 0}} \min_{\substack{M_p \in \cM_p \\ M_{\|x\|^4} \in \cM_{\|x\|^4} \\ c \in \R}} \iprod{M, M_p - c\cdot M_{\|x\|^4}} + c\mper
\end{align}

By weak duality, we can interchange the $\min$ and the $\max$ in \pref{eq:sos-matrix-deg-4} to obtain the dual program:
\begin{align}
  \label{eq:sos-matrix-dual}
  \max_{\substack{\Tr M = 1 \\ M \succeq 0}} \min_{\substack{M_p \in \cM_p \\ M_{\|x\|^4} \in \cM_{\|x\|^4} \\ c \in \R}} \iprod{M, M_p - c \cdot M_{\|x\|^4}}
  & \leq \min_{\substack{M_p \in \cM_p \\ M_{\|x\|^4} \in \cM_{\|x\|^4} \\ c  \in \R}} \max_{\substack{\Tr M = 1 \\ M \succeq 0}} \iprod{M, M_p - c \cdot M_{\|x\|^4}} + c\\
  & = \min_{\substack{M_p \in \cM_p \\ M_{\|x\|^4} \in \cM_{\|x\|^4} \\ c \in \R}} \max_{\|v\| = 1} \iprod{vv^T, M_p - c \cdot M_{\|x\|^4}} + c
\end{align}
We call this dual program the spectral SoS relaxation of $\max_{\|x\| = 1} p(x)$.
If $p = \sum_i \iprod{x,A_i x}$ for $\bA$ with independent entries from $\cN(0,1)$,
the spectral SoS relaxation achieves the same bound as our analysis of the full-strength SoS relaxation:
for such $p$, the spectral SoS relaxation is at most $O(n^{3/2} \log(n)^{1/2})$ with high probability.
The reason is exactly the same as in our analysis of the full-strength SoS relaxation:
the matrix $\sum_i A_i \tensor A_i$, whose spectrum we used before to bound the full-strength SoS relaxation, is still a feasible dual solution.

\subsection{Recovery via the $\sum_i T_i \tensor T_i$ Spectral SoS Solution}
Let $\bT = \tau \cdot v_0^{\tensor 3} + \bA$ be the spiked-tensor input to tensor PCA.
We know from our initial characterization of SoS proofs of boundedness for degree-$3$ polynomials that the polynomial $\bT'(x) \seteq (x \tensor x)^T(\sum_i T_i \tensor T_i)(x \tensor x)$ gives SoS-certifiable upper bounds on $\bT(x)$ on the unit sphere.
We consider the spectral SoS relaxation of $\max_{\|x\| = 1} \bT'(x)$,
\begin{displaymath}
  \min_{\substack{M_{\bT(x)} \in \cM_{\bT(x)} \\ M_{\|x\|^4} \in \cM_{\|x\|^4} \\ c \in \R}} \|M_{\bT(x)} - c\cdot M_{\|x\|^4} \| + c\mper
\end{displaymath}
Our goal now is to guess a good $M' \in \cM_{\bT(x)}$.
We will take as our dual-feasible solution the top singular vector of $\sum_i T_i \tensor T_i - \E \sum_i A_i \tensor A_i$.
This is dual feasible with $c = n$, since routine calculation gives $\iprod{x^{\tensor 2}, (\E \sum_i A_i \tensor A_i) x^{\tensor 2}} = \|x\|^4$.
This top singular vector, which differentiates the spectrum of $\sum_i T_i \tensor T_i$ from that of $\sum_i A_i \tensor A_i$,
is exactly the manifestation of the signal $v_0$ which differentiates $\bT(x)$ from $\bA(x)$.
The following algorithm and analysis captures this.

\begin{center}
\fbox{\begin{minipage}{\textwidth}
\begin{center}\textbf{Recovery and Certification with $\sum_i T_i \tensor T_i$}\end{center}
\noindent\textbf{Input:} $\bT = \tau \cdot v_0^{\otimes 3} + \bA$, where $v_0 \in \R^n$ and $\bA$ is a $3$-tensor. \\
\noindent\textbf{Goal:} Find $v \in \R^n$ with $\abs{\iprod{v,v_0}} \geq 1 - o(1)$.
\begin{algo}[Recovery]
\label{alg:fast-recovery-opt}
Compute the top (left or right) singular vector $v'$ of $M \seteq \sum_i T_i \tensor T_i - \E \sum_i A_i \tensor A_i$.
Reshape $v'$ into an $n \times n$ matrix $V'$.
Compute the top singular vector $v$ of $V'$.
Output $v/\|v\|$.
\end{algo}
\begin{algo}[Certification]
\label{alg:fast-certify-opt}
Run \pref{alg:fast-recovery-opt} to obtain $v$.
Let $\bS \seteq \bT - v^{\tensor 3}$.
Compute the top singular value $\lambda$ of
\[
  \sum_i S_i \tensor S_i - \E \sum_i A_i \tensor A_i\mper
\]
If $\lambda \leq O(n^{3/2} \log(n)^{1/2})$, output {\scshape certify}.
Otherwise, output {\scshape fail}.
\end{algo}
\end{minipage}}
\end{center}
The following theorem describes the behavior of \pref{alg:fast-recovery-opt} and \pref{alg:fast-certify-opt} and gives a proof of \pref{thm:optimize-certify-fast} and \pref{cor:tpca-main}.
\begin{theorem}[Formal version of \pref{thm:optimize-certify-fast}]
  \label{thm:nonlinear-unfolding}
  Let $\bT = \tau \cdot v_0^{\tensor 3} + \bA$, where $v_0 \in \R^n$ and $\bA$ has independent entries from $\cN(0,1)$.
  In other words, we are given an instance of \pref{prob:spiked-tensor}.
  Let $\tau \geq n^{3/4} \log(n)^{1/4} / \epsilon$.
  Then:
  \begin{itemize}[\quad---]
    \item With high probability, \pref{alg:fast-recovery-opt} returns $v$ with $\iprod{v,v_0}^2 \geq 1 - O(\epsilon)$.
    \item If \pref{alg:fast-certify-opt} outputs \textup{\scshape certify} then $\bT(x) \leq \tau \cdot \iprod{v,x}^3 + O(n^{3/4} \log(n)^{1/4})$ (regardless of the distribution of $\bA$). If $\bA$ is distributed as above, then \pref{alg:fast-certify-opt} outputs \textup{\scshape certify} with high probability.
    \item Both \pref{alg:fast-recovery-opt} and \pref{alg:fast-certify-opt} can be implemented in time $O(n^4 \log(1/\epsilon))$.
  \end{itemize}
\end{theorem}

The argument that \pref{alg:fast-recovery-opt} recovers a good vector in the spiked tensor model comes in three parts:
we show that under appropriate regularity conditions on the noise $\bA$ that $\sum_i T_i \tensor T_i - \E A_i \tensor A_i$ has a good singular vector,
then that with high probability in the spiked tensor model those regularity conditions hold,
and finally that the good singular vector can be used to recover the signal.

\begin{lemma}
  \label{lem:nonlinear-unfolding-general}
  Let $\bT = \tau \cdot v_0^{\tensor 3} + \bA$ be an input tensor.
  Suppose $\| \sum_i A_i \tensor A_i - \E \sum_i A_i \tensor A_i \| \leq \epsilon \tau^2$ and that $\| \sum_i v_0(i) A_i \| \leq \epsilon \tau$.
  Then the top (left or right) singular vector $v'$ of $M$ has $\iprod{v',v_0 \tensor v_0}^2 \geq 1 - O(\epsilon)$.
\end{lemma}

\begin{lemma}
  \label{lem:nonlinear-unfolding-noise}
  Let $\bT = \tau \cdot v_0^{\tensor 3} + \bA$.
  Suppose $\bA$ has independent entries from $\cN(0,1)$.
  Then with high probability we have $\| \sum_i A_i \tensor A_i - \E \sum_i A_i \tensor A_i \| \leq O(n^{3/2} \log(n)^{1/2})$ and $\| \sum_i v_0(i) A_i \| \leq O(\sqrt n)$.
\end{lemma}

\begin{lemma}
  \label{lem:nonlinear-unfolding-recursion}
  Let $v_0 \in \R^n$ and $v' \in \R^{n^2}$ be unit vectors so that
  $\iprod{v',v_0 \tensor v_0} \geq 1 - O(\epsilon)$.
  Then the top right singular vector $v$ of the $n \times n$ matrix folding
  $V'$ of $v'$ satisfies $\iprod{v,v_0} \geq 1 - O(\epsilon)$.
\end{lemma}
A similar fact to \pref{lem:nonlinear-unfolding-recursion} appears in
\cite{richard2014tensorpca}.

The proofs of \pref{lem:nonlinear-unfolding-general} and \pref{lem:nonlinear-unfolding-recursion} follow here.
The proof of \pref{lem:nonlinear-unfolding-noise} uses only standard concentration of measure arguments; we defer it to \pref{sec:concentration-bounds}.

\begin{proof}[Proof of \pref{lem:nonlinear-unfolding-general}]
  We expand $M$ as follows.
  \begin{align*}
    M & = \sum_i \tau^2 \cdot (v_0^{\tensor 3})_i \tensor (v_0^{\tensor 3})_i + \tau \cdot ((v_0^{\tensor 3})_i \tensor A_i + A_i \tensor (v_0^{\tensor 3})_i) + A_i \tensor A_i - \E A_i \tensor A_i\\
      & = \tau^2 \cdot (v_0 \tensor v_0)(v_0 \tensor v_0)^T + \tau \cdot v_0 v_0^T \tensor \sum_i v_0(i) A_i + \tau \cdot \sum_i v_0(i) A_i \tensor v_0 v_0^T + A_i \tensor A_i - \E A_i \tensor A_i\mper
  \end{align*}
  By assumption, the noise term is bounded in operator norm:
  we have $\| \sum_i A_i \tensor A_i - \E \sum_i A_i \tensor A_i \| \leq\epsilon \tau^2$.
  Similarly, by assumption the cross-term has $\| \tau \cdot v_0 v_0^T \tensor \sum_i v_0(i) A_i \| \leq \epsilon \tau^2$.
  \begin{displaymath}
    \tau \cdot \sum_i \Pu((v_0^{\tensor 3})_i \tensor A_i + A_i \tensor (v_0^{\tensor 3})_i)\Pu
      = \tau \cdot \sum_i v_0(i) \Pu (v_0 v_0^T \tensor A_i + A_i \tensor v_0 v_0^T)\Pu \mper
  \end{displaymath}
  All in all, by triangle inequality,
  \begin{displaymath}
      \left \| \tau \cdot v_0 v_0^T \tensor \sum_i v_0(i) A_i + \tau \cdot \sum_i v_0(i) A_i \tensor v_0 v_0^T + A_i \tensor A_i - \E A_i \tensor A_i \right \| \leq O(\epsilon \tau^2)\mper
  \end{displaymath}
  Again by triangle inequality,
  \begin{align*}
    \| M \| \geq (v_0 \tensor v_0)^T M (v_0 \tensor v_0) & = \tau^2 - O(\epsilon \tau^2)\mper
  \end{align*}
  Let $u, w$ be the top left and right singular vectors of $M$. We have
  \[
    u^T M w = \tau^2 \cdot \iprod{u,v_0 \tensor v_0} \iprod{w, v_0 \tensor v_0} + O(\epsilon \tau^2) \geq \tau^2 - O(\epsilon \tau^2)\mcom
  \]
  so rearranging gives the result.
\end{proof}

\begin{proof}[Proof of \pref{lem:nonlinear-unfolding-recursion}]
  Let $v_0, v', V', v$, be as in the lemma statement.
  We know $v$ is the maximizer of $\max_{\| w \|, \|w' \| = 1} w^T V' w'$.
  By assumption,
  \begin{displaymath}
    v_0^T V' v_0 = \iprod{v', v_0 \tensor v_0} \geq 1 - O(\epsilon).
  \end{displaymath}
  Thus, the top singular value of $V'$ is at least $1 - O(\epsilon)$,
  and since $\| v' \|$ is a unit vector,
  the Frobenius norm of $V'$ is $1$ and so all the rest of the singular values
  are $O(\epsilon)$.
  Expressing $v_0$ in the right singular basis of $V'$ and examining the norm
  of $V' v_0$ completes the proof.
\end{proof}

\begin{proof}[Proof of \pref{thm:nonlinear-unfolding}]
  The first claim, that \pref{alg:fast-recovery-opt} returns a good vector, follows from the previous three lemmas,
  \pref{lem:nonlinear-unfolding-general}, \pref{lem:nonlinear-unfolding-noise}, \pref{lem:nonlinear-unfolding-recursion}.
  The next, for \pref{alg:fast-certify-opt}, follows from noting that $\sum_i S_i \tensor S_i - \E \sum_i A_i \tensor A_i$ is a feasible solution to the spectral SoS dual.
  For the claimed runtime, since we are working with matrices of size $n^4$, it will be enough to show that the top singular vector of $M$ and the top singular value of $\sum_i S_i \tensor S_i - \E \sum_i A_i \tensor A_i$ can be recovered with $O(\poly \log (n))$ matrix-vector multiplies.

  In the first case, we start by observing that it is enough to find a vector $w$ which has $\iprod{w,v'} \geq 1 - \epsilon$, where $v'$ is a top singular vector of $M$.
  Let $\lambda_1, \lambda_2$ be the top two singular values of $M$.
  The analysis of the algorithm already showed that $\lambda_1 / \lambda_2 \geq \Omega(1/\epsilon)$.
  Standard analysis of the matrix power method now yields that $O(\log(1/\epsilon))$ iterations will suffice.

  We finally turn to the top singular value of $\sum_i S_i \tensor S_i - \E \sum_i A_i \tensor A_i$.
  Here the matrix may not have a spectral gap, but all we need to do is ensure that the top singular value is no more than $O(n^{3/2} \log(n)^{1/2})$.
  We may assume that some singular value is greater than $O(n^{3/2} \log (n)^{1/2})$.
  If all of them are, then a single matrix-vector multiply initialized with a random vector will discover this.
  Otherwise, there is a constant spectral gap,
  so a standard analysis of matrix power method says that within $O(\log n)$ iterations a singular value greater than $O(n^{3/2} \log(n)^{1/2})$ will be found, if it exists.
\end{proof}

\subsection{Nearly-Linear-Time Recovery via Tensor Unfolding and Spectral SoS}
On input $\bT = \tau \cdot v_0^{\tensor 3} + \bA$, where as usual $v_0 \in \R^n$ and $\bA$ has independent entries from $\cN(0,1)$, Montanari and Richard's Tensor Unfolding algorithm computes the top singular vector $u$ of the squarest-possible flattening of $T$ into a matrix.
It then extracts $v$ with $\iprod{v,v_0}^2 \geq 1 - o(1)$ from $u$ with a second singular vector computation.
\begin{center}
\fbox{\begin{minipage}{\textwidth}
\begin{center}\textbf{Recovery with $TT^T$, a.k.a. Tensor Unfolding}\end{center}
\noindent\textbf{Input:} $\bT = \tau \cdot v_0^{\otimes 3} + \bA$, where $v_0 \in \R^n$ and $\bA$ is a $3$-tensor. \\
\noindent\textbf{Goal:} Find $v \in \R^n$ with $\abs{\iprod{v,v_0}} \geq 1 - o(1)$.
\begin{algo}[Recovery]
\label{alg:fast-recover-subopt}
Compute the top eigenvector $v$ of $M \seteq T^T T$.
Output $v$.
\end{algo}
\end{minipage}}
\end{center}

We show that this algorithm successfully recovers a vector $v$ with $\iprod{v,v_0}^2 \geq 1 - O(\epsilon)$ when $\tau \geq n^{3/4}/\epsilon$.
Montanari and Richard conjectured this but were only able to show it when $\tau \geq n$.
We also show how to implement the algorithm in time $\tO(n^3)$, that is to say, in time nearly-linear in the input size.

Despite its a priori simplicity, the analysis of \pref{alg:fast-recover-subopt} is more subtle than for any of our other algorithms.
This would not be true for even-order tensors, for which the square matrix unfolding tensor has one singular value asymptotically larger than all the rest,
and indeed the corresponding singular vector is well-correlated with $v_0$.
However, in the case of odd-order tensors the unfolding has no spectral gap.
Instead, the signal $v_0$ has some second-order effect on the spectrum of the matrix unfolding, which is enough to recover it.

We first situate this algorithm in the SoS framework.
In the previous section we examined the feasible solution $\sum_i T_i \tensor T_i - \E \sum_i A_i \tensor A_i$ to the spectral SoS relaxation of $\max_{\|x\| = 1} \bT(x)$.
The tensor unfolding algorithm works by examining the top singular vector of the flattening $T$ of $\bT$,
which is the top eigenvector of the $n \times n$ matrix $M = T^TT$, which in turn has the same spectrum as the $n^2 \times n^2$ matrix $TT^T$.
The latter is also a feasible dual solution to the spectral SoS relaxation of $\max_{\|x\| = 1} \bT(x)$.
However, the bound it provides on $\max_{\|x\| = 1} \bT(x)$ is much worse than that given by $\sum_i T_i \tensor T_i$.
The latter, as we saw in the preceding section, gives the bound $O(n^{3/4} \log(n)^{1/4})$.
The former, by contrast, gives only $O(n)$, which is the operator norm of a random $n^2 \times n$ matrix (see \pref{lem:rectangular-concentration}).
This $n$ versus $n^{3/4}$ is the same as the gap between Montanari and Richard's conjectured bound and what they were able to prove.

\begin{theorem}
\label{thm:mr-success-scaled}
  For an instance of \pref{prob:spiked-tensor} with $\tau \geq n^{3/4}/\epsilon$, with high probability \pref{alg:fast-recover-subopt} recovers a vector $v$ with $\iprod{v,v_0}^2 \geq 1 - O(\epsilon)$.
  Furthermore, \pref{alg:fast-recover-subopt} can be implemented in time $\tO(n^3)$.
\end{theorem}

\begin{lemma}
\label{lem:mr-success}
  Let $\bT = \tau \cdot v_0^{\tensor 3} + \bA$ where $v_0 \in \R^n$ is a unit vector,
  so an instance of \pref{prob:spiked-tensor}.
  Suppose $\bA$ satisfies $A^TA = C \cdot \Id_{n \times n} + E$ for some $C \geq 0$ and $E$ with $\| E \| \leq \epsilon \tau^2$ and that $\|A^T(v_0 \tensor v_0)\| \leq \epsilon \tau$.
  Let $u$ be the top left singular vector of the matrix $T$.
  Then $\iprod{v_0,u}^2 \geq 1 - O(\epsilon)$.
\end{lemma}
\begin{proof}
  The vector $u$ is the top eigenvector of the $n \times n$ matrix $T T^T$, which is also the top eigenvector of $M \seteq T T^T -  C \cdot \Id$.
  We expand:
  \begin{align*}
    u^T M u & = u^T \Brac{ \tau^2 \cdot v_0v_0^T + \tau \cdot v_0(v_0 \tensor v_0)^T A
          + \tau\cdot A^T (v_0 \tensor v_0)v_0^T + E}u\\
      & = \tau^2 \cdot \iprod{u,v_0}^2 + u^T \Brac{\tau \cdot v_0(v_0 \tensor v_0)^T A
          + \tau\cdot A^T (v_0 \tensor v_0)v_0^T + E}u\\
          & \leq \tau^2 \iprod{u,v_0}^2 + O(\epsilon \tau^2)\mper
  \end{align*}
  Again by triangle inequality, $u^T M u \geq v_0^T M v = \tau^2 - O(\epsilon \tau^2)$.
  So rearranging we get $\iprod{u,v_0}^2 \geq 1 - O(\epsilon)$ as desired.
\end{proof}

The following lemma is a consequence of standard matrix concentration inequalities;
we defer its proof to \pref{sec:concentration-bounds}, \pref{lem:mr-concentration-general}.
\begin{lemma}
\label{lem:mr-concentration}
  Let $\bA$ have independent entries from $\cN(0,1)$.
  Let $v_0 \in \R^n$ be a unit vector.
  With high probability, the matrix $A$ satisfies $A^T A = n^2 \cdot \Id + E$ for some $E$ with $\| E \| \leq O(n^{3/2})$ and $\| A^T (v_0 \tensor v_0) \| \leq O(\sqrt{ n \log n})$.
\end{lemma}

The final component of a proof of \pref{thm:mr-success-scaled} is to show how it can be implemented in time $\tO(n^3)$.
Since $M$ factors as $T^T T$, a matrix-vector multiply by $M$ can be implemented in time $O(n^3)$.
Unfortunately, $M$ does not have an adequate eigenvalue gap to make matrix power method efficient.
As we know from \pref{lem:mr-concentration}, suppressing $\epsilon$s and constants, $M$ has eigenvalues in the range $n^2 \pm n^{3/2}$.
Thus, the eigenvalue gap of $M$ is at most $g = O(1 + 1/\sqrt n)$.
For any number $k$ of matrix-vector multiplies with $k \leq n^{1/2 - \delta}$, the eigenvalue gap will become at most $(1 + 1/\sqrt n)^{n^{1/2} - \delta}$, which is subconstant.
To get around this problem, we employ a standard trick to improve spectral gaps of matrices close to $C \cdot \Id$: remove $C \cdot \Id$.

\begin{lemma}
\label{lem:mr-runtime}
  Under the assumptions of \pref{thm:mr-success-scaled}, \pref{alg:fast-recover-subopt} can be implemented in time $\tO(n^3)$ (which is linear in the input size, $n^3$).
\end{lemma}
\begin{proof}
  Note that the top eigenvector of $M$ is the same as that of $M - n^2 \cdot \Id$. The latter matrix, by the same analysis as in \pref{lem:mr-success}, is given by
  \[
    M - n^2 \cdot \Id = \tau^2 \cdot v_0 v_0^T + M'
  \]
  where $\|M'\| = O(\epsilon \tau^2)$.
  Note also that a matrix-vector multiply by $M - n^2 \cdot \Id$ can still be done in time $O(n^3)$.
  Thus, $M - n^2 \cdot \Id$ has eigenvalue gap $\Omega(1/\epsilon)$, which is enough so that the whole algorithm runs in time $\tO(n^3)$.
\end{proof}

\begin{proof}[Proof of \pref{thm:mr-success-scaled}]
  Immediate from \pref{lem:mr-success}, \pref{lem:mr-concentration}, and \pref{lem:mr-runtime}.
\end{proof}

\subsection{Fast Recovery in the Semi-Random Model}
There is a qualitative difference between the aggregate matrix statistics needed by our certifying algorithms (\pref{alg:sdp-recovery}, \pref{alg:sdp-certify}, \pref{alg:fast-recovery-opt}, \pref{alg:fast-certify-opt}) and those needed by rounding the tensor unfolding solution spectral SoS \pref{alg:fast-recover-subopt}. In a precise sense, the needs of the latter are greater.
The former algorithms rely only on first-order behavior of the spectra of a tensor unfolding, while the latter relies on second-order spectral behavior.
Since it uses second-order properties of the randomness, \pref{alg:fast-recover-subopt} fails in the semi-random model.

\begin{theorem}
  Let $\bT = \tau \cdot v_0^{\tensor 3} + \bA$, where $v_0 \in \R^n$ is a unit
  vector and $\bA$ has independent entries from $\cN(0,1)$.
  There is $\tau = \Omega(n^{7/8})$ so that with high probability there is an
  adversarial choice of $Q$ with $\|Q - \Id\| \leq O(n^{-1/4})$ so that the
  matrix $(TQ)^T TQ = n^2 \cdot \Id$.
  In particular, for such $\tau$, \pref{alg:fast-recover-subopt} cannot recover
  the signal $v_0$.
\end{theorem}
\begin{proof}
  Let $M$ be the $n \times n$ matrix $M \seteq T^T T$.
  Let $Q = n \cdot M^{-1/2}$.
  It is clear that $(TQ)^T TQ = n^2 \Id$.
  It suffices to show that $\|Q - \Id\| \leq n^{1/4}$ with high probability.
  We expand the matrix $M$ as
  \[
    M = \tau^2 \cdot v_0 v_0^T + \tau \cdot v_0(v_0 \tensor v_0)^T A + \tau \cdot A^T (v_0 \tensor v_0) v_0^T + A^T A\mper
  \]
  By \pref{lem:mr-concentration}, $A^T A = n^2 \cdot \Id + E$ for some $E$ with $\|E\| \leq O(n^{3/2})$
  and $\|A^T (v_0 \tensor v_0)\| \leq O(\sqrt{n \log n})$, both with high probability.
  Thus, the eigenvalues of $M$ all lie in the range $n^2 \pm n^{1 + 3/4}$.
  The eigenvalues of $Q$ in turn lie in the range
  \[
    \frac n {(n^2 \pm O(n^{1 + 3/4}))^{1/2}} = \frac 1 {(1 \pm O(n^{-1/4}))^{1/2}} = \frac 1 {1 \pm O(n^{1/4})}\mper
  \]
  Finally, the eigenvalues of $Q - \Id$ lie in the range $\frac 1 {1 \pm O(n^{1/4})} - 1 = \pm O(n^{-1/4})$, so we are done.
\end{proof}

The argument that that \pref{alg:fast-recovery-opt} and \pref{alg:fast-certify-opt} still succeed in the semi-random model is routine; for completeness we discuss here the necessary changes to the proof of \pref{thm:nonlinear-unfolding}.
The non-probabilistic certification claims made in \pref{thm:nonlinear-unfolding} are independent of the input model, so we show that \pref{alg:fast-recovery-opt} still finds the signal with high probability and that \pref{alg:fast-certify-opt} still fails only with only a small probability.

\begin{theorem}
  \label{thm:fast-matrix-semirandom-success}
  In the semi-random model, $\epsilon \geq n^{-1/4}$ and $\tau \geq n^{3/4} \log(n)^{1/4} / \epsilon$, with high probability, \pref{alg:fast-recovery-opt} returns $v$ with $\iprod{v,v_0}^2 \geq 1 - O(\epsilon)$ and \pref{alg:fast-certify-opt} outputs \textup{\scshape certify}.
\end{theorem}
\begin{proof}
  We discuss the necessary modifications to the proof of \pref{thm:nonlinear-unfolding}.
  Since $\epsilon \geq n^{-1/4}$, we have that $\|(Q - \Id)v_0\| \leq O(\epsilon)$.
  It suffices then to show that the probabilistic bounds in \pref{lem:nonlinear-unfolding-noise} hold with $A$ replaced by $AQ$.
  Note that this means each $A_i$ becomes $A_i Q$.
  By assumption, $\|Q \tensor Q - \Id \tensor \Id\| \leq O(\epsilon)$, so the probabilistic bound on $\|\sum_i A_i \tensor A_i = \E \sum_i A_i \tensor A_i \|$ carries over to the semi-random setting.
  A similar argument holds for $\sum_i v_0(i) A_i Q$, which is enough to complete the proof.
\end{proof}

\subsection{Fast Recovery with Symmetric Noise}
We suppose now that $\bA$ is a symmetric Gaussian noise tensor;
that is, that $\bA$ is the average of $\bA_0^{\pi}$ over all $\pi \in \cS_3$,
for some order-$3$ tensor $\bA_0$ with iid standard Gaussian entries.

It was conjectured by Montanari and Richard \cite{richard2014tensorpca} that the
tensor unfolding technique can recover the signal vector $v_0$ in the
single-spike model $\bT = \tau v_0^{\ot 3} + \bA$ with signal-to-noise ratio
$\tau \ge \tilde\Omega(n^{3/4})$ under both asymmetric and symmetric noise.

Our previous techniques fail in this symmetric noise scenario due to lack of
independence between the entries of the noise tensor.
However, we sidestep that issue here by restricting our attention to an
asymmetric block of the input tensor.

The resulting algorithm is not precisely identical to the tensor unfolding
algorithm investigated by Montanari and Richard,
but is based on tensor unfolding with only superficial modifications.

\begin{center}
\fbox{\begin{minipage}{\textwidth}
\begin{center}\textbf{Fast Recovery under Symmetric Noise}\end{center}
\noindent\textbf{Input:} $\bT = \tau \cdot v_0^{\otimes 3} + \bA$, where $v_0 \in \R^n$ and $\bA$ is a $3$-tensor. \\
\noindent\textbf{Goal:} Find $v \in \R^n$ with $\abs{\iprod{v,v_0}} \geq 1 - o(1)$.
\begin{algo}[Recovery]
\label{alg:fast-recover-sym}
Take $X,Y,Z$ a random partition of $[n]$, and $R$ a random rotation of $\R^n$.
Let $P_X$, $P_Y$, and $P_Z$ be the diagonal projectors onto the coordinates indicated by $X$, $Y$, and $Z$.
Let $\bU \seteq R^{\ot 3} \bT$, so that we have the matrix unfolding $U \seteq (R \ot R)TR^T$
Using the matrix power method, compute the top singular vectors $v_X$, $v_Y$, and $v_Z$ respectively of the matrices
\begin{align*}
  M_X &\seteq P_XU^T(P_Y \ot P_Z)UP_X - n^2/9 \cdot \Id
\\M_Y &\seteq P_YU^T(P_Z \ot P_X)UP_Y - n^2/9 \cdot \Id
\\M_Z &\seteq P_ZU^T(P_X \ot P_Y)UP_Z - n^2/9 \cdot \Id \mper
\end{align*}
Output the normalization of $R^{-1}(v_X + v_Y + v_Z)$.
\end{algo}
\end{minipage}}
\end{center}

\begin{remark}[Implementation of \pref{alg:fast-recover-sym} in nearly-linear time.]
It is possible to implement each iteration of the matrix power method in
\pref{alg:fast-recover-sym} in linear time.
We focus on multiplying a vector by $M_X$ in linear time; the other cases follow similarly.

We can expand $M_X = P_X R T^T (R \ot R)^T (P_Y \ot P_Z) (R \ot R) T R^T P_X - n^2/9 \cdot \Id$.
It is simple enough to multiply an $n$-dimensional vector by $P_X$, $R$, $R^T$, $T$, and $\Id$
in linear time.
Furthermore multiplying an $n^2$-dimensional vector by $T^T$ is also a simple linear time operation.
The trickier part lies in multiplying an $n^2$-dimensional vector, say $v$, by
the $n^2$-by-$n^2$ matrix $(R \ot R)^T (P_Y \ot P_Z) (R \ot R)$.

To accomplish this, we simply reflatten our tensors.
Let $V$ be the $n$-by-$n$ matrix flattening of $v$.
Then we compute the matrix $R^TP_YR \cdot V \cdot R^T{P_Z}^TR$,
and return its flattening back into an $n^2$-dimensional vector,
and this will be equal to $(R \ot R)^T (P_Y \ot P_Z) (R \ot R)\,v$.
This equivalence follows by taking the singular value decomposition
$V = \sum_i \lambda_i u_iw_i^T$,
and noting that $v = \sum_i \lambda_i u_i \ot w_i$.
\end{remark}

\begin{lemma}
\label{lem:random-projection}
Given a unit vector $u \in \R^n$, a random rotation $R$ over $\R^n$,
and a projection $P$ to an $m$-dimensional subspace,
with high probability
\[ \Big| \|PRu\|^2 - m/n \Big| \le O(\sqrt{m/n^2}\log m) \mper \]
\end{lemma}
\begin{proof}
Let $\gamma$ be a random variable distributed as the norm of a vector in $\R^n$
with entries independently drawn from $\cN(0, 1/n)$.
Then because Gaussian vectors are rotationally invariant
and $Ru$ is a random unit vector,
the coordinates of $\gamma Ru$ are independent and Gaussian in any orthogonal basis.

So $\gamma^2\|PRu\|^2$ is the sum of the squares of $m$ independent variables
drawn from $\cN(0, 1/n)$.
By a Bernstein inequality, $\big|\gamma^2\|PRu\|^2 - m/n\big| \le O(\sqrt{m/n^2}\log m)$
with high probability.
Also by a Bernstein inequality, $\gamma^2 - 1 < O(\sqrt{1/n} \log n)$ with high probability.
\end{proof}

\begin{theorem}
For $\tau \geq n^{3/4}/\epsilon$, with high probability, \pref{alg:fast-recover-sym}
recovers a vector $v$ with $\iprod{v,v_0} \ge 1 - O(\epsilon)$ when $\bA$ is a
symmetric Gaussian noise tensor (as in \pref{prob:spiked-tensor-symm})
and $\epsilon \geq \log(n) / \sqrt{n}$.

Furthermore the matrix power iteration steps in \pref{alg:fast-recover-sym}
each converge within $\tO(-\log(\epsilon))$ steps, so that the algorithm
overall runs in almost linear time $\tO(n^3\log(1/\epsilon))$.
\end{theorem}
\begin{proof}
Name the projections $U_X \seteq (P_Y \ot P_Z)UP_X$,
$U_Y \seteq (P_Z \ot P_X)UP_Y$, and $U_Z \seteq (P_X \ot P_Y)UP_Z$.

First off, $\bU = \tau(Rv_0)^{\ot 3} + \bA'$ where $\bA'$ is a symmetric
Gaussian tensor (distributed identically to $\bA$).
This follows by noting that multiplication by $R^{\ot 3}$ commutes with permutation of indices,
so that $(R^{\ot 3}\bB)^{\pi} = R^{\ot 3}\bB^{\pi}$,
where we let $\bB$ be the asymmetric Gaussian tensor
so that $\bA = \sum_{\pi \in \cS_3} \bB^{\pi}$.
Then $\bA' = R^{\ot 3}\sum_{\pi \in \cS_3}\bB^{\pi} = \sum_{\pi \in \cS_3}(R^{\ot 3}\bB)^{\pi}$.
This is identically distributed with $\bA$, as follows from the rotational symmetry of $\bB$.

Thus $U_X = \tau (P_Y \ot P_Z)(R \ot R)(v_0 \ot v_0)(P_XRv_0)^T + (P_Y \ot P_Z)A'P_X$, and
\begin{align}
M_X + n^2/9 \cdot \Id &= U_X^T U_X  \nonumber
\\&= \tau^2 \|P_YRv_0\|^2\|P_ZRv_0\|^2(P_XRv_0)(P_XRv_0)^T
\label{ex:sym-signal}
\\&\hphantom{=}+ \tau (P_XRv_0) (v_0 \ot v_0)^T(R \ot R)^T(P_Y \ot P_Z)A'P_X
\label{ex:sym-cross-1}
\\&\hphantom{=}+ \tau P_X{A'}^T(P_Y \ot P_Z)(R \ot R)(v_0 \ot v_0)(P_XRv_0)^T
\label{ex:sym-cross-2}
\\&\hphantom{=}+ P_X{A'}^T(P_Y \ot P_Z)A'P_X \mper
\label{ex:sym-noise}
\end{align}

Let $S$ refer to Expression \ref{ex:sym-signal}.
By \pref{lem:random-projection}, $\big| \|PRv_0\|^2 - \tfrac{1}{3} \big| < O(\sqrt{1/n}\log n)$
with high probability for $P \in \{P_X, P_Y, P_Z\}$.
Hence $S = (\tfrac{1}{9} \pm O(\sqrt{1/n}\log n))\tau^2(P_XRv_0)(P_XRv_0)^T$
and $\|S\| = (\tfrac{1}{27} \pm O(\sqrt{1/n}\log n))\tau^2$.

Let $C$ refer to Expression \ref{ex:sym-cross-1} so that Expression \ref{ex:sym-cross-2} is $C^T$.
Let also $A'' = (P_Y \ot P_Z)A'P_X$.
Note that, once the identically-zero rows and columns of $A''$ are removed,
$A''$ is a matrix of iid standard Gaussian entries.
Finally, let $v'' = P_YRv_0 \ot P_ZRv_0$.
By some substitution and by noting that $\|P_XR\| \le 1$, we have that
$\|C\| \le \tau\, \|v_0{v''}^TA''\|$.
Hence by \pref{lem:mr-concentration-general}, $\|C\| \le O(\epsilon\tau^2)$.

Let $N$ refer to Expression \ref{ex:sym-noise}.
Note that $N = {A''}^TA''$.
Therefore by \pref{lem:mr-concentration}, $\|N - n^2/9 \cdot \Id\| \le O(n^{3/2})$.

Thus $M_X = S + C + (N - n^2/9 \cdot \Id)$, so that $\|M_X - S\| \le O(\epsilon\tau^2)$.
Since $S$ is rank-one and has $\|S\| \ge \Omega(\tau^2)$,
we conclude that matrix power iteration converges in $\tO(-\log \epsilon)$ steps.

The recovered eigenvector $v_X$ satisfies $\iprod{v_X, M_X v_X} \ge \Omega(\tau^2)$
and $\iprod{v_X, (M_X - S)v_X} \le O(\epsilon\tau^2)$ and therefore
$\iprod{v_X, Sv_X} = (\tfrac{1}{27}\pm O(\epsilon + \sqrt{1/n}\log n))\tau^2$.
Substituting in the expression for $S$,
we conclude that $\iprod{P_XRv_0, v_X} = (\tfrac{1}{\sqrt{3}}\pm O(\epsilon + \sqrt{1/n}\log n))$.

The analyses for $v_Y$ and $v_Z$ follow in the same way.
Hence
\begin{align*}
  \iprod{v_X + v_Y + v_Z, Rv_0} %
  &= \iprod{v_X, P_XRv_0} + \iprod{v_Y, P_YRv_0} + \iprod{v_Z, P_ZRv_0}
 \\&\ge \sqrt{3} - O(\epsilon + \sqrt{1/n}\log n) \mper
\end{align*}
At the same time, since $v_X$, $v_Y$, and $v_Z$ are each orthogonal to each other,
$\|v_X + v_Y + v_Z\| = \sqrt{3}$.
Hence with the output vector being $v \seteq R^{-1}(v_X+v_Y+v_Z)/\|v_X+v_Y+v_Z\|$,
we have
\[ \iprod{v, v_0} = \iprod{Rv,Rv_0} = \tfrac{1}{\sqrt{3}}\iprod{v_X+v_Y+v_Z,Rv_0}
  \ge 1 - O(\epsilon + \sqrt{1/n}\log n) \mper\]
\end{proof}

\subsection{Numerical Simulations}

\begin{figure}[!htb]
  \includegraphics[scale=0.45,trim=0in 0.3in 0in 0.2in,clip]{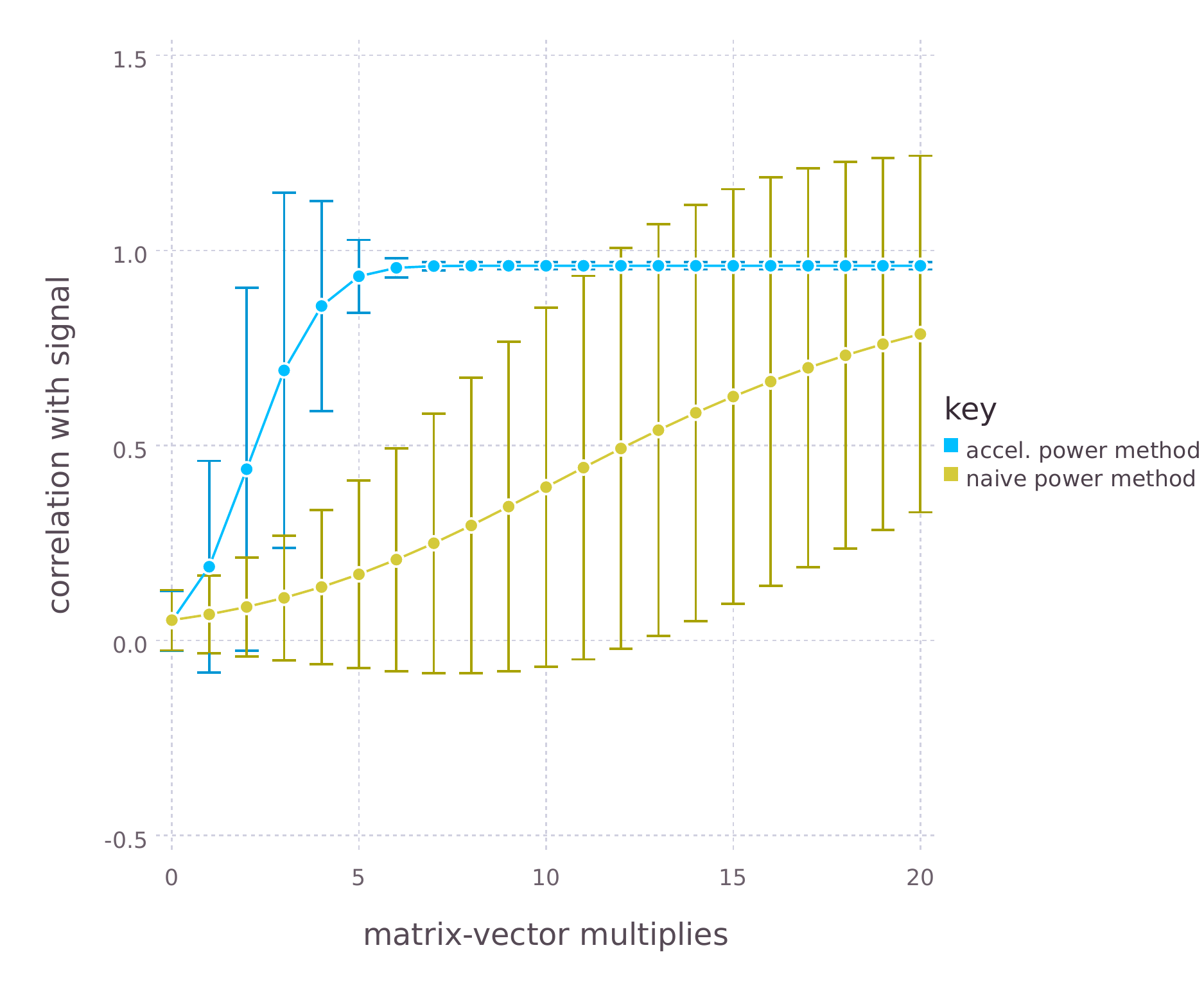}
  \includegraphics[scale=0.45,trim=0in 0.3in 0in 0.2in,clip]{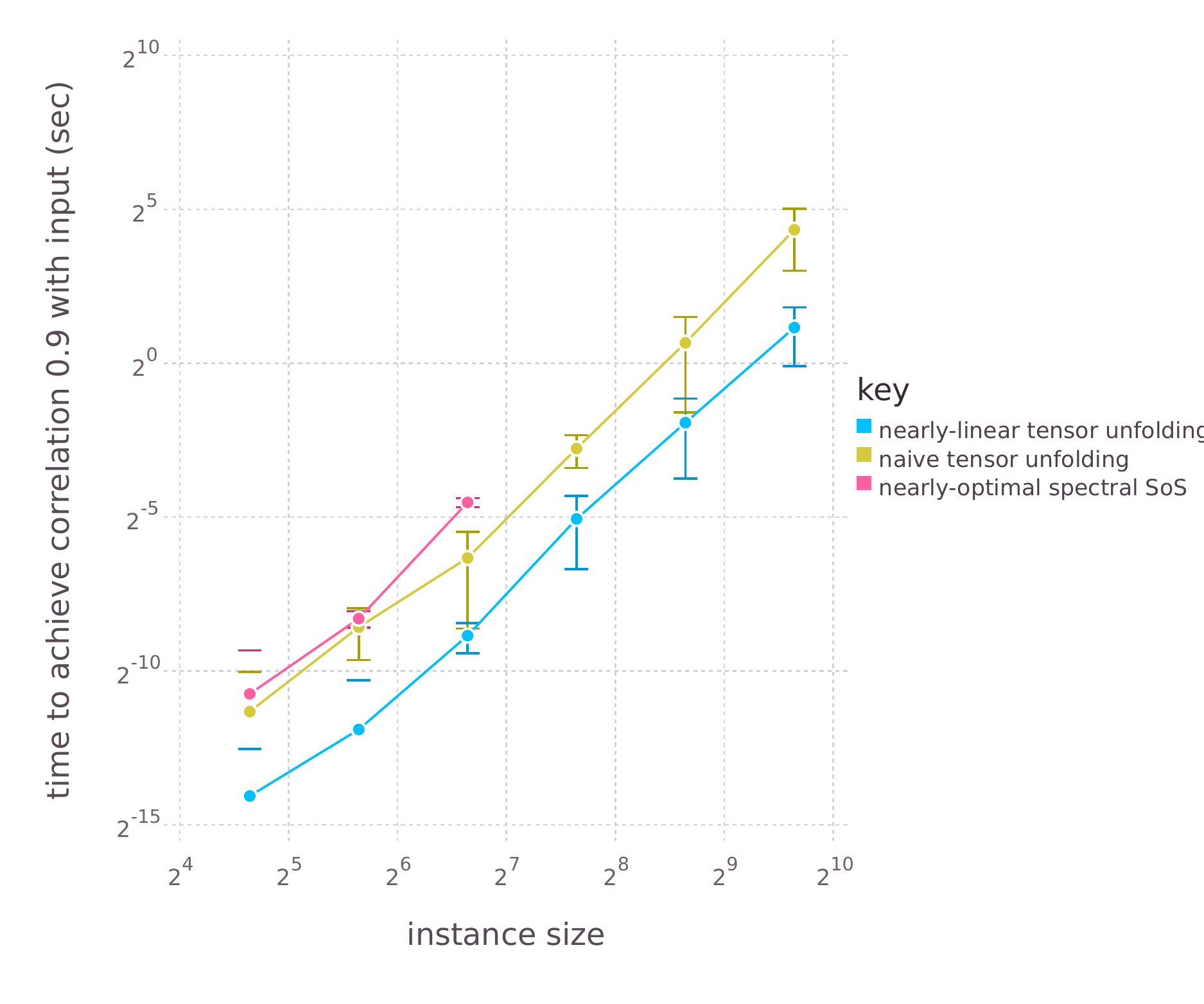}
  \caption{Numerical simulation of \pref{alg:fast-recovery-opt}
  (``Nearly-optimal spectral SoS'' implemented with matrix power method),
  and two implementations of \pref{alg:fast-recover-subopt}
  (``Accelerated power method''/``Nearly-linear tensor unfolding''
  and ``Naive power method''/``Naive tensor unfolding''.
  Simulations were run in Julia on a Dell Optiplex 7010 running Ubuntu 12.04
  with two Intel Core i7 3770 processors at 3.40 ghz and 16GB of RAM.
  Plots created with Gadfly.
  Error bars denote $95\%$ confidence intervals.
  Matrix-vector multiply experiments were conducted with $n = 200$.
  Reported matrix-vector multiply counts are the average of 50 independent trials.
  Reported times are in cpu-seconds and are the average of 10 independent trials.
  Note that both axes in the right-hand plot are log scaled.
  \label{fig:simulations}
  }
\end{figure}

We report now the results of some basic numerical simulations of the algorithms from this section.
In particular, we show that the asymptotic running time differences among
\pref{alg:fast-recovery-opt}, \pref{alg:fast-recover-subopt} implemented
\naively, and the linear-time implementation of
\pref{alg:fast-recover-subopt} are apparent at reasonable values of $n$,
e.g. $n = 200$.

Specifics of our experiments are given in \pref{fig:simulations}.
We find pronounced differences between all three algorithms.
The \naive implementation of \pref{alg:fast-recover-subopt} is markedly slower than the linear implementation, as measured either by number of matrix-vector multiplies or processor time.
\pref{alg:fast-recovery-opt} suffers greatly from the need to construct an $n^2 \times n^2$ matrix; although we do not count the time to construct this matrix against its reported running time, the memory requirements are so punishing that we were unable to collect data beyond $n = 100$ for this algorithm.

\section{Lower Bounds}
\label{sec:lower-bounds}
We will now prove lower bounds on the performance of degree-$4$ SoS on random
instances of the degree-$4$ and degree-$3$ homogeneous polynomial maximization problems.
As an application, we show that our analysis of degree-$4$ for Tensor PCA is tight up to a small logarithmic factor in the signal-to-noise ratio.

\begin{theorem}[Part one of formal version of \pref{thm:lb-intro-informal}]
  \label{thm:lower-bound-tensor-pca-4}
  There is $\tau = \Omega(n)$ and a function $\eta : \bA \mapsto \{x\}$ mapping $4$-tensors to degree-$4$ pseudo-distributions satisfying $\{ \|x\|^2  = 1\}$
  so that for every unit vector $v_0$,
  if $\bA$ has unit Gaussian entries,
  then, with high probability over random choice of $\bA$,
  the pseudo-expectation $\pE_{x \sim \eta(\bA)} \tau \cdot \iprod{v_0,x}^4 + \bA(x)$
  is maximal up to constant factors among $\pE \tau \cdot \iprod{v_0,y}^4 + \bA(y)$ over all degree-4 pseudo-distributions $\{ y \}$ satisfying $\{ \|y\|^2 = 1 \}$.
\end{theorem}

\begin{theorem}[Part two of formal version of \pref{thm:lb-intro-informal}]
  \label{thm:lower-bound-tensor-pca}
  There is $\tau = \Omega(n^{3/4}/(\log n)^{1/4})$ and
  a function $\eta : \bA \mapsto \{x\}$ mapping $3$-tensors to degree-$4$ pseudo-distributions satisfying $\{ \|x\|^2  = 1\}$
  so that for every unit vector $v_0$,
  if $\bA$ has unit Gaussian entries,
  then, with high probability over random choice of $\bA$,
  the pseudo-expectation $\pE_{x \sim \eta(\bA)} \tau \cdot \iprod{v_0,x}^3 + \bA(x)$
  is maximal up to logarithmic factors among $\pE \tau \cdot \iprod{v_0,y}^3 + \bA(y)$ over all degree-4 pseudo-distributions $\{ y \}$ satisfying $\{ \|y\|^2 = 1 \}$.
\end{theorem}

The existence of the maps $\eta$ depending only on the random part $\bA$ of the tensor PCA input $v_0^{\tensor 3} + \bA$ formalizes the claim from \pref{thm:lb-intro-informal} that no algorithm can reliably recover $v_0$ from the pseudo-distribution $\eta(\bA)$.

Additionally, the lower-bound construction holds for the symmetric noise model also:
the input tensor $\bA$ is symmetrized wherever it occurs in the construction, so it
does not matter if it had already been symmetrized beforehand.

The rest of this section is devoted to proving these theorems, which we eventually accomplish in \pref{sec:formal}.

\subsubsection{Discussion and Outline of Proof}
Given a random $3$-tensor $\bA$, we will take the degree-$3$ pseudo-moments of our $\eta(\bA)$ to be $\epsilon \bA$, for some small $\epsilon$, so that $\pE_{x \sim \eta(\bA)} \bA(x)$ is large.
The main question is how to give degree-$4$ pseudo-moments to go with this.
We will construct these from $AA^T$ and its permutations as a $4$-tensor under the action of $\cS_4$.

We have already seen that a spectral \emph{upper bound} on one of these permutations,
$\sum_i A_i \tensor A_i$, provides a performance guarantee for degree-$4$ SoS
optimization of degree-$3$ polynomials.
It is not a coincidence that this SoS lower bound depends on the negative
eigenvalues of the permutations of $AA^T$.
Running the argument for the upper bound in reverse, a pseudo-distribution
$\{ x \}$ satisfying $\{ \|x\|_2^2 = 1 \}$ and with $\pE \bA(x)$ large must
(by pseudo-Cauchy-Schwarz) also have
$\pE \iprod{x^{\tensor 2}, \Paren{\sum_i A_i \tensor A_i} x^{\tensor 2}}$ large.
The permutations of $AA^T$ are all matrix representations of that same polynomial,
$\iprod{x^{\tensor 2}, \Paren{\sum_i A_i \tensor A_i} x^{\tensor 2}}$.
Hence $\pE \bA(x)$ will be large only if the matrix
representation of the pseudo-distribution $\{ x \}$ is well correlated with
the permutations of $AA^T$.
Since this matrix representation will also need to be positive-semidefinite, control on the spectra of permutations of $AA^T$ is therefore the key to our approach.

The general outline of the proof will be as follows:
\begin{enumerate}
\item Construct a pseudo-distribution that is well correlated with the
  permutations of $AA^T$ and gives a large value to $\pE \bA(x)$, but
  which is not on the unit sphere.
\item Use a procedure modifying the first and second degree moments of
  the pseudo-distribution to force it onto a sphere,
  at the cost of violating the condition that $\pE p(X)^2 \ge 0$ for all
  $p \in \R[x]_{\le 2}$, then rescale so it lives on the unit sphere.
  Thus, we end up with an object that is no longer a valid pseudo-distribution
  but a more general linear functional $\cL$ on polynomials.
\item Quantitatively bound the failure of $\cL$ to be a
  pseudo-distribution, and repair it by statistically mixing the
  almost-pseudo-distribution with a small amount of the uniform distribution
  over the sphere. Show that $\pE \bA(x)$ is still large for this new
  pseudo-distribution over the unit sphere.
\end{enumerate}

But before we can state a formal version of our theorem, we will need a few
facts about polynomials, pseudo-distributions, matrices, vectors,
and how they are related by symmetries under actions of permutation groups.

\subsection{Polynomials, Vectors, Matrices, and Symmetries, Redux}

Here we further develop the matrix view of SoS presented in
\pref{sec:sos-matrix-view}.

We will need to use general linear functionals $\cL : \R[x]_{\leq 4}\rightarrow \R$ on polynomials
as an intermediate step between matrices and pseudo-distributions.
Like pseudo-distributions, each such linear-functional $\cL$ has a unique matrix representation $M_{\cL}$ satisfying certain maximal symmetry constraints.
The matrix $M_{\cL}$ is positive-semidefinite if and only if $\cL p(x)^2 \geq 0$ for every $p$.
If $\cL$ satisfies this and $\cL 1 = 1$, then $\cL$ is a pseudo-expectation, and $M_{\cL}$ is the matrix representation of the corresponding pseudo-distribution.

\subsubsection{Matrices for Linear Functionals and Maximal Symmetry}
Let $\cL : \R[x]_{\leq d} \rightarrow \R$.
$\cL$ can be represented as an $n^{\tuples(d)} \times n^{\tuples(d)}$
matrix indexed by all $d'$-tuples over $[n]$ with $d' \leq d/2$.
For tuples $\alpha, \beta$, this matrix $M_{\cL}$ is given by
\begin{displaymath}
  M_{\cL}[\alpha, \beta] \defeq \cL x^\alpha x^\beta\mper
\end{displaymath}
For a linear functional $\cL : \R[x]_{\leq d} \rightarrow \R$, a polynomial
$p(x) \in \R[x]_{\leq d}$, and a matrix representation $M_p$ for $p$ we
thus have $\iprod{M_{\cL}, M_p} = \cL p(x)$.

A polynomial in $\R[x]_{\leq d}$ may have many matrix representations, while
for us, a linear functional $\cL$ has just one: the matrix $M_{\cL}$.
This is because in our definition we have required that $M_{\cL}$ obey the
constraints
\begin{displaymath}
  M_{\cL}[\alpha,\beta] = M_{\cL}[\alpha', \beta']
  \qquad \mbox{ when } \qquad x^\alpha x^\beta = x^{\alpha'} x^{\beta'}\mper
\end{displaymath}
in order that they assign consistent values to each representation of the same
polynomial.
We call such matrices \emph{maximally symmetric}
(following Doherty and Wehner \cite{doherty2012convergence}).

We have particular interest in the maximally-symmetric version of the identity
matrix. The degree-$d$ symmetrized identity matrix $\symId$ is the unique
maximally symmetric matrix so that
\begin{equation}\label{eq:symId}
  \iprod{x^{\tensor d/2}, \symId x^{\tensor d/2}} = \| x \|_2^d.
\end{equation}
The degree $d$ will always be clear from context.

In addition to being a matrix representation of the polynomial $\|x\|_2^d$,
the maximally symmetric matrix $\symId$ also serves a dual purpose as a linear functional.
We will often be concerned with the expectation operator $\E^{\mu}$
for the uniform distribution over the $n$-sphere, and indeed
for every polynomial $p(x)$ with matrix representation $M_p$,
\[ \E\nolimits^{\mu}p(x) = \frac{1}{n^2 +2n}\iprod{\symId, M_p}\mcom \]
and so $\symId/(n^2+2n)$ is the unique matrix representation of
$\E^{\mu}$.

\subsubsection{The Monomial-Indexed (i.e. Symmetric) Subspace}
We will also require vector representations of polynomials.
We note that $\R[x]_{\leq d/2}$ has a canonical embedding into $\R^{\tuples(d)}$
as the subspace given by the following family of constraints,
expressed in the basis of $d'$-tuples for $d' \leq d/2$:
\begin{displaymath}
  \R[x]_{\leq d/2} \simeq \{ p \in \R^{\tuples(d)}
    \mbox{ such that } p_\alpha = p_{\alpha'}
    \mbox{ if $\alpha'$ is a permutation of $\alpha$ } \}\mper
\end{displaymath}
We let $\Pi$ be the projector to this subspace.
For any maximally-symmetric $M$ we have $\Pi M \Pi = M$,
but the reverse implication is not true
(for readers familiar with quantum information: any $M$ which has
$M = \Pi M \Pi$ is \emph{Bose-symmetric}, but may not be \emph{PPT-symmetric};
maximally symmetric matrices are both. See \cite{doherty2012convergence} for further discussion.)

If we restrict attention to the embedding this induces of $\R[x]_{d/2}$
(i.e. the homogeneous degree-$d/2$ polynomials) into $\R^{n^{d/2}}$,
the resulting subspace is sometimes called the \emph{symmetric subspace}
and in other works is denoted by $\vee^{d/2} \R^n$.
We sometimes abuse notation and let $\Pi$ be the projector from $\R^{n^{d/2}}$
to the canonical embedding of $\R[x]_{d/2}$.

\subsubsection{Maximally-Symmetric Matrices from Tensors}
The group $\cS_d$ acts on the set of $d$-tensors (canonically flattened to
matrices $\R^{n^{\lfloor d/2 \rfloor} \times n^{\lceil d/2 \rceil}}$)
by permutation of indices.
To any such flattened
$M \in \R^{n^{\lfloor d/2 \rfloor} \times n^{\lceil d/2 \rceil}}$,
we associate a family of maximally-symmetric matrices $\Sym M$ given by
\begin{displaymath}
  \Sym M \defeq \left \{ t \sum_{\pi \in \cS_d} \pi \cdot M
      \mbox{ for all $t \geq 0$} \right \}\mper
\end{displaymath}
That is, $\Sym M$ represents all scaled averages of $M$ over different possible
flattenings of its corresponding $d$-tensor.
The following conditions on a matrix $M$ are thus equivalent:
(1) $M \in \Sym M$,
(2) $M$ is maximally symmetric,
(3) a tensor that flattens to $M$ is invariant under the index-permutation
action of $\cS_d$, and
(4) $M$ may be considered as a linear functional on the space of homogeneous
polynomials $\R[x]_d$.
When we construct maximally-symmetric matrices from un-symmetric ones,
the choice of $t$ is somewhat subtle and will be important in not being too
wasteful in intermediate steps of our construction.

There is a more complex group action characterizing maximally-symmetric
matrices in $\R^{\tuples(d) \times \tuples(d)}$, which projects to the action
of $\cS_{d'}$ under the projection of $\R^{\tuples(d) \times \tuples(d)}$
to $\R^{n^{d'/2} \times n^{d'/2}}$.
We will never have to work explicitly with this full symmetry group;
instead we will be able to construct linear functionals on $\R[x]_{\leq d}$
(i.e. maximally symmetric matrices in $\R^{\tuples(d) \times \tuples(d)}$)
by symmetrizing each degree (i.e. each $d' \leq d$) more or less separately.

\subsection{Formal Statement of the Lower Bound}
\label{sec:formal}
We will warm up with the degree-$4$ lower bound, which is conceptually somewhat simpler.
\begin{theorem}[Degree-$4$ Lower Bound, General Version]
  \label{thm:lower-bound-4}
  Let $\bA$ be a $4$-tensor and let $\lambda > 0$ be a function of $n$.
  Suppose the following conditions hold:
  \begin{itemize}[\quad---]
  \item\emph{$\bA$ is significantly correlated with $\sum_{\pi \in \cS_4} \bA^\pi$.}\\
      $\iprod{ \bA, \sum_{\pi \in \cS_4} \bA^\pi } \ge \Omega(n^4)$.
  \item\emph{Permutations have lower-bounded spectrum.}\\
      For every $\pi \in \cS_4$, the Hermitian $n^2 \times n^2$ unfolding $\frac 1 2 (A^\pi + (A^\pi)^T)$ of
      $\bA^\pi$ has no eigenvalues smaller than $-\lambda^2$.
  \item\emph{Using $\bA$ as $4$th pseudo-moments does not imply that $\|x\|^4$ is too large.}\\
      For every $\pi \in \cS_4$, we have
      $\iprod{\symId, A^\pi} \le O(\lambda^2 n^{3/2})$
  \item\emph{Using $\bA$ for $4$th pseudo-moments does not imply first and second degree moments are
  too large.} \\
      Let $\cL : \R[x]_4 \rightarrow \R$ be the linear functional given by the
      matrix representation $M_{\cL} \seteq \frac 1 {\lambda^2 n^2} \sum_{\pi \in \cS_4} A^\pi$.
      Let
      \begin{align*}
        \delta_2 & \defeq \max_{i \neq j} \left | \cL \|x\|_2^2x_ix_j \right |\\
        \delta_2' & \defeq \max_{i} \left | \cL \|x\|_2^2x_i^2 \right |
      \end{align*}
      Then $n^{3/2}\delta_2' + n^2 \delta_2 \le O(1)$.
  \end{itemize}
  Then there is a degree-$4$ pseudo-distribution $\{ x \}$ satisfying
  $\{ \|x\|_2^2 = 1 \}$ so that
  $\pE \bA(x) \ge \Omega(n^2/\lambda^2) + \Theta(\E^\mu \bA(x))$.
\end{theorem}

The degree-$3$ version of our lower bound requires bounds on the spectra of the flattenings
not just of the $3$-tensor $\bA$ itself but also of the flattenings of an associated $4$-tensor,
which represents the polynomial $\iprod{x^{\tensor 2}, (\sum_i A_i \tensor A_i) x^{\tensor 2}}$.

\begin{theorem}[Degree-$3$ Lower Bound, General Version]
  \label{thm:lower-bound}
  Let $\bA$ be a $3$-tensor and let $\lambda > 0$ be a function of $n$.
  Suppose the following conditions hold:
  \begin{itemize}[\quad---]
  \item\emph{$\bA$ is significantly correlated with $\sum_{\pi \in \cS_3} \bA^\pi$.} \\
      $\iprod{ \bA, \sum_{\pi \in \cS_3} \bA^\pi } \ge \Omega(n^3)$.
  \item\emph{Permutations have lower-bounded spectrum.} \\
      For every $\pi \in \cS_3$, we have
      \[
        -2\lambda^2 \cdot \Pi \Id \Pi \preceq
        \frac 1 2 \Pi(\sigma \cdot A^\pi (A^\pi)^T + \sigma^2 \cdot A^\pi (A^\pi)^T)\Pi
        + \frac 1 2 \Pi(\sigma \cdot A^\pi (A^\pi)^T + \sigma^2 \cdot A^\pi (A^\pi)^T)^T\Pi\mper
      \]
  \item\emph{Using $AA^T$ for $4$th moments does not imply $\|x\|^4$ is too large.} \\
      For every $\pi \in \cS_3$, we have
      $\iprod{ \symId, A^\pi (A^\pi)^T } \le O(\lambda^2 n^{2})$
  \item\emph{Using $A$ and $AA^T$ for $3$rd and $4$th moments do not imply first and second degree moments are
  too large.} \\
      Let $\pi \in \cS_3$.
      Let $\cL : \R[x]_{4} \rightarrow \R$ be the linear functional given by
      the matrix representation
      $M_{\cL} \seteq \frac 1 {\lambda^2 n^2}\sum_{\pi'\in\cS_4}\pi'\cdot AA^T$.
      Let
      \begin{align*}
        \delta_1 & \defeq \max_i \left | \frac 1 {\lambda n^{3/2}}
            \iprod{\Id_{n \times n}, A^\pi_i} \right |\\
        \delta_2 & \defeq \max_{i \neq j} \left | \cL \|x\|_2^2x_ix_j \right |\\
        \delta_2' & \defeq \max_{i} \left |\cL \|x\|_2^2x_i^2 - \tfrac{1}{n}\cL\|x\|_2^4 \right |
      \end{align*}
      Then $n\delta_1 + n^{3/2}\delta_2' + n^2 \delta_2 \le O(1)$.
  \end{itemize}
  Then there is a degree-$4$ pseudo-distribution $\{ x \}$
  satisfying $\{ \| x \|_2^2 = 1 \}$ so that
  \begin{displaymath}
    \pE \bA(x) \ge
      \Omega \Paren{ \frac{n^{3/2}}{\lambda}}
      + \Theta(\E\nolimits^\mu \bA(x))
    \mper\qedhere
  \end{displaymath}
\end{theorem}

\subsubsection{Proof of \pref{thm:lower-bound-tensor-pca}}
We prove the degree-$3$ corollary;
the degree-$4$ case is almost identical using \pref{thm:lower-bound-4}
and \pref{lem:concentration-lb-4} in place of their degree-$3$ counterparts.
\begin{proof}
  Let $\bA$ be a $3$-tensor.
  If $\bA$ satisfies the conditions of \pref{thm:lower-bound} with $\lambda = O(n^{3/4}\log(n)^{1/4})$,
  we let $\eta(\bA)$ be the pseudo-distribution described there, with
  \begin{align*}
    \pE_{x \sim \eta(\bA)} \bA(x) \geq \Omega\Paren{ \frac {n^{3/2}}{\lambda}} + \Theta(\E\nolimits^{\mu} \bA(x))
  \end{align*}
  If $\bA$ does not satisfy the regularity conditions, we let $\eta(\bA)$ be the uniform distribution on the unit sphere.
  If $\bA$ has unit Gaussian entries, then \pref{lem:concentration-lb} says that the regularity conditions are satisfied with this choice of $\lambda$ with high probability.
  The operator norm of $\bA$ is at most $O(\sqrt n)$, so $\E^\mu \bA(x) = O(\sqrt n)$ (all with high probability) \Snote{} \cite{tomioka2014spectral}.
  We have chosen $\lambda$ and $\tau$ so that when the conditions of \pref{thm:lower-bound} and the bound on $\E^\mu \bA(x)$, obtain,
  \[
    \pE_{x \sim \eta(\bA)} \tau \cdot \iprod{v_0,x}^3 + \bA(x) \geq \Omega \Paren{\frac{n^{3/4}}{\log(n)^{1/4}}}\mper
  \]
  On the other hand, our arguments on degree-$4$ SoS certificates for random polynomials say with high probability every degree-$4$ pseudo-distribution $\{ y \}$ satisfying $\{ \|y\|^2 = 1 \}$ has $\pE \tau \cdot \iprod{v,y}^3 + \bA(y) \leq O(n^{3/4}\log(n)^{1/4})$.
  Thus, $\{ x\}$ is nearly optimal and we are done.
\end{proof}

\subsection{In-depth Preliminaries for Pseudo-Expectation Symmetries}
This section gives the preliminaries we will need to construct maximally-symmetric
matrices (a.k.a. functionals $\cL : \R[x]_{\leq 4} \rightarrow \R$) in what follows.
For a non-maximally-symmetric $M \in \R^{n^2 \times n^2}$ under the action of $\cS_4$ by permutation of indices, the subgroup $\cC_3 < \cS_4$
represents all the significant permutations whose spectra may differ from
one another in a nontrivial way.
The lemmas that follow will make this more precise.
For concreteness, we take $\cC_3 = \iprod{\sigma}$ with $\sigma = (234)$,
but any other choice of $3$-cycle would lead to a merely syntactic change in the proof.
\begin{lemma}
  \label{lem:s4-decomp}
  Let $\cD_8 < \cS_4$ be given by $\cD_8 = \iprod{(12),(34),(13)(24)}$.
  Let $\cC_3 = \{ (), \sigma, \sigma^2 \} = \langle \sigma \rangle$, where $()$
  denotes the identity in $\cS_4$.
  Then $\{ gh : g \in \cD_8, h \in \cC_3 \} = \cS_4$.
\end{lemma}
\begin{proof}
  The proof is routine; we provide it here for completeness.
  Note that $\cC_3$ is a subgroup of order $3$ in the alternating group $\cA_4$.
  This alternating group can be decomposed as $\cA_4 = \cK_4 \cdot \cC_3$,
  where $\cK_4 = \langle (12)(34), (13)(24) \rangle$ is a normal subgroup of
  $\cA_4$.
  We can also decompose $\cS_4 = \cC_2 \cdot \cA_4$ where
  $\cC_2 = \langle(12)\rangle$ and $\cA_4$ is a normal subgroup of $\cS_4$.
  Finally, $\cD_8 = \cC_2 \cdot \cK_4$ so by associativity,
  $\cS_4 = \cC_2 \cdot \cA_4
        = \cC_2 \cdot \cK_4 \cdot \cC_3
        = \cD_8 \cdot \cC_3$.
\end{proof}
This lemma has two useful corollaries:
\begin{corollary}
  \label{cor:s4-decomp-permuted}
  For any subset $S \subseteq \cS_4$, we have $\{ ghs : g \in \cD_8, h \in \cC_3, s \in S \} = \cS_4$.
\end{corollary}

\begin{corollary}
  \label{cor:deg-4-sym}
  Let $M \in \R^{n^{2} \times n^{2}}$. Let the matrix $M'$ be given by
  \begin{displaymath}
    M' \defeq \frac 1 2 \Pi \Paren{ M + \sigma \cdot M + \sigma^2 \cdot M} \Pi
      + \frac 1 2 \Pi \Paren{ M + \sigma \cdot M + \sigma^2 \cdot M}^T \Pi\mper
  \end{displaymath}
  Then $M' \in \Sym M$.
\end{corollary}
\begin{proof}
  Observe first that
  $M + \sigma \cdot M + \sigma^2 \cdot M = \sum_{\pi \in \cC_3} \pi \cdot M$.
  For arbitrary $N \in \R^{n^2 \times n^2}$, we show that
  $\frac 1 2 \Pi N \Pi + \frac 1 2 \Pi N^T \Pi
    = \frac 1 8 \sum_{\pi \in \cD_8} \pi \cdot N$.
  First, conjugation by $\Pi$ corresponds to averaging $M$ over the group
  $\langle (12),(34) \rangle$ generated by interchange of indices in row and
  column indexing pairs, individually.
  At the same time, $N + N^T$ is the average of $M$ over the matrix transposition
  permutation group $\langle (13)(24) \rangle$.
  All together,
  \begin{displaymath}
    M' = \frac 1 8 \sum_{g \in \cD_8} \sum_{h \in \cC_3} (gh) \cdot M
    = \frac 1 8\sum_{\pi \in \cS_4} \pi \cdot M
  \end{displaymath}
  and so $M' \in \Sym M$.
\end{proof}
We make an useful observation about the nontrivial permutations of $M$,
in the special case that $M = AA^T $ for some $3$-tensor $\bA$.
\begin{lemma}
  \label{lem:deg-4-sym-AAT}
  Let $\bA$ be a $3$-tensor and let $A \in \R^{n^2 \times n}$ be its flattening,
  where the first and third modes lie on the longer axis and the third mode lies
  on the shorter axis.
  Let $A_i$ be the $n \times n$ matrix slices of $\bA$ along the first mode,
  so that
  \begin{displaymath}
    A = \left( \begin{array}{c} A_1 \\ A_2 \\ \vdots \\ A_n \end{array} \right)
    \mper
  \end{displaymath}
  Let $P : \R^{n^2} \rightarrow \R^{n^2}$ be the orthogonal linear operator so
  that $[Px](i,j) = x(j,i)$.
  Then
  \begin{displaymath}
    \sigma \cdot AA^T = \Paren{\sum_i A_i \tensor A_i} P
    \qquad \mbox{ and } \qquad
    \sigma^2 \cdot AA^T = \sum_i A_i \tensor A_i^T\mper
  \end{displaymath}
\end{lemma}
\begin{proof}
  We observe that
  $AA^T[(j_1, j_2),(j_3,j_4)] = \sum_i A_{i j_1 j_2} A_{i j_3 j_4}$
  and that $(\sum_i A_i \tensor A_i)[(j_1, j_2),(j_3,j_4)]
  = \sum_i A_{i j_1 j_3} A_{i j_2 j_4}$.
  Multiplication by $P$ on the right has the effect of switching the order of
  the second indexing pair, so
  $[(\sum_i A_i \tensor A_i) P][(j_1, j_2),(j_3,j_4)]
      = \sum_i A_{ij_1j_4} A_{ij_2j_3}$.
  From this it is easy to see that
  $\sigma \cdot AA^T = (234) \cdot AA^T = (\sum_i A_i \tensor A_i) P$.

  Similarly, we have that
  \begin{displaymath}
    (\sigma^2 AA^T)[(j_1,j_2),(j_3,j_4)] = ((243)\cdot AA^T)[(j_1,j_2),(j_3,j_4)]
    = \sum_k A_{ij_1j_3}A_{ij_4j_2} \mcom
  \end{displaymath}
  from which we see that $\sigma^2 \cdot AA^T = \sum_i A_i \tensor A_i^T$.
\end{proof}

\paragraph{Permutations of the Identity Matrix}
The nontrivial permutations of $\Id_{n^2 \times n^2}$ are:
\begin{align*}
  \Id[(j,k),(j',k')] & = \delta(j,k)\delta(j',k')\\
  \sigma \cdot \Id[(j,k),(j',k')] & = \delta(j,j')\delta(k,k')\\
  \sigma^2 \cdot \Id[(j,k),(j',k')] & = \delta(j,k')\delta(j',k)\mper
\end{align*}
Since $(\Id + \sigma \cdot \Id + \sigma^2 \cdot \Id)$
is invariant under the action of $\cD_8$,
we have $(\Id + \sigma \cdot \Id + \sigma^2 \cdot \Id) \in \Sym M$;
up to scaling this matrix is the same as $\symId$ defined in \pref{eq:symId}.
We record the following observations:
\begin{itemize}[\quad---]
  \item $\Id$, $\sigma \cdot \Id$, and $\sigma^2 \cdot \Id$ are all symmetric
      matrices.
  \item Up to scaling, $\Id + \sigma^2 \Id$ projects to identity on the
      canonical embedding of $\R[x]_2$.
  \item The matrix $\sigma \cdot \Id$ is rank-$1$, positive-semidefinite, and
      has $\Pi (\sigma \cdot \Id) \Pi = \sigma \cdot \Id$.
  \item The scaling
    $[1/(n^2 + 2n)] (\Id + \sigma \cdot \Id + \sigma^2 \cdot \Id)$
    is equal to a linear functional $\E^\mu : \R[x]_4 \rightarrow \R$
    giving the expectation under the uniform distribution
    over the unit sphere $S^{n-1}$.
\end{itemize}

\subsection{Construction of Initial Pseudo-Distributions}
We begin by discussing how to create an initial guess at a pseudo-distribution
whose third moments are highly correlated with the polynomial $\bA(x)$.
This initial guess will be a valid pseudo-distribution, but will fail to be on
the unit sphere, and so will require some repairing later on.
For now, the method of creating this initial pseudo-distribution involves using
a combination of symmetrization techniques to ensure that the matrices we
construct are well defined as linear functionals over polynomials, and spectral
techniques to establish positive-semidefiniteness of these matrices.

\subsubsection{Extending Pseudo-Distributions to Degree Four}
In this section we discuss a construction that takes a linear functional
$\cL : \R[x]_{\leq 3} \rightarrow \R$ over degree-$3$ polynomials
and yields a degree-$4$ pseudo-distribution $\{ x \}$.
We begin by reminding the reader of the Schur complement criterion
for positive-semidefiniteness of block matrices.

\begin{theorem}
Let $M$ be the following block matrix.
\begin{displaymath}
  M \defeq \left ( \begin{array}{cc}
                      B & C^T\\
                      C & D\\
                   \end{array} \right )
\end{displaymath}
where $B \succeq 0$ and is full rank. Then $M \succeq 0$ if and only if
$D \succeq CB^{-1}C^T$.
\end{theorem}

Suppose we are given a linear functional $\cL : \R[x]_{\leq 3} \rightarrow \R$
with $\cL 1 = 1$.
Let $\cL|_1$ be $\cL$ restricted to $\R[x]_1$ and similarly for $\cL|_2$ and
$\cL|_3$.
We define the following matrices:
\begin{itemize}[\quad---]
  \item $M_{\cL|_1} \in \R^{n \times 1}$
      is the matrix representation of $\cL|_1$.
  \item $M_{\cL|_2} \in \R^{n \times n}$
      is the matrix representation of $\cL|_2$.
  \item $M_{\cL|_3} \in \R^{n^2 \times n}$
      is the matrix representation of $\cL|_3$.
  \item $V_{\cL|_2} \in \R^{n^2 \times 1}$
      is the vector flattening of $M_{\cL|_2}$.
\end{itemize}
Consider the block matrix $M \in \R^{\tuples(2) \times \tuples(2)}$ given by
\begin{displaymath}
  M \defeq \left ( \begin{array}{ccc}
                      1 & M_{\cL|_1}^T & V_{\cL|_2}^T\\
                      M_{\cL|_1} & M_{\cL|_2} & M_{\cL|_3}^T \\
                      V_{\cL|_2} & M_{\cL|_3} & D
                   \end{array} \right )\mcom
\end{displaymath}
with $D \in \R^{n^2 \times n^2}$ yet to be chosen.
By taking
\begin{displaymath}
  B = \left (  \begin{array}{cc}
                      1 & M_{\cL|_1}^T \\
                      M_{\cL|_1} & M_{\cL|_2}
                   \end{array} \right )
    \qquad
  C = \left ( \begin{array}{cc} V_{\cL|_2} & M_{\cL|_3} \end{array} \right )\mcom
\end{displaymath}
we see by the Schur complement criterion that $M$ is positive-semidefinite
so long as $D \succeq CB^{-1}C^T$.
However, not any choice of $D$ will yield $M$ maximally symmetric, which is
necessary for $M$ to define a pseudo-expectation operator $\pE$.

We would ideally take $D$ to be the spectrally-least maximally-symmetric matrix
so that $D \succeq C B^{-1} C^T$.
But this object might not be well defined, so we instead take the following
substitute.
\begin{definition} \label{def:sym-schur}
  Let $\cL, B,C$ as be as above.
  The \emph{symmetric Schur complement} $D \in \Sym C B^{-1} C^T$ is
  $t \sum_{\pi \in \cS_4} \pi \cdot (C B^{-1} C^T)$ for the least $t$ so that
  $t \sum_{\pi \in \cS_4} \pi \cdot (C B^{-1} C^T) \succeq C B^{-1} C^T$.
  We denote by $\pE^{\cL}$ the linear functional
  $\pE^{\cL} : \R[x]_{\leq 4} \rightarrow \R$
  whose matrix representation is $M$ with this choice of $D$,
  and note that $\pE^{\cL}$ is a valid degree-$4$ pseudo-expectation.
\end{definition}

\begin{example}[Recovery of Degree-$4$ Uniform Moments from Symmetric Schur Complement]
  Let $\cL : \R[x]_{\leq 3} \rightarrow \R$ be given by
  $\cL p(x) \seteq \E^{\mu} p(x)$.
  We show that $\pE^{\cL} = \E^\mu$.
  In this case it is straightforward to compute that
  $C B^{-1} C^T = \sigma \cdot \Id / n^2$.
  Our task is to pick $t \geq 0$ minimal so that
  $\frac t {n^2} \Pi(\Id + \sigma \cdot \Id + \sigma^2 \cdot \Id)\Pi
  \succeq \frac 1 {n^2} \Pi (\sigma \cdot \Id) \Pi$.

  We know that $\Pi (\sigma \cdot \Id) \Pi = \sigma \cdot \Id$.
  Furthermore, $\Pi \Id \Pi = \Pi (\sigma^2 \cdot \Id) \Pi$,
  and both are the identity on the canonically-embedded
  subspace $\R[x]_{2}$ in $\R^{\tuples(4)}$.
  We have previously observed that $\sigma \cdot \Id$
  is rank-one and positive-semidefinite,
  so let $w \in \R^{\tuples(4)}$ be such that $ww^T = \sigma \cdot \Id$.

  We compute $w^T (\Id + \sigma \cdot \Id + \sigma^2 \cdot \Id ) w
  = 2 \| w \|_2^2 + \| w \|_2^4 = 2n + n^2$ and
  $w^T (\sigma \cdot \Id) w = \|w\|_2^4 = n^2$.
  Thus $t = n^2 / (n^2 + 2n)$ is the minimizer.
  By a previous observation, this yields $\E^\mu$.
\end{example}
To prove our lower bound, we will generalize the above example
to the case that we start with an operator $\cL : \R[x]_{\leq 3} \rightarrow \R$
which does not match $\E^\mu$ on degree-3 polynomials.

\subsubsection{Symmetries at Degree Three}
  We intend on using the symmetric Schur complement to construct a
  pseudo-distribution from some $\cL : \R[x]_{\leq 3} \rightarrow \R$
  for which $\cL \bA(x)$ is large.
  A good such $\cL$ will have $\cL x_i x_j x_k$ correlated with
  $\sum_{\pi \in \cS_3} \bA^\pi_{ijk}$ for all (or many) indices $i,j,k$.
  That is, it should be correlated with the
  coefficient of the monomial $x_i x_j x_k$ in $\bA(x)$.
  However, if we do this directly by setting $\cL x_ix_jx_k = \sum_{\pi} \bA^{\pi}_{ijk}$,
  it becomes technically inconvenient to
  control the spectrum of the resulting symmetric Schur complement.
  To this avoid, we discuss how to utilize a decomposition of
  $M_{\cL|_3}$ into nicer matrices if such a decomposition exists.

  \begin{lemma}
  \label{lem:deg-3-sym}
  Let $\cL : \R[x]_{\leq 3} \rightarrow \R$, and suppose that
  $M_{\cL|_3} = \frac 1 k (M_{\cL|_3}^1 + \cdots + M_{\cL|_3}^k)$ for some
  $M_{\cL|_3}^1,\ldots,M_{\cL|_3}^k \in \R^{n^2 \times n}$.
  Let $D_1,\ldots,D_k$ be the respective symmetric Schur complements of the
  family of matrices
  \begin{displaymath}
    \left \{ \left ( \begin{array}{ccc}
                      1 & M_{\cL|_1}^T & V_{\cL|_2}^T\\
             M_{\cL|_1} & M_{\cL|_2} & (M_{\cL|_3}^i)^T \\
                      V_{\cL|_2} & M_{\cL|_3}^i & \bullet
    \end{array} \right ) \right \}_{1 \leq i \leq k} \mper
   \end{displaymath}
   Then the matrix
   \begin{displaymath}
     M \defeq \frac 1 k \mathlarger{\mathlarger{\sum}}_{i = 1}^k
    \left ( \begin{array}{ccc}
                      1 & M_{\cL|_1}^T & V_{\cL|_2}^T\\
             M_{\cL|_1} & M_{\cL|_2} & (M_{\cL|_3}^i)^T \\
                      V_{\cL|_2} & M_{\cL|_3}^i & D_i
    \end{array} \right )
  \end{displaymath}
  is positive-semidefinite and maximally symmetric.
  Therefore it defines a valid pseudo-expectation $\pE^{\cL}$.
  (This is a slight abuse of notation, since the pseudo-expectation
  defined here in general differs from the one in \pref{def:sym-schur}.)

\end{lemma}
\begin{proof}
  Each matrix in the sum defining $M$ is positive-semidefinite, so $M \succeq 0$.
  Each $D_i$ is maximally symmetric and therefore so is $\sum_{i = 1}^k D_i$.
  We know that $M_{\cL|_3} = \sum_{i = 1}^k M_{\cL|_3}^i$ is maximally-symmetric,
  so it follows that $M$ is the matrix representation of a valid
  pseudo-expectation.
\end{proof}

\subsection{Getting to the Unit Sphere}
Our next tool takes a pseudo-distribution $\pE$ that is slightly off the
unit sphere, and corrects it to give a linear functional
$\cL : \R[x]_{\leq 4} \rightarrow \R$ that lies on the unit sphere.

We will also characterize how badly the resulting linear functional
deviates from the nonnegativity condition ($\cL p(x)^2 \ge 0$ for $p \in \R[x]_{\leq 2}$) required
to be a pseudo-distribution
\begin{definition}
  Let $\cL : \R[x]_{\leq d} \rightarrow \R$. We define
  \begin{displaymath}
    \lmin \cL \defeq
      \min_{p \in \R[x]_{\leq d/2}} \frac{ \cL p(x)^2}{\E^\mu p(x)^2}
  \end{displaymath}
  where $\E^\mu p(x)^2$ is the expectation of $p(x)^2$ when $x$ is distributed
  according to the uniform distribution on the unit sphere.
\end{definition}
Since $\E^\mu p(x)^2 \geq 0$ for all $p$,
we have $\cL p(x)^2 \geq 0$ for all $p$ if and only if $\lmin \cL \geq 0$.
Thus $\cL$ on the unit sphere is a pseudo-distribution if and only if
$\cL 1 = 1$ and $\lmin \cL \ge 0$.

\begin{lemma}
  \label{lem:deg-one-two}
  Let $\pE : \R[x]_{\leq 4} \rightarrow \R$ be a valid pseudodistribution.
  Suppose that:
  \begin{enumerate}
    \item $c \seteq \pE \|x\|_2^4 \ge 1$.
    \item $\pE$ is close to lying on the sphere, in the sense that
      there are $\delta_1, \delta_2,\delta_2' \geq 0$
      so that:
      \begin{enumerate}
        \item $|\tfrac{1}{c} \pE \|x\|_2^2x_i - \cL' x_i| \leq \delta_1$ for all $i$.
        \item $|\tfrac{1}{c} \pE \|x\|_2^2 x_i x_j - \cL' x_i x_j | \leq \delta_2$
            for all $i \neq j$.
        \item $|\tfrac{1}{c} \pE \| x\|_2^2 x_i^2 -\cL' x_i^2 | \leq \delta_2'$ for all $i$.
      \end{enumerate}
  \end{enumerate}
  Let $\cL : \R[x]_{\leq 4} \rightarrow \R$ be as follows on homogeneous $p$:
  \begin{displaymath}
    \cL p(x) \defeq \begin{cases}
      \pE 1 & \mbox{ if $\deg p = 0$}\\
      \tfrac{1}{c} \pE p(x) & \mbox{ if $\deg p = 3,4$}\\
      \tfrac{1}{c} \pE p(x)\|x\|_2^2 & \mbox{ if $\deg p = 1,2$} \mper
    \end{cases}
  \end{displaymath}
  Then $\cL$ satisfies $\cL p(x) (\|x \|_2^2 - 1) = 0$
  for all $p(x) \in \R[x]_{\leq 2}$
  and has $\lmin \cL \geq -\tfrac{c-1}{c} -O(n)\delta_1 - O(n^{3/2})\delta_2' - O(n^2)\delta_2$.
\end{lemma}
\begin{proof}
  It is easy to check that $\cL p(x) (\|x \|_2^2 - 1) = 0$
  for all $p \in \R[x]_{\leq 2}$ by expanding the definition of $\cL$.

  Let the linear functional $\cL': \R[x]_{\leq 4} \rightarrow \R$ be defined
  over homogeneous polynomials $p$ as
  \begin{displaymath}
    \cL' p(x) \defeq \begin{cases}
      c & \mbox{ if $\deg p = 0$}\\
      \pE p(x) & \mbox{ if $\deg p = 3,4$}\\
      \pE p(x)\|x\|_2^2 & \mbox{ if $\deg p = 1,2$} \mper
    \end{cases}
  \end{displaymath}
  Note that $\cL'p(x) = c \cL p(x)$ for all $p \in \R[x]_{\leq 4}$.
  Thus $\lmin \cL \ge \lmin \cL'/c$, and the kernel of $\cL'$ is identical to
  the kernel of $\cL$.

  In particular, since $(\|x\|_2^2 -1)$ is in the kernel of $\cL'$,
  either $\lmin \cL' = 0$ or
  \[
    \lmin \cL' = \min_{p \in \R[x]_{\leq 2}, p \perp (\|x\|_2^2 - 1)}
        \frac{\cL' p(x)^2}{\E\nolimits^\mu p(x)^2}.
  \]
  Here $p \perp (\|x\|^2 - 1)$ means that the polynomials $p$ and $\|x\|^2 - 1$ are perpendicular in the coefficient basis. That is, if $p(x) = p_0 + \sum_i p_i x_i + \sum_{ij} p_{ij} x_i x_j$, this means $\sum_{ii} p_{ii} = p_0$.
  The equality holds because any linear functional on polynomials $\cK$ with $(\|x\|^2 - 1)$ in its kernel satisfies $\cK(p(x) + \alpha (\|x\|^2 - 1))^2 = \cK p(x)^2$ for every $\alpha$. The functionals $\cL'$ and $\E^\mu$ in particular both satisfy this.

  Let $\Delta \seteq \cL' - \pE$, and note that $\Delta$ is nonzero only when
  evaluated on the degree-$1$ or -$2$ parts of polynomials.
  It will be sufficient to bound $\Delta$, since assuming $\lmin \cL' \ne 0$,
  \begin{align*}
    \lmin \cL' \nolimits
      &= \min_{p \in \R[x]_{\leq 2}, p \perp (\|x\|_2^2 - 1)}
          \frac{ \Delta p(x)^2 + \pE p(x)^2}{\E\nolimits^\mu p(x)^2}
      \\&\geq \min_{p \in \R[x]_{\leq 2}, p \perp (\|x\|_2^2 - 1)}
          \frac{ \Delta p(x)^2}{\E\nolimits^\mu p(x)^2} \mper
  \end{align*}

  Let $p \in \R[x]_{\leq 2}$.
  We expand $p$ in the monomial basis: $p(x) = p_0 + \sum_i p_i x_i + \sum_{i,j} p_{ij} x_i x_j$.
  Then
  \[
    p(x)^2 = p_0^2 + 2 p_0 \sum_i p_i x_i + 2 p_0 \sum_{ij} p_{ij} x_i x_j + \Paren{\sum_i p_i x_i}^2 + 2 \Paren{\sum_i p_i x_i}\Paren{\sum_{ij} p_{ij} x_i x_j} + \Paren{\sum_{ij} p_{ij} x_i x_j}^2\mper
  \]
  An easy calculation gives
  \begin{displaymath}
    \E\nolimits^\mu p(x)^2 = p_0^2 + \frac {2p_0} n \sum_i p_{ii} + \frac 1 n \sum_i p_i^2 + \frac 1 {n^2 + 2n} \Paren{\Paren{\sum_i p_{ii}}^2 + \sum_{ij} p_{ij}^2 + \sum_i p_{ii}^2}\mper
  \end{displaymath}
  The condition $p \perp (\|x\|_2^2 - 1)$ yields $p_0 = \sum_i p_{ii}$.
  Substituting into the above, we obtain the sum of squares
  \[
    \E\nolimits^\mu p(x)^2 = p_0^2 + \frac {2p_0^2} n + \frac 1 n \sum_i p_i^2 + \frac 1 {n^2 + 2n} \Paren{p_0^2 + \sum_{ij} p_{ij}^2 + \sum_i p_{ii}^2}\mper
  \]

  Without loss of generality we assume $\E\nolimits^\mu p(x)^2 = 1$, so now it is enough just to bound $\Delta p(x)^2$.
  We have assumed that $|\Delta x_i| \leq c\delta_1$ and $|\Delta x_i x_j| \leq c\delta_2$ for $i \neq j$ and $| \Delta x_i^2 | \leq c\delta_2'$.
  We also know $\Delta 1 = c-1$ and $\Delta p(x) = 0$ when $p$ is a homogeneous degree-$3$ or -$4$ polynomial.
  So we expand
  \begin{displaymath}
    \Delta p(x)^2 = p_0^2(c-1) + 2 p_0 \sum_i p_i \Delta x_i + 2 p_0 \sum_{ij} p_{ij} \Delta x_i x_j + \sum_{i,j} p_i p_j \Delta x_i x_j
  \end{displaymath}
  and note that this is maximized in absolute value when all the signs line up:
  \begin{displaymath}
    |\Delta p(x)^2| \leq p_0^2(c-1) + 2 c\delta_1 |p_0| \sum_i |p_i| + 2|p_0| \Paren{ c\delta_2 \sum_{i \neq j} |p_{ij}| + c\delta_2' \sum_i |p_{ii}|} + c\delta_2 \Paren{\sum_i |p_i|}^2 + c\delta_2' \sum_i p_i^2\mper
  \end{displaymath}

  We start with the second term.
  If $p_0^2 = \alpha$ for $\alpha \in [0,1]$, then
  $\sum_i p_i^2 \leq n(1 - \alpha)$ by our assumption that $\E^\mu p(x)^2 = 1$.
  This means that
  \begin{displaymath}
    2 c\delta_1 |p_0| \sum_i |p_i|
      \leq 2 c\delta_1 \sqrt{\alpha n\sum_i p_i^2}
      \leq 2 c\delta_1 n \sqrt{\alpha(1 - \alpha)}
      \leq O(n)c\delta_1 \mcom
  \end{displaymath}
  where we have used Cauchy-Schwarz and the fact
  $\max_{0 \leq \alpha \leq 1} \alpha (1 - \alpha) = (1/2)^2$.
  The other terms are all similar:
  \begin{align*}
    p_0^2(c-1) \le c-1\\
    2|p_0| c\delta_2 \sum_{i \neq j} |p_{ij}| \leq 2 c\delta_2 \sqrt{\alpha n^2 \sum_{ij} p_{ij}^2} \leq 2 c\delta_2 O(n^2) \sqrt{\alpha(1-\alpha)} \leq O(n^2) c\delta_2\\
    2|p_0| c\delta_2' \sum_i |p_{ii}| \leq 2 c\delta_2' \sqrt{ \alpha n \sum_i p_{ii}^2} \leq O(n^{3/2}) c\delta_2'\\
    c\delta_2 \Paren{\sum_i |p_i| }^2 \leq c\delta_2 n \sum_i p_i^2 \leq O(n^2) c\delta_2\\
    c\delta_2' \sum_i p_i^2 \leq O(n) c\delta_2'\mcom
  \end{align*}
  where in each case we have used Cauchy-Schwarz and our assumption $\E^\mu p(x)^2 = 1$.

  Putting it all together, we get
  \begin{displaymath}
    \lmin \Delta \geq - (c-1) - O(n)c\delta_1 - O(n^{3/2})c\delta_2' - O(n^2)c\delta_2\mper\qedhere
  \end{displaymath}
\end{proof}

\subsection{Repairing Almost-Pseudo-Distributions}
Our last tool takes a linear functional $\cL: \R[x]_{\le d}$
that is ``almost'' a pseudo-distribution over the unit sphere,
in the precise sense that all conditions for being a pseudo-distribution
over the sphere are satisfied except that $\lmin \cL = -\epsilon$.
The tool transforms it into a bona fide pseudo-distribution
at a slight cost to its evaluations at various polynomials.

\begin{lemma}
  \label{lem:mix-with-uniform}
  Let $\cL : \R[x]_{\leq d} \rightarrow \R$ and suppose that
  \begin{itemize}[\quad---]
    \item $\cL 1 = 1$
    \item $\cL p(x) (\| x \|^2 - 1) = 0$ for all $p \in \R[x]_{\leq d - 2}$.
    \item $\lmin \cL = -\epsilon$.
  \end{itemize}
  Then the operator $\pE : \R[x]_{\leq d} \rightarrow \R$ given by
  \begin{displaymath}
    \pE p(x) \defeq \frac 1 {1 + \epsilon}
        ( \cL p(x) + \epsilon \E\nolimits^\mu p(x) )
  \end{displaymath}
  is a valid pseudo-expectation satisfying $\{ \|x\|^2 = 1 \}$.
\end{lemma}
\begin{proof}
  It will suffice to check that $\lmin \pE \ge 0$ and
  that $\pE$ has $\pE (\| x \|_2^2 - 1)^2 = 0$ and $\pE 1 = 1$.
  For the first, let $p \in \R[x]_{\geq 2}$. We have
  \begin{align*}
    \frac{\pE p(x)^2}{\E^\mu p(x)^2}
    = \Paren{\frac 1 {1 + \epsilon}}
        \Paren{\frac{\E^0 p(x)^2 + \epsilon \E^\mu p(x)^2}{\E^\mu p(x)^2}}
    \geq \Paren{\frac 1 {1 + \epsilon}} ( -\epsilon + \epsilon)
    \geq 0 \mper
  \end{align*}
  Hence, $\lmin \pE \ge 0$.

  It is straightforward to check the conditions that $\pE 1 = 1$ and that
  $\pE$ satisfies $\{\|x\|^2 - 1 = 0\}$, since $\pE$ is a convex combination
  of linear functionals that already satisfy these linear constraints.
\end{proof}

\subsection{Putting Everything Together}
We are ready to prove \pref{thm:lower-bound-4} and \pref{thm:lower-bound}. The
proof of \pref{thm:lower-bound-4} is somewhat simpler and contains many of the
ideas of the proof of \pref{thm:lower-bound}, so we start there.

\subsubsection{The Degree-4 Lower Bound}
\begin{proof}[Proof of \pref{thm:lower-bound-4}]
  We begin by constructing a degree-$4$ pseudo-expectation
  $\pE^0 : \R[x]_{\leq 4} \rightarrow \R$
  whose degree-$4$ moments are biased towards $\bA(x)$
  but which does not yet satisfy $\{\|x\|_2^2 - 1 = 0\}$.

  Let $\cL : \R[x]_{\leq 4} \rightarrow \R$ be the functional whose matrix
  representation when restricted to $\cL|_4 : \R[x]_4 \rightarrow \R$
  is given by $M_{\cL|_4} = \frac 1 {|\cS_4|n^2} \sum_{\pi \in \cS_4} A^\pi$,
  and which is $0$ on polynomials of degree at most $3$.

  Let $\pE^0 \seteq \E^\mu + \epsilon \cL$, where $\epsilon$ is a parameter to
  be chosen soon so that $\pE^0 p(x)^2 \geq 0$ for all $p \in \R[x]_{\leq 2}$.
  Let $p \in \R[x]_{\leq 2}$.
  We expand $p$ in the monomial basis as
  $p(x) = p_0 + \sum_i p_i x_i + \sum_{ij} p_{ij} x_i x_j$.
  Then
  \begin{displaymath}
    \E\nolimits^\mu p(x)^2 \geq \frac 1 {n^2} \sum_{ij} p_{ij}^2\mper
  \end{displaymath}
  By our assumption on negative eigenvalues of $A^\pi$ for all $\pi \in \cS_4$,
  we know that $\cL p(x)^2 \geq \frac{-\lambda^2}{n^2} \sum_{ij} p_{ij}^2$.
  So if we choose $\epsilon \leq 1/\lambda^2$, the operator
  $\pE^0 = \pE^{\mu} + \cL/\lambda^2$ will be a valid pseudo-expectation.
  Moreover $\pE^0$ is well correlated with $A$, since it was obtained by
  maximizing the amount of $\cL$, which is simply the (maximally-symmetric) dual of $A$.
  However the calculation of $\pE^0 \|x\|_2^4$ shows that this pseudo-expectation
  is not on the unit sphere, though it is close. Let $c$ refer to
  \[ c \seteq \pE\nolimits^0 \|x\|_2^4
      = \E\nolimits^\mu \|x\|_2^4 + \frac{1}{\lambda^2} \cL \|x\|_2^4
      = 1 + \frac{1}{|\cS_4|n^2\lambda^2}
          \sum_{\pi \in \cS_4}\iprod{\symId, A^\pi}
      = 1 + O(n^{-1/2})\mper
  \]

  We would like to use \pref{lem:deg-one-two} together with $\pE^0$ to obtain
  some $\cL^1 : \R[x]_{\leq 4} \rightarrow \R$ with $\|x\|_2^2 - 1$ in its
  kernel and bounded $\lmin \cL^1$ while still maintaining a high correlation
  with $A$.
  For this we need $\xi_1, \xi_2, \xi_2'$ so that
  \begin{itemize}[\quad---]
    \item $\Big|\tfrac{1}{c}\pE^0 \|x\|_2^2x_i - \pE^0 x_i\Big| \leq \xi_1$ for all $i$.
    \item $\Big|\tfrac{1}{c}\pE^0 \|x\|_2^2 x_i x_j - \pE^0 x_i x_j\Big| \leq \xi_2$
        for all $i \neq j$.
    \item $\Big|\tfrac{1}{c}\pE^0 \| x\|_2^2 x_i^2 -\pE^0 x_i^2\Big| \leq \xi_2'$
        for all $i$.
  \end{itemize}

  Since $\pE^0 p(x) = 0$ for all homogeneous odd-degree $p$,
  we may take $\xi_1 = 0$.
  For $\xi_2$, we have that when $i \neq j$,
  \begin{displaymath}
    \Big| \tfrac{1}{c} \pE\nolimits^0 \|x\|_2^2 x_i x_j - \pE\nolimits^0 x_i x_j \Big|
      = \left|\frac{1}{c\lambda^2} \cL \|x\|_2^2 x_i x_j \right|
      \leq \delta_2 \mcom
  \end{displaymath}
  where we recall $\delta_2$ and $\delta_2'$ defined in the theorem statement.
  Finally, for $\xi_2'$, we have
  \begin{align*}
    \Big| \tfrac{1}{c} \pE\nolimits^0 \|x\|_2^2 x_i^2 - \pE\nolimits^0 x_i^2 \Big|
      \leq \left| \frac{1}{c\lambda^2} \cL \|x\|_2^2 x_i^2 \right|
        + \left | \tfrac{1}{c}\E\nolimits^\mu \|x\|_2^2x_i^2 - \E\nolimits^\mu x_i^2\right|
      \leq \delta_2' + \tfrac{c-1}{cn} \mper
  \end{align*}
  Thus, \pref{lem:deg-one-two} yields $\cL^1 : \R[x]_{\leq 4} \rightarrow \R$
  with $\|x\|_2^2 -1$ in its kernel in the sense that
  $\cL^1 p(x) (\|x\|_2^2 - 1) = 0$ for all $p \in \R[x]_{\leq 2}$.
  If we take $\xi_2 = \delta_2$ and $\xi_2' = \delta_2' + \tfrac{c-1}{cn}$,
  then $\lmin \cL^1 \geq -\tfrac{c-1}{c} - n^2\delta_2 - n^{3/2}(\delta_2' + \tfrac{c-1}{cn}) = -O(1)$.
  Furthermore, $\cL^1 \bA(x) = \frac{1}{c\lambda^2} \cL \bA(x) = \Theta(\frac{1}{\lambda^2}\cL\bA(x))$.

  So by \pref{lem:mix-with-uniform},
  there is a degree-$4$ pseudo-expectation $\pE$ satisfying
  $\{ \|x\|_2^2 = 1 \}$ so that
  \begin{align*}
    \pE \bA(x)
      &= \Theta\left(\frac{1}{\lambda^2}\cL \bA(x)\right)
        + \Theta(\E\nolimits^\mu \bA(x)) \\
      &= \Theta\left(\frac{1}{|\cS_4|n^2\lambda^2}
          \iprod{A, \sum_{\pi \in \cS_4} A^\pi} \right)
        + \Theta(\E\nolimits^\mu \bA(x)) \\
      &\ge \Omega \Paren{ \frac {n^2}{\lambda^2} }
        + \Theta(\E\nolimits^\mu \bA(x))
        \mper \qedhere
  \end{align*}
\end{proof}

\subsubsection{The Degree-3 Lower Bound}

Now we turn to the proof of \pref{thm:lower-bound}.

\begin{proof}[Proof of \pref{thm:lower-bound}]
Let $\bA$ be a $3$-tensor.
Let $\epsilon \geq 0$ be a parameter to be chosen later.
We begin with the following linear functional
$\cL : \R[x]_{\leq 3} \rightarrow \R$.
For any monomial $x^\alpha$ (where $\alpha$ is a multi-index of degree at most
$3$),
\begin{displaymath}
  \cL x^\alpha \defeq \begin{cases}
    \E\nolimits^\mu x^\alpha
        & \mbox{ if $\deg x^\alpha \leq 2$}\\
    \frac {\epsilon} {n^{3/2}} \sum_{\pi \in \cS_3} \bA^\pi_\alpha
        & \mbox{ if $\deg x^\alpha = 3$}
  \end{cases}\mper
\end{displaymath}
The functional $\cL$ contains our current best guess at the degree 1 and 2
moments of a pseudo-distribution whose degree-3 moments
are $\epsilon$-correlated with $\bA(x)$.

The next step is to use symmetric Schur complement to extend $\cL$ to
a degree-$4$ pseudo-expectation.
Note that $M_{\cL|_3}$ decomposes as
\begin{displaymath}
  M_{\cL|_3} = \sum_{\pi \in \cS_3} \Pi A^\pi
\end{displaymath}
where, as a reminder, $A^\pi$ is the $n^2 \times n$ flattening of $\bA^\pi$
and $\Pi$ is the projector to the canonical embedding of $\R[x]_{2}$ into
$\R^{n^2}$.
So, using \pref{lem:deg-3-sym}, we want to find the symmetric Schur complements
of the following family of matrices (with notation matching the statement of
\pref{lem:deg-3-sym}):
\begin{displaymath}
  \left \{ \left ( \begin{array}{ccc}
                    1 & M_{\cL|_1}^T & V_{\cL|_2}^T\\
           M_{\cL|_1} & M_{\cL|_2} & \frac \epsilon {n^{3/2}} (\Pi A^\pi)^T \\
           V_{\cL|_2} & \frac \epsilon {n^{3/2}} \Pi A^\pi & \bullet
  \end{array} \right ) \right \}_{\pi \in \cS_3} \mper
\end{displaymath}
Since we have the same assumptions on $A^\pi$ for all $\pi \in \cS_3$,
without loss of generality we analyze just the case that $\pi$ is the
identity permutation, in which case $A^\pi = A$.

Since $\cL$ matches the degree-one and degree-two moments of the uniform
distribution on the unit sphere, we have $M_{\cL|_1} = \mathbf{0}$,
the $n$-dimensional zero vector, and $M_{\cL|_2} = \frac 1 n \Id_{n \times n}$.
Let $w \in \R^{n^2}$ be the $n^2$-dimensional vector flattening of
$\Id_{n \times n}$.
We observe that $ww^T = \sigma \cdot \Id$ is one of the permutations of
$\Id_{n^2 \times n^2}$.
Taking $B$ and $C$ as follows,
\begin{displaymath}
  B = \left(\begin{array}{cc}
          1       & \mathbf{0} \\
          \mathbf{0}  & \frac 1 n \Id_{n \times n}
      \end{array}\right )
    \qquad
  C = \left ( \begin{array}{cc}
          w & \frac{\epsilon}{n^{3/2}} A
      \end{array} \right )\mcom
\end{displaymath}
we compute that
\begin{displaymath}
  CB^{-1}C^T
    = \frac 1 {n^2} (\sigma \cdot \Id)
        + \frac{\epsilon^2}{n^2} \Pi AA^T \Pi\mper
\end{displaymath}
Symmetrizing the $\Id$ portion and the $AA^T$ portion of this matrix separately,
we see that the symmetric Schur complement that we are looking for is the
spectrally-least $M \in \Sym \Paren{\frac 1 {n^2} (\sigma \cdot \Id)
+ \frac{\epsilon^2}{n^2} AA^T }$ so that
\begin{align*}
  M & = \frac t {n^2}\Brac{3\symId + \frac{\epsilon^2}{2}
    \Paren{\Pi (AA^T + \sigma \cdot AA^T + \sigma^2 \cdot AA^T) \Pi
  + \Pi (AA^T + \sigma \cdot AA^T + \sigma^2 \cdot AA^T)^T \Pi}}\\
  & \succeq \frac 1 {n^2} (\sigma \cdot \Id) + \frac {\epsilon^2}{n^2} \Pi AA^T \Pi \mper
\end{align*}
Here we have used \pref{cor:deg-4-sym} and \pref{cor:s4-decomp-permuted} to
express a general element of $\Sym (\frac{\epsilon^2}{n^2}\Pi AA^T \Pi)$
in terms of $\Pi, AA^T, \sigma \cdot AA^T$, and $\sigma^2 \cdot AA^T$.

Any spectrally small $M$ satisfying the above suffices for us.
Taking $t = 1$, canceling some terms, and making the substitution
$3\symId - \sigma \cdot \Id = 2 \Pi \Id \Pi$,
we see that it is enough to have
\begin{displaymath}
        -2\, \Pi \Id \Pi \preceq
        \frac{\epsilon^2}{2} \Pi(\sigma \cdot A A^T + \sigma^2 \cdot A A^T)\Pi
        + \frac{\epsilon^2}{2} \Pi(\sigma \cdot A A^T + \sigma^2 \cdot AA^T)^T\Pi\mcom
\end{displaymath}
which by the premises of the theorem holds for $\epsilon = 1/\lambda$.
Pushing through our symmetrized Schur complement rule with our decomposition
of $M_{\cL|_3}$ (\pref{lem:deg-3-sym}), this $\epsilon$ yields a valid
degree-$4$ pseudo-expectation $\pE^0 : \R[x]_{\leq 4} \rightarrow \R$.
From our choice of parameters, we see that $\pE^0|_4$, the degree-$4$ part of
$\pE^0$, is given by $\pE^0|_4 = \frac {n^2 + 2n}{n^2} \E^\mu + \cL$,
where $\cL : \R[x]_4 \rightarrow \R$ is as defined in the theorem statement.
Furthermore, $\pE^0 p(x) = \E^\mu p(x)$ for $p$ with $\deg p \leq 2$.

We would like to know how big $\pE^0 \| x \|_2^4$ is. We have
\begin{displaymath}
  c \seteq
  \pE^0 \| x \|_2^4
    = \Paren{1 + \frac 1 n } \E\nolimits^\mu \|x\|_2^4 + \cL \|x\|_2^4
    = 1 + \frac 1 n + \cL \|x\|_2^4\mper
\end{displaymath}
We have assumed that $\iprod{\symId, AA^T} \le  O(\lambda^2 n^2)$.
Since $\symId$ is maximally symmetric, we have
$\iprod{\symId, \sum_{\pi \in \cS_4} \pi \cdot AA^T}
= \iprod{\symId, |\cS_4|AA^T}$ and so
\begin{displaymath}
  \cL \|x\|_2^4
    = \frac 1 {\lambda^2 n^2} \iprod{\symId, M_{\cL|_4}}
    = \frac 1 {n^2 \lambda^2}
        \Theta(\iprod{\symId, \sum_{\pi \in \cS_4} \pi \cdot AA^T})
    \le O(1)\mper
\end{displaymath}

Finally, our assumptions on $\iprod{A,\sum_{\pi \in \cS_3} A^\pi}$ yield
\begin{displaymath}
  \pE^0 \bA(x)
    = \frac{\epsilon}{n^{3/2}} \iprod{A,\sum_{\pi \in \cS_3}A^\pi}
  \ge \Omega \Paren{\frac{n^{3/2}}{\lambda}}\mper
\end{displaymath}

We have established the following lemma.
\begin{lemma}
  Under the assumptions of \pref{thm:lower-bound} there is a degree-$4$
  pseudo-expectation operator $\pE^0$ so that
  \begin{itemize}[\quad---]
    \item $c \seteq \pE^0 \|x \|_2^4 = 1 + O(1)$.
    \item $\pE^0 \bA(x) \ge \Omega(n^{3/2}/\lambda)$.
    \item $\pE^0 p(x) = \E^\mu p(x)$ for all $p \in \R[x]_{\leq 2}$.
    \item $\pE^0|_4 = (1 + \frac{1}{n}) \E^\mu|_4 + \cL$.
  \qed
  \end{itemize}
\end{lemma}
Now we would like feed $\pE^0$ into \pref{lem:deg-one-two} to get a linear
functional $\cL^1 : \R[x]_{\leq 4} \rightarrow \R$ with $\|x\|_2^2 - 1$ in its
kernel (equivalently, which satisfies $\{\|x\|_2^4 - 1 = 0\}$), but in order
to do that we need to find $\xi_1, \xi_2, \xi_2'$ so that
\begin{itemize}[\quad---]
  \item \label{item:deg-one}
      $\Big|\tfrac{1}{c}\pE^0 \|x\|_2^2x_i - \pE^0 x_i\Big| \leq \xi_1$ for all $i$.
  \item \label{item:deg-two}
      $\Big|\tfrac{1}{c}\pE^0 \|x\|_2^2 x_i x_j - \pE^0 x_i x_j\Big| \leq \xi_2$
      for all $i \neq j$.
  \item \label{item:deg-two-diag}
      $\Big|\tfrac{1}{c}\pE^0 \| x\|_2^2 x_i^2 -\pE^0 x_i^2 \Big| \leq \xi_2'$ for all $i$.
\end{itemize}
For $\xi_1$, we note that for every $i$, $\pE^0 x_i = 0$ since $\pE^0$
matches the uniform distribution on degree one and two polynomials.
Thus, $\left|\tfrac{1}{c}\pE^0 \|x\|_2^2 x_i - \pE^0 x_i \right|
= \left|\tfrac{1}{c}\pE^0 \|x\|_2^2 x_i \right|$.

We know that $M_{\pE^0|_3}$, the matrix representation of the degree-$3$ part
of $\pE^0$, is $\frac{1}{|\cS_3|n^{3/2}\lambda} A$.
Expanding $\pE^0 \|x\|_2^2 x_i$ with matrix representations, we get
\begin{displaymath}
  \Big|\tfrac{1}{c}\pE\nolimits^0 \|x\|_2^2 x_i\Big|
    = \frac 1 {|\cS_3| c n^{3/2}\lambda}
      \Big|\iprod{\Id_{n \times n}, \sum_{\pi \in \cS_3} A_i}\Big|
    \leq \delta_1\mcom
\end{displaymath}
where $\delta_1$ is as defined in the theorem statement.

Now for $\xi_2$ and $\xi_2'$.
Let $\cL$ be the operator in the theorem statement.
By the definition of $\pE^0$, we get
\begin{displaymath}
  \pE^0|_4 \le
      \Brac{ \Paren{ 1 + \frac 1 n } \E\nolimits^\mu|_4 + \cL}\mper
\end{displaymath}
In particular, for $i \neq j$,
\begin{displaymath}
  \Big|\tfrac{1}{c} \pE\nolimits^0 \|x\|_2^2 x_i x_j - \pE^0 x_ix_j \Big|
  = \Big|\tfrac{1}{c} \cL \|x\|_2^2x_i x_j\Big|
  \leq \delta_2 \mper
\end{displaymath}
For $i = j$,
\begin{align*}
  \Big| \tfrac{1}{c}\pE\nolimits^0 \|x\|_2^2 x_i^2 - \pE\nolimits^\mu x_i^2 \Big|
    &= \tfrac{1}{c}\left| \cL \|x\|_2^2 x_i^2
        + \Paren{1 + \frac 1 n } \E\nolimits^{\mu} \|x\|_2^2 x_i^2
        - c\,\E\nolimits^\mu x_i^2 \right|
    \\&= \tfrac{1}{c}\left| \cL \|x\|_2^2 x_i^2 - \tfrac{1}{n} \cL \|x\|_2^4
        + \tfrac{1}{n} \cL \|x\|_2^4
        + \frac{1}{n}\Paren{1 + \frac{1}{n}} \E\nolimits^{\mu} \|x\|_2^4
        - c\,\E\nolimits^\mu x_i^2 \right|
    \\&= \tfrac{1}{c}\left| \cL \|x\|_2^2 x_i^2 - \tfrac{1}{n} \cL \|x\|_2^4
        + \tfrac{1}{n} \pE\nolimits^0 \|x\|_2^4
        - c\,\E\nolimits^\mu x_i^2 \vphantom{\pE^0} \right|
    \\&= \tfrac{1}{c}\left|\cL \|x\|_2^2 x_i^2 - \tfrac{1}{n}\cL \|x\|_2^4
       \vphantom{\pE^0} \right|
    \\&\leq \delta_2'\mper
\end{align*}
Thus, we can take $\xi_1 = \delta_1$, $\xi_2 = \delta_2$,
$\xi_2' = \delta_2'$, and $c = \pE^0 \|x\|_2^4 = 1 + O(1)$, and apply
\pref{lem:deg-one-two} to conclude that
\begin{displaymath}
  \lmin \cL^1 \geq -\tfrac{c-1}{c} - O(n)\xi_1 - O(n^{3/2})\xi_2' - O(n^2) \xi_2 = -O(1)\mper
\end{displaymath}
The functional $\cL^1$ loses a constant factor in the value assigned to $\bA(x)$ as compared to $\pE^0$:
\begin{displaymath}
  \cL\nolimits^1 \bA(x)
    = \frac{\pE\nolimits^0 \bA(x)}{c}
    \ge \Omega\Paren{\frac{n^{3/2}}{\lambda}}\mper
\end{displaymath}

Now using \pref{lem:mix-with-uniform},
we can correct the negative eigenvalue of $\cL^1$ to get a pseudo-expectation
\begin{displaymath}
  \pE \defeq \Theta(1) \cL\nolimits^1 + \Theta(1)\E\nolimits^\mu\mper
\end{displaymath}
By \pref{lem:mix-with-uniform}, the pseudo-expectation $\pE$ satisfies $\{ \|x\|_2^2 = 1 \}$.
Finally, to complete the proof, we have:
\begin{displaymath}
  \pE \bA(x)
    = \Omega \Paren{ \frac{n^{3/2}}{\lambda}}
        + \Theta(1)\E\nolimits^\mu \bA(x)\mper\qedhere
\end{displaymath}
\end{proof}

\section{Higher-Order Tensors}
\label{sec:higher-order-tensors}

We have heretofore restricted ourselves to the case $k = 3$ in our algorithms
for the sake of readability.
In this section we state versions of our main results for general $k$ and
indicate how the proofs from the $3$-tensor case may be generalized to handle
arbitrary $k$.
Our policy is to continue to treat $k$ as constant with respect to $n$,
hiding multiplicative losses in $k$ in our asymptotic notation.

The case of general odd $k$ may be reduced to $k = 3$ by a standard trick, which we describe here for completeness.
Given $\bA$ an order-$k$ tensor, consider the polynomial $\bA(x)$ and make the
variable substitution $y_\beta = x^\beta$ for each multi-index $\beta$ with $|\beta| = (k+1)/2$.
This yields a degree-$3$ polynomial $\bA'(x,y)$ to which the analysis in \pref{sec:sos-proofs} and \pref{sec:sos-algorithm} applies almost unchanged,
now using pseudo-distributions $\{ x, y \}$ satisfying $\{ \|x\|^2 = 1, \|y\|^2 = 1 \}$.
In the analysis of tensor PCA, this change of variables should be conducted after the input is split into signal and noise parts,
in order to preserve the analysis of the second half of the rounding argument (to get from $\pE \iprod{v_0,x}^k$ to $\pE \iprod{v_0,x}$),
which then requires only syntactic modifications to \pref{lem:deg-3-to-1-correlation}.
The only other non-syntactic difference is the need to generalize the $\lambda$-boundedness results for random polynomials
to handle tensors whose dimensions are not all equal; this is already done in \pref{thm:concentration}.

For even $k$, the degree-$k$ SoS approach does not improve on the tensor unfolding algorithms of Montanari and Richard \cite{richard2014tensorpca}.
Indeed, by performing a similar variable substitution, $y_\beta = x^{\beta}$ for all $|\beta| = k/2$,
the SoS algorithm reduces exactly to the eigenvalue/eigenvector computation from tensor unfolding.
If we perform instead the substitution $y_\beta = x^{\beta}$ for $|\beta| = k/2 - 1$,
it becomes possible to extract $v_0$ directly from the degree-$2$ pseudo-moments of an (approximately) optimal degree-$4$ pseudo-distribution,
rather than performing an extra step to recover $v_0$ from $v$ well-correlated with $v_0^{\tensor k/2}$.
Either approach recovers $v_0$ only up to sign, since the input is unchanged under the transformation $v_0 \mapsto -v_0$.

We now state analogues of all our results for general $k$.
Except for the above noted differences from the $k = 3$ case,
the proofs are all easy transformations of the proofs of their degree-$3$ counterparts.

\begin{theorem}
\label{thm:recovery-certify-general-higher}
Let $k$ be an odd integer, $v_0 \in \R^n$ a unit vector, $\tau \succsim n^{k/4} \log(n)^{1/4} / \epsilon$,
and $\bA$ an order-$k$ tensor with independent unit Gaussian entries.
\begin{enumerate}
\item
There is an algorithm, based on semidefinite programming,
which on input $\bT(x) = \tau \cdot \iprod{v_0,x}^k + \bA(x)$
returns a unit vector $v$ with $\iprod{v_0,v} \geq 1 - \epsilon$
with high probability over random choice of $\bA$.

\item
There is an algorithm, based on semidefinite programming,
which on input $\bT(x) = \tau \cdot \iprod{v_0,x}^k + \bA(x)$
certifies that $\bT(x) \leq \tau \cdot \iprod{v,x}^k + O(n^{k/4} \log(n)^{1/4})$
for some unit $v$
with high probability over random choice of $\bA$.
This guarantees in particular that $v$ is close to a maximum likelihood estimator for the problem of recovering the signal $v_0$ from the input $\tau \cdot v_0^{\tensor k} + \bA$.

\item
By solving the semidefinite relaxation approximately,
both algorithms can be implemented in time $\tilde O(m^{1 + 1/k})$, where $m = n^k$ is the input size.
\end{enumerate}
For even $k$, the above all hold, except now we recover $v$ with $\iprod{v_0,v}^2 \geq 1 - \epsilon$,
and the algorithms can be implemented in nearly-linear time.
\end{theorem}

The next theorem partially resolves a conjecture of Montanari and Richard regarding tensor unfolding algorithms for odd $k$.
We are able to prove their conjectured signal-to-noise ratio $\tau$, but under an asymmetric noise model.
They conjecture that the following holds when $\bA$ is symmetric with unit Gaussian entries.

\begin{theorem}
\label{thm:faster-general-higher}
Let $k$ be an odd integer, $v_0 \in \R^n$ a unit vector, $\tau \succsim n^{k/4} / \epsilon$,
and $\bA$ an order-$k$ tensor with independent unit Gaussian entries.
There is a nearly-linear-time algorithm, based on tensor unfolding, which,
with high probability over random choice of $\bA$,
recovers a vector $v$ with $\iprod{v,v_0}^2 \geq 1 - \epsilon$.
\end{theorem}

\section{Conclusion}

\subsection*{Open Problems}
One theme in this work has been efficiently certifying upper bounds on homogeneous polynomials with random coefficients.
It is an interesting question to see whether one can (perhaps with the degree $d > 4$ SoS meta-algorithm) give an algorithm certifying a bound of $n^{3/4 - \delta}$ over the unit sphere on a degree $3$ polynomial with standard Gaussian coefficients.
Such an algorithm would likely yield improved signal-to-noise guarantees for tensor PCA, and would be of interest in its own right.

Conversely, another problem is to extend our lower bound to handle degree $d > 4$ SoS.
Together, these two problems suggest (as was independently suggested to us by Boaz Barak) the problem of characterizing the SoS degree required to certify a bound of $n^{3/4 - \delta}$ as above.

Another problem is to simplify the linear time algorithm we give for tensor PCA under symmetric noise.
Montanari and Richard's conjecture can be interpreted to say that the random rotations and decomposition into submatrices involved in our algorithm are unnecessary, and that in fact our linear time algorithm for recovery under asymmetric noise actually succeeds in the symmetric case.

\subsection*{Acknowledgments}
We thank Moses Charikar for bringing to our attention the work of Montanari and Richard.
We would like to thank Boaz Barak, Rong Ge, and Ankur Moitra for enlightening conversations.
S. B. H. acknowledges the support of an NSF Graduate Research Fellowship
under award no. 1144153.
D. S. acknowledges support from the Simons Foundation, the National Science Foundation, an Alfred P. Sloan Fellowship, and a Microsoft Research Faculty Fellowship,
A large portion of this work was completed while the authors were long-term visitors to the Simons Institute for the Theory of Computing (Berkeley) for the program on Algorithmic Spectral Graph Theory.

\addreferencesection
\bibliographystyle{amsalpha}
\bibliography{bib/mr,bib/dblp,bib/scholar,bib/tensor-pca}

\appendix

\section{Pseudo-Distribution Facts}
\label{pseudoexpectation-facts}
\begin{lemma}[Quadric Sampling]
  \label{lem:quadratic-sampling}
  Let $\{ x \}$ be a pseudo-distribution over $\R^n$ of degree $d \geq 2$.
  Then there is an actual distribution $\{ y \}$
  over $\R^n$ so that for any polynomial $p$ of degree at most $2$,
  $\E[p(y)] = \pE[p(x)]$.
  Furthermore, $\{ y \}$ can be sampled from in time $\poly n$.
\end{lemma}

\begin{lemma}[Pseudo-Cauchy-Schwarz, Function Version, \cite{DBLP:conf/stoc/BarakBHKSZ12}]
  \label{lem:pseudo-cs}
  Let $x, y$ be vector-valued polynomials. Then
  \begin{displaymath}
    \iprod{x,y} \preceq \frac 1 2 (\|x\|^2 + \|y \|^2).
  \end{displaymath}
\end{lemma}
See \cite{DBLP:conf/stoc/BarakKS14} for the cleanest proof.

 \begin{lemma}[Pseudo-Cauchy-Schwarz, Powered Function Version]
   \label{lem:pseudo-cs-powered}
   Let $x, y$ be vector-valued polynomials and $d > 0$ an integer. Then
   \begin{displaymath}
     \iprod{x,y}^d \preceq \frac 1 2 (\| x \|^{2d} + \| y \|^{2d}).
   \end{displaymath}
 \end{lemma}
 \begin{proof}
   Note that $\iprod{x,y}^d = \iprod{x^{\tensor d}, y^{\tensor d}}$  and apply \pref{lem:pseudo-cs}.
 \end{proof}

Yet another version of pseudo-Cauchy-Schwarz will be useful:
\begin{lemma}[Pseudo-Cauchy-Schwarz, Multiplicative Function Version, \cite{DBLP:conf/stoc/BarakBHKSZ12}]
  \label{lem:pseudo-cs-multiplicative}
  Let $\{x,y\}$ be a degree $d$ pseudo-distribution over a pair of vectors, $d \geq 2$. Then
  \begin{displaymath}
    \pE[\iprod{x,y}] \leq \sqrt{\pE[\| x \|^2]} \sqrt{\pE[\| y \|^2]}.
  \end{displaymath}
\end{lemma}
Again, see \cite{DBLP:conf/stoc/BarakKS14} for the cleanest proof.

We will need the following inequality relating $\pE \iprod{x,v_0}^3$ and $\pE \iprod{x,v_0}$
when $\pE \iprod{x,v_0}^3$ is large.

\begin{lemma}
  \label{lem:deg-3-to-1-correlation}
  Let $\{ x \}$ be a degree-$4$ pseudo-distribution satisfying $\{ \|x\|^2 = 1 \}$,
  and let $v_0 \in \R^n$ be a unit vector.
  Suppose that $\pE \iprod{x,v_0}^3 \geq 1 - \epsilon$ for some $\epsilon \geq 0$.
  Then $\pE \iprod{x,v_0} \geq 1 - 2\epsilon$.
\end{lemma}
\begin{proof}
  Let $p(u)$ be the univariate polynomial $p(u) = 1 - 2u^3 + u$. It is easy to check
  that $p(u) \geq 0$ for $u \in [-1,1]$. It follows from classical results
  about univariate polynomials that $p(u)$ then can
  be written as
  \[
    p(u) = s_0(u) + s_1(u)(1 + u) + s_2(u)(1 - u)
  \]
  for some SoS polynomials $s_0, s_1, s_2$ of degrees at most $2$. (See \cite{DBLP:conf/soda/ODonnellZ13},
  fact 3.2 for a precise statement and attributions.)

  Now we consider
  \[
    \pE p(\iprod{x,v_0}) \geq \pE[s_1(\iprod{x,v_0})(1 + \iprod{x,v_0})] + \pE [s_2(\iprod{x,v_0})
    (1 - \iprod{x,v_0})]\mper
  \]
  We have by \pref{lem:pseudo-cs} that $\iprod{x,v_0} \preceq \frac 1 2 (\|x\|^2 + 1)$
  and also that $\iprod{x,v_0} \succeq -\frac 1 2 (\|x\|^2 + 1)$. Multiplying the latter SoS
  relation
  by the SoS polynomial $s_1(\iprod{x,v_0})$ and the former by $s_2(\iprod{x,v_0})$, we
  get that
  \begin{align*}
    \pE[s_1(\iprod{x,v_0})(1 + \iprod{x,v_0})] & = \pE[s_1(\iprod{x,v_0})]
    + \pE[s_1(\iprod{x,v_0}) \iprod{x,v_0}]\\
    & \geq \pE[s_1(\iprod{x,v_0})] - \frac 1 2 \pE[s_1(\iprod{x,v_0})(\|x\|^2 + 1)]\\
    & \geq  \pE[s_1(\iprod{x,v_0})] -\pE[s_1(\iprod{x,v_0})]\\
    & \geq 0\mcom
  \end{align*}
  where in the second-to-last step we have used the
  assumption that $\{ x \}$ satisfies $\{ \|x\|^2 = 1 \}$. A similar analysis yields
  \[
    \pE[s_2(\iprod{x,v_0})(1 - \iprod{x,v_0})] \geq 0\mper
  \]
  All together, this means that $\pE p(\iprod{x,v_0}) \geq 0$. Expanding, we get
  $\pE[1 - 2\iprod{x,v_0}^3 + \iprod{x,v_0}] \geq 0$. Rearranging yields
  \[
    \pE \iprod{x,v_0} \geq 2 \pE \iprod{x,v_0}^3 - 1 \geq 2(1 - \epsilon) - 1 \geq 1 - 2\epsilon\mper\qedhere
  \]
\end{proof}

We will need a bound on the pseudo-expectation of a degree-$3$ polynomial in terms of the operator norm of its coefficient matrix.

\begin{lemma}
\label{lem:sos-bound-from-op-norm}
Let $\{ x \}$ be a degree-$4$ pseudo-distribution.
Let $M \in \R^{n^2 \times n}$.
Then $\pE \iprod{x^{\tensor 2}, Mx} \leq \|M\| (\pE \|x\|^4)^{3/4}$.
\end{lemma}
\begin{proof}
  We begin by expanding in the monomial basis and using pseudo-Cauchy-Schwarz:
  \begin{align*}
  \pE \iprod{x^{\tensor 2}, Mx} & = \pE \sum_{ijk} M_{(j,k),i} x_i x_j x_k\\
  & = \pE \sum_i x_i \sum_{jk} M_{(j,k),i} x_j x_k\\
  & \leq (\pE \|x\|^2)^{1/2} \Brac{\pE \sum_i \Paren{\sum_{jk} M_{(j,k),i} x_i x_j}^2}^{1/2}\\
  & \leq (\pE \|x\|^4)^{1/4} \Brac{\pE \sum_i \Paren{\sum_{jk} M_{(j,k),i} x_i x_j}^2}^{1/2}\\
  \end{align*}
  We observe that $MM^T$ is a matrix representation of $\sum_i \Paren{\sum_{jk} M_{(j,k),i} x_i x_j}^2$.
  We know $MM^T \preceq \|M\|^2 \Id$, so
  \[
  \pE \sum_i \Paren{\sum_{jk} M_{(j,k),i} x_i x_j}^2 \leq \|M\|^2 \pE \|x\|^4\mper
  \]
  Putting it together, we get $\pE \iprod{x^{\tensor 2}, Mx} \leq \|M\| (\pE\|x\|^4)^{3/4}$ as desired.
\end{proof}

\section{Concentration bounds}
\label{sec:concentration-bounds}
\subsection{Elementary Random Matrix Review}
We will be extensively concerned with various real random matrices.
A great deal is known about natural classes of such matrices; see the excellent book of Tao \cite{tao2012matrices} and the notes by Vershynin and Tropp \cite{vershynin2011random,DBLP:journals/focm/Tropp12}.

Our presentation here follows Vershynin's \cite{vershynin2011random}.
Let $X$ be a real random variable.
The subgaussian norm $\| X \|_{\psi_2}$ of $X$ is $\sup_{p \geq 1} p^{-1/2} (\E \abs{X}^p)^{1/p}$.
Let $\{ a \}$ be a distribution on $\R^n$. The subgaussian norm $\| a \|_{\psi_2}$ of $\{ a \}$ is the maximal subgaussian norm of the one-dimensional marginals: $\| a \|_{\psi_2} = \sup_{ \|u\| = 1} \|\iprod{a,u}\|_{\psi_2}$.
A family of random variables $\{ X_n \}_{n \in \mathbb N}$ is subgaussian if $\|X_n\|_{\psi_2} = O(1)$.
The reader may easily check that an $n$-dimensional vector of independent standard Gaussians or independent $\pm 1$ variables is subgaussian.

It will be convenient to use the following standard result on the concentration of empirical covariance matrices.
This statement is borrowed from \cite{vershynin2011random}, Corollary 5.50.
\begin{lemma}
  \label{lem:covariance-concentration}
  Consider a sub-gaussian distribution $\{ a \}$ in $\R^m$ with covariance matrix $\Sigma$, and let $\delta \in (0,1), t \geq 1$.
  If $a_1,\ldots,a_N \sim \{ a \}$ with $N \geq C(t/\delta)^2 m$ then $\| \frac 1 N \sum a_i a_i^T - \Sigma \| \leq \delta$ with probability at least $1 - 2\exp(-t^2 m)$.
  Here $C = C(K)$ depends only on the sub-gaussian norm $K  = \| a \|_{\psi_2}$ of a random vector taken from this distribution.
\end{lemma}

We will also need the matrix Bernstein inequality.
This statement is borrowed from Theorem 1.6.2 of Tropp \cite{DBLP:journals/focm/Tropp12}.

\begin{theorem}[Matrix Bernstein]
  Let $S_1,\ldots,S_m$ be independent square random matrices with dimension $n$.
  Assume that each matrix has bounded deviation from its mean:
  $\| S_i - \E S_i \| \leq R$ for all $i$.
  Form the sum $Z = \sum_i S_i$ and introduce a variance parameter
  \begin{displaymath}
    \sigma^2 = \max \{ \| \E (Z - \E Z)(Z - \E Z)^T \|, \|\E(Z - \E Z)^T (Z - \E Z) \| \} \mper
  \end{displaymath}
  Then
  \begin{displaymath}
    \Pr \{ \| Z - \E Z \| \geq t \} \leq 2n \exp \Paren{\frac{t^2/2}{\sigma^2 + Rt/3}} \qquad \mbox{for all } t \geq 0\mper
  \end{displaymath}
\end{theorem}

We will need bounds on the operator norm of random square rectangular matrices, both of which are special cases of Theorem 5.39 in \cite{vershynin2011random}.
\begin{lemma}
  \label{lem:square-concentration}
   Let $A$ be an $n \times n$ matrix with independent entries from $\cN(0,1)$.
   Then with probability $1 - n^{-\omega(1)}$, the operator norm $\|A\|$ satisfies $\|A\|\leq O(\sqrt n)$.
\end{lemma}

\begin{lemma}
\label{lem:rectangular-concentration}
  Let $A$ be an $n^2 \times n$ matrix with independent entries from $\cN(0,1)$.
  Then with probability $1 - n^{-\omega(1)}$, the operator norm $\|A\|$ satisfies $\|A\|\leq O(n)$.
\end{lemma}

\subsection{Concentration for $\sum_i A_i \tensor A_i$ and Related Ensembles}
Our first concentration theorem provides control over the nontrivial permutations of the matrix $AA^T$ under the action of $\cS_4$ for a tensor $\bA$ with independent entries.

\begin{theorem}
\label{thm:concentration}
  Let $c \in \{ 1,2 \}$ and $d \geq 1$ an integer.
  Let $A_1,\ldots,A_{n^c}$ be iid random matrices in $\{\pm 1\}^{n^d \times n^d}$ or with independent entries from $\cN(0,1)$.
  Then, with probability $1 - O(n^{-100})$,
  \begin{displaymath}
    \Norm{\sum_{i \in [n^c]} A_i \otimes A_i - \E A_i \tensor A_i} \precsim \sqrt{d} n^{(2d+c)/2}\cdot (\log n)^{1/2}\mper
  \end{displaymath}
  and
  \begin{displaymath}
	\Norm{\sum_{i \in [n^c]} A_i \otimes A_i^T - \E A_i \tensor A_i^T} \precsim \sqrt{d} n^{(2d+c)/2}\cdot (\log n)^{1/2}\mper
  \end{displaymath}
\end{theorem}

We can prove \pref{thm:concentration-3} as a corollary of the above.
\begin{proof}[Proof of \pref{thm:concentration-3}]
  Let $A$ have iid Gaussian entries.
  We claim that $\E A \tensor A$ is a matrix representation of $\|x\|^4$.
  To see this, we compute
  \begin{align*}
    \iprod{x^{\tensor 2}, \E (A \tensor A) x^{\tensor 2}} & = \E \iprod{x, Ax}^2\\
     & = \sum_{i,j,k,l} \E A_{ij}A_{kl} x_i x_j x_k x_l\\
     & = \sum_{ij}  x_i^2 x_j^2 \\
     & = \|x\|^4\mper
  \end{align*}
  Now by \pref{thm:concentration}, we know that for $A_i$ the slices of the tensor $\bA$ from the statement of \pref{thm:concentration-3},
  \[
    \sum_i A_i \tensor A_i \preceq n \E A \tensor A + \lambda^2 \cdot \Id
  \]
  for $\lambda = O(n^{3/4} \log(n)^{1/4})$.
  Since $n = O(\lambda)$ and both $\Id$ and $\E A \tensor A$ are matrix representations of $\|x\|^4$, we are done.
\end{proof}

Now we prove \pref{thm:concentration}.
We will prove only the statement about $\sum_i  A_i \tensor A_i$, as the case of $\sum_i A_i \tensor A_i^T$ is similar.

Let $A_1,\ldots,A_{n^c}$ be as in \pref{thm:concentration}.
We first need to get a handle on their norms individually, for which we need the following lemma.
\begin{lemma}
  \label{lem:random-matrix-concentration}
  Let $A$ be a random matrix in $\{\pm 1\}^{n^d \times n^d}$ or with independent entries from $\cN(0,1)$.
  For all $t\ge 1$, the probability of the event $\{\lVert A \rVert> tn^{d/2}\}$ is at most $2^{-t^2{n^d}/K}$ for some absolute constant $K$.
\end{lemma}
\begin{proof}
  The subgaussian norm of the rows of $A$ is constant and they are identically and isotropically distributed.
  Hence Theorem 5.39 of \cite{vershynin2011random} applies to give the result.
\end{proof}

Since the norms of the matrices $A_1,\ldots,A_{n^c}$ are concentrated around $n^{d/2}$ (by \pref{lem:random-matrix-concentration}),
it will be enough to prove \pref{thm:concentration} after truncating the matrices $A_1,\ldots,A_{n^c}$.
For $t\ge 1$, define iid random matrices $A'_1,\ldots,A'_{n^c}$ such that
\begin{displaymath}
  A_i' \defeq
  \begin{cases}
    A_i & \text{ if $\lVert  A_i\rVert\le t n^{d/2}$},\\
    0 & \text{ otherwise}
  \end{cases}
\end{displaymath}

for some $t$ to be chosen later.
\pref{lem:random-matrix-concentration} allows us to show that the random matrices $A_i\otimes A_i$ and $A'_i\otimes A'_i$ have almost the same expectation.
For the remainder of this section, let $K$ be the absolute constant from \pref{lem:random-matrix-concentration}.

\begin{lemma}
\label{lem:close-expectation}
  For every $i\in [n^c]$ and all $t\ge 1$, the expectations of $A_i \otimes A_i$ and $A'_i\otimes A'_i$ satisfy
  \begin{displaymath}
    \left\lVert \E [A_i\otimes A_i] - \E [A'_i\otimes A'_i] \right\rVert
    \le O(1)\cdot 2^{-t{n^d}/K}\mper
  \end{displaymath}
\end{lemma}
\begin{proof}
  Using Jensen's inequality and that $A_i = A_i'$ unless $\|A_i\| > tn^{d/2}$, we have
  \begin{align*}
    \|\E A_i \tensor A_i - A_i' \tensor A_i' \| & \leq \E \|A_i \tensor A_i - A_i' \tensor A_i'\| \quad \text{Jensen's inequality}\\
    &  = \int^\infty_{tn^{d/2}} \Pr(\|A_i\| \geq \sqrt s) \, ds \quad \text{since $A_i = A_i'$ unless $\|A_i\| \geq tn^{d/2}$}\\
&  \leq \int_{tn^{d/2}}^\infty 2^{-s/K} \, ds \quad \text{by \pref{lem:random-matrix-concentration}}\\
& \leq \sum_{i=0}^\infty 2^{-tn^{d/2}/K} \cdot 2^{-i/K} \quad \text{discretizing the integral}\\
& = O(2^{-tn^{d/2}/K}) \quad \text{as desired.}\qedhere
  \end{align*}
\end{proof}

\begin{lemma}
\label{lem:apply-bernstein}
  Let $B'_1,\ldots,B'_{n^c}$ be i.i.d. matrices such that $B'_i=A_i'\otimes A'_i-\E [A'_i\otimes A'_i]$.
  Then for every $C\ge 1$ with $C \le 3 t^2 n^{c/2}$,
  \begin{displaymath}
    \Pr \left \{ \Norm{\sum_{i \in [n^c]} B_i'} > C\cdot n^{(2d + c)/2}  \right\}
    \le  2 n^{2d} \cdot \exp\Paren{\frac{-C^2}{6 t^4}}\mper
  \end{displaymath}
\end{lemma}

\begin{proof}
  For $R=2 t^2 n^d$, the random matrices $B'_1,\ldots,B'_{n^c}$ satisfy $\{\lVert B'_i\rVert \le R \}$ with probability $1$.
  Therefore, by the Bernstein bound for non-symmetric matrices \cite[Theorem 1.6]{DBLP:journals/focm/Tropp12},
  \begin{displaymath}
    \Pr \left\{ \left\lVert  \sum\nolimits_{i=1}^{n^c} B'_i\right \rVert\ge s\right\}
    \le 2n^{2d} \cdot\exp\Paren{\frac {-s^2/2}{\sigma^2+Rs/3}}\mcom
  \end{displaymath}
  where $\sigma^2=\max\{\lVert  \sum_i \E B'_i(B'_i)^\top \rVert, \lVert  \sum_i \E (B'_i)^\top B'_i \rVert\}\le n^c \cdot R^2$.
  For $s=C\cdot n^{(2d + c)/2}$, the probability is bounded by
  \begin{displaymath}
    \Pr \left\{ \left\lVert  \sum\nolimits_{i=1}^n B'_i\right \rVert\ge s\right\}
    \le 2n^{2d} \cdot\exp\Paren{\frac {-C^2 \cdot n^{(2d+c)}/2}{4t^4\cdot n^{2d+c} + 2t^2C\cdot n^{(4d+c)/2}/3}}\mper
  \end{displaymath}
  Since our parameters satisfy $t^2 C \cdot n^{(4d+c)/2}/3 \le t^4 n^{(2d+c)}$, this probability is bounded by
  \begin{displaymath}
    \Pr \left\{ \left\lVert  \sum\nolimits_{i=1}^n B'_i\right \rVert\ge s\right\}
    \le 2 n^{2d} \cdot \exp\Paren{\frac{-C^2}{6 t^4}}\mper
    \qedhere
  \end{displaymath}
\end{proof}

At this point, we have all components of the proof of \pref{thm:concentration}.

\begin{proof}[Proof of \pref{thm:concentration} for $\sum_i A_i \tensor A_i$ (other case is similar)]
  By \pref{lem:apply-bernstein},
  \begin{displaymath}
    \Pr \Biggl\{ \left\lVert  \sum\nolimits_i A'_i\otimes A'_i - \sum\nolimits_i \E [A'_i\otimes A'_i] \right\rVert > C\cdot n^{(2d+c)/2}  \Biggr\}
    \le  2 n^{2d} \cdot \exp\Paren{\frac{-C^2}{K t^4}}\mper
  \end{displaymath}

  At the same time, by \pref{lem:random-matrix-concentration} and a union bound,
  \begin{displaymath}
    \Pr\left\{\vbig A_1=A'_1,\ldots,A_n=A'_{n^c}\right\} \ge 1-n^c\cdot 2^{-t^2{n^d}/K}\mper
  \end{displaymath}
  By \pref{lem:close-expectation} and triangle inequality,
  \begin{displaymath}
    \left\lVert  \sum\nolimits_i \E[A_i\otimes A_i] - \sum\nolimits_i \E[ A'_i\otimes A_i'] \right\rVert \le n^{c} \cdot 2^{-tn^d/K}\mper
  \end{displaymath}
  Together, these bounds imply
  \begin{multline*}
    \Pr \Biggl\{ \left\lVert  \sum\nolimits_i A_i\otimes A_i - \sum\nolimits_i \E [A_i\otimes A_i] \right\rVert > C\cdot n^{(2d+c)/2} + n^{c}\cdot 2^{-t{n^d}/K} \Biggr\}\\
    \le  2 n^{2d} \cdot \exp\Paren{\frac{-C^2}{K t^4}} + n^c \cdot 2^{-t^2{n^d}/K}\mper
  \end{multline*}
  We choose $t=1$ and $C=100 \sqrt {2K d \log n}$ and assume that $n$ is large enough so that $C\cdot n^{(2d + c)/2}\ge n^{c}\cdot 2^{-t{n^d}/K}$ and $2 n^{2d} \cdot \exp\Paren{\frac{-C^2}{K t^4}}\ge n^c \cdot 2^{-t^2{n^d}/K}$.
  Then the probability satisfies
  \begin{displaymath}
    \Pr \Biggl\{ \left\lVert  \sum\nolimits_i A_i\otimes A_i - \sum\nolimits_i \E [A_i\otimes A_i] \right\rVert > 20 n^{(2d + c)/2}{\textstyle \sqrt{ 2Kd \log n }} \Biggr\}
    \le  4 n^{-100}\mper
    \qedhere
  \end{displaymath}
\end{proof}

\subsection{Concentration for Spectral SoS Analyses}
\begin{lemma}[Restatement of \pref{lem:nonlinear-unfolding-noise}]
  Let $\bT = \tau \cdot v_0^{\tensor 3} + \bA$.
  Suppose $\bA$ has independent entries from $\cN(0,1)$.
  Then with probability $1 - O(n^{-100})$ we have $\| \sum_i A_i \tensor A_i - \E \sum_i A_i \tensor A_i \| \leq O(n^{3/2} \log(n)^{1/2})$ and $\| \sum_i v_0(i) A_i \| \leq O(\sqrt{n})$.
\end{lemma}
\begin{proof}
  The first claim is immediate from \pref{thm:concentration}.
  For the second, we note that since $v_0$ is a unit vector, the matrix $\sum_i v_0(i) A_i$ has independent entries from $\cN(0,1)$.
  Thus, by \pref{lem:square-concentration}, $\|\sum_i v_0(i) A_i \| \leq O(\sqrt n)$ with probability $1 - O(n^{-100})$, as desired.
\end{proof}

\begin{lemma}[Restatement of \pref{lem:mr-concentration} for General Odd $k$]
\label{lem:mr-concentration-general}
  Let $\bA$ be a $k$-tensor with $k$ an odd integer, with independent entries from $\cN(0,1)$.
  Let $v_0 \in \R^n$ be a unit vector, and let $V$ be the $n^{(k+1)/2} \times n^{(k-1)/2}$ unfolding of $v_0^{\tensor k}$.
  Let $A$ be the $n^{(k+1)/2} \times n^{(k-1)/2}$ unfolding of $\bA$.
  Then with probability $1 - O(n^{-100})$, the matrix $A$ satisfies $A^T A = n^{(k+1)/2} I + E$ for some $E$ with $\| E \| \leq O(n^{k/2}\log(n))$ and $\| A^T V \| \leq O(n^{(k-1)/4} \log(n)^{1/2})$.
\end{lemma}

\begin{proof}
  With $\delta = O(1/\sqrt{n})$ and $t = 1$, our parameters will satisfy $n^{(k+1)/2} \geq (t/\delta)^2 n^{(k-1)/2}$.
  Hence, by \pref{lem:covariance-concentration},
  \begin{displaymath}
    \| E \| = \| A^T A^T - n^{(k+1)/2} I \| = \left \lVert \sum_{|\alpha| = (k+1)/2} a_\alpha a_\alpha^T - n^{(k+1)/2} \cdot \Id \right \rVert
      \leq n^{(k+1)/2} \cdot O\Paren{\frac{1}{\sqrt{n}}} = O(n^{k/2})
  \end{displaymath}
  with probability at least $1 - 2\exp(-n^{(k+1)/2}) \geq 1 - O(n^{-100})$.

  It remains to bound $\| A^T V \|$.
  Note that $V = uw^T$ for fixed unit vectors $u \in \R^{(k-1)/2}$ and $w \in \R^{(k+1)/2}$.
  So $\|A^T V \| \leq \|A^T u\|$. But $A^T u$ is distributed according to $\cN(0,1)^n$ and so $\|A^Tu\| \leq O(\sqrt{n \log n})$ with probability $1 -n^{-100}$ by standard arguments.
\end{proof}

\subsection{Concentration for Lower Bounds}
The next theorems collects the concentration results necessary to apply our lower bounds \pref{thm:lower-bound-4} and \pref{thm:lower-bound} to random polynomials.
\begin{lemma}
  \label{lem:concentration-lb}
  Let $\bA$ be a random $3$-tensor with unit Gaussian entries.
  For a real parameter $\lambda,$
  let $\cL : \R[x]_4 \rightarrow \R$ be the linear operator whose matrix
  representation $M_{\cL}$ is given by
  $M_{\cL} \seteq \frac 1 {n^2 \lambda^2} \sum_{\pi \in \cS_4} \pi \cdot AA^T$.
  There is $\lambda = O(n^{3/2} / \log(n)^{1/2})$ so that
  with probability $1 - O(n^{-50})$ the following events all occur for every $\pi \in \cS_3$.
\begin{align}
  \label{eq:1} -2\lambda^2 & \cdot \Pi\Id\Pi \notag \\
                           & \preceq \frac 1 2 \Pi\left[\sigma \cdot A^\pi (A^\pi)^T + \sigma^2 \cdot A^\pi(A^\pi)^T + (\sigma \cdot A^\pi (A^\pi)^T)^T + ( \sigma^2 \cdot A^\pi(A^\pi)^T )^T \right]\Pi
\end{align}
\begin{align}
  \label{eq:3}                                                       \iprod {\bA, \sum_{\pi \in \cS_3} \bA^\pi } = \Omega(n^3) \\
  \label{eq:4}                                \iprod{\symId, A^\pi (A^\pi)^T } = O(n^3) \\
  \label{eq:5} n \Paren{\max_i \left | \frac 1 {\lambda n^{3/2} } \iprod{\Id_{n \times n}, A^\pi_i} \right |} = O(1)\\
  \label{eq:6} n^2 \Paren{\max_{i \neq j} \left | \cL \|x\|^2x_ix_j \right |} = O(1) \\
  \label{eq:7} n^{3/2} \Paren{\max_i \left | \cL \|x\|^2 x_i^2 - \tfrac{1}{n} \cL \|x\|^4 \vphantom{\pE^0} \right |}  = O(1/n)
\end{align}
\end{lemma}
\begin{proof}
  For \pref{eq:1}, from \pref{thm:concentration}, \pref{lem:deg-4-sym-AAT},
  the observation that multiplication by an orthogonal operator cannot increase
  the operator norm, a union bound over all $\pi$, and the triangle inequality,
  it follows that:
  \[ \left\| \sigma \cdot A^\pi(A^{\pi})^T - \E[\sigma \cdot A^\pi(A^\pi)^T]
    + \sigma^2 \cdot A^\pi(A^\pi)^T - \E[\sigma^2 \cdot A^\pi(A^\pi)^T] \right\| \le 2\lambda^2. \]
  with probability $1 - n^{-100}$.
  By the definition of the operator norm and another application of triangle inequality, this implies
  \begin{align*}
    -4\lambda^2\Id & \preceq \sigma \cdot A^\pi(A^{\pi})^T + \sigma^2 \cdot A^\pi(A^\pi)^T
    + (\sigma \cdot A^\pi (A^\pi)^T)^T + (\sigma^2 \cdot A^\pi (A^\pi)^T)^T\\
  & - \E[\sigma \cdot A^\pi(A^\pi)^T] - \E[\sigma^2 \cdot A^\pi(A^\pi)^T]
    - \E[(\sigma \cdot A^\pi (A^\pi)^T)^T] - \E[(\sigma^2 \cdot A^\pi (A^\pi)^T)^T]\mper
  \end{align*}

  We note that $\E[\sigma \cdot A^\pi(A^\pi)^T] = \sigma \cdot \Id$ and
  $\E[\sigma^2 \cdot A^\pi(A^\pi)^T] = \sigma^2 \cdot \Id$, and the same for their transposes,
  and that
  $\Pi(\sigma\cdot\Id + \sigma^2\cdot\Id)\Pi \succeq 0$.
  So, dividing by $2$ and projecting onto the $\Pi$ subspace:
  \begin{align*}
    -2\lambda^2 & \cdot \Pi\Id\Pi \\
                             & \preceq \frac 1 2 \Pi\left(\sigma \cdot A^\pi (A^\pi)^T + \sigma^2 \cdot A^\pi(A^\pi)^T + (\sigma \cdot A^\pi (A^\pi)^T)^T + ( \sigma^2 \cdot A^\pi(A^\pi)^T )^T \right)\Pi\mper
  \end{align*}

  We turn to \pref{eq:3}.
  By a Chernoff bound, $\iprod{ \bA, \bA } = \Omega(n^3)$ with probability $1 - n^{-100}$.
  Let $\pi \in \cS_3$ be a nontrivial permutation.
  To each multi-index $\alpha$ with $|\alpha| = 3$ we associate its orbit $\cO_\alpha$ under $\langle \pi \rangle$.
  If $\alpha$ has three distinct indices, then $|\cO_\alpha| > 1$ and $\sum_{\beta \in \cO_\alpha} A_\beta A^\pi_\beta$ is a random variable $X_\alpha$ with the following properties:
  \begin{itemize}
    \item $|X_\alpha| < O(\log n)$ with probability $1 - n^{-\omega(1)}$.
    \item $X_\alpha$ and $-X_\alpha$ are identically distributed.
  \end{itemize}
  Next, we observe that we can decompose
  \begin{displaymath}
    \iprod{ \bA, \bA^\pi } = \sum_{|\alpha| = 3} \bA_\alpha \bA^\pi_\alpha = R + \sum_{\cO_\alpha} X_\alpha\mcom
  \end{displaymath}
  where $R$ is the sum over multi-indices $\alpha$ with repeated indices, and therefore has $|R| = \tilde O(n^2)$ with probability $1 - n^{-100}$.
  By a standard Chernoff bound, $|\sum_{\cO_\alpha} X_\alpha| = O(n^2)$ with probability $1 - O(n^{-100})$.
  By a union bound over all $\pi$, we get that with probability $1 - O(n^{-100})$,
  \begin{displaymath}
    \iprod{ \bA, \sum_{\pi \in \cS_3} \bA^\pi } = n^3 - O(n^2) = \Omega(n^3)\mcom
  \end{displaymath}
  establishing \pref{eq:3}.

  Next up is \pref{eq:4}.
  Because $A^\pi$ are identically distributed for all $\pi \in \cS_3$ we assume without loss of generality that $A^\pi = A$.
  The matrix $\symId$ has $O(n^2)$ nonzero entries. Any individual entry of $AA^T$ is with probability $1 - n^{-\omega(1)}$ at most $O(n)$.
  So $\iprod{\symId, AA^T} = O(n^3)$ with probability $1 - O(n^{-100})$.

  Next, \pref{eq:5}.
  As before, we assume without loss of generality that $\pi$ is the trivial permutation.
  For fixed $1 \leq i \leq n$, we have $\iprod{\Id_{n \times n},A_i} =  \sum_j A_{ijj}$,
  which is a sum of $n$ independent unit Gaussians, so $|\iprod{\Id_{n \times n},A_i}| \leq O(\sqrt n \log n)$ with probability $1 - n^{-\omega(1)}$.
  By a union bound over $i$ this also holds for $\max_i |\iprod{\Id_{n\times n}, A_i} |$.
  Thus with probability $1 - O(n^{-100})$,
  \begin{displaymath}
    n \Paren{\max_i \left | \frac 1 {n^{3/2} \lambda } \iprod{\Id_{n \times n}, A_i} \right |} \leq
    \frac {\tilde O(1)} {\lambda}\mper
  \end{displaymath}

  Last up are \pref{eq:6} and \pref{eq:7}.
  Since we will do a union bound later, we fix $i,j \leq n$. Let $w \in \R^{n^2}$ be the matrix flattening of $\Id_{n \times n}$.
  We expand $\cL \|x\|^2 x_i x_j$ as
  \begin{align*}
    \cL \|x \|^2 x_i x_j & = \frac 1 {n^2 O(\lambda^2)}( w^T \Pi (AA^T + \sigma \cdot AA^T + \sigma^2 \cdot AA^T)\Pi(e_i \tensor e_j)  \\
                           & + w^T \Pi (AA^T + \sigma \cdot AA^T + \sigma^2 \cdot AA^T)^T\Pi (e_i \tensor e_j) )
    \mper
  \end{align*}
  We have $\Pi w = w$ and we let $e_{ij} \seteq \Pi(e_i \tensor e_j) = \frac 1 2 (e_i \tensor e_j + e_j \tensor e_i)$.
  So using \pref{lem:deg-4-sym-AAT},
  \begin{align*}
    n^2 O(\lambda^2) \cL \|x\|^2 x_i x_j & = w^T (AA^T + \sigma \cdot AA^T + \sigma^2 \cdot AA^T)e_{ij}\\
                                           & + w^T (AA^T + \sigma \cdot AA^T + \sigma^2 \cdot AA^T)^Te_{ij}\\
                                           & = w^T \Big ( AA^T e_{ij}\\
                                           & + \frac 1 2 \sum_k A_k e_j \tensor A_k e_i + \frac 1 2 \sum_k A_k e_i \tensor A_k e_j\\
                                           & + \frac 1 2 \sum_k A_k^T e_j \tensor A_k^T e_i + \frac 1 2 \sum_k A_k^T e_i \tensor A_k^T e_j\\
                                           & + \frac 1 2 \sum_k A_k e_i \tensor A_k^T e_j + \frac 1 2 \sum_k A_k e_j \tensor A_k^T e_i\\
                                           & + \frac 1 2 \sum_k A_k^T e_i \tensor A_k e_j + \frac 1 2 \sum_k A_k^T e_j \tensor A_k e_i \Big )\mper
  \end{align*}
  For $i \neq j$, each term $w^T (A_k e_j \tensor A_k e_i)$ (or similar, with various transposes) is the sum of $n$ independent products of pairs of independent unit Gaussians,
  so by a Chernoff bound followed by a union bound, with probability $1 - n^{-\omega(1)}$ all of them are $O(\sqrt n \log n)$.
  There are $O(n)$ such terms, for an upper bound of $O(n^{3/2}(\log n))$ on the contribution from the tensored parts.

  At the same time, $w^T A$ is a sum $\sum_k a_{kk}$ of $n$ rows of $A$ and $Ae_{ij}$ is the average of two rows of $A$;
  since $i \neq j$ these rows are independent from $w^T A$. Writing this out, $w^T AA^T e_{ij} = \frac 1 2 \sum_k \iprod{a_{kk}, a_{ij} + a_{ji}}$.
  Again by a standard Chernoff and union bound argument this is in absolute value at most $O(n^{3/2}( \log n))$ with probability $1 - n^{-\omega(1)}$.
  In sum, when $i \neq j$, with probability at least $1 - n^{-\omega(1)}$, we get $|\cL \|x\|^2 x_i x_j | = O(1/n^2 \log n)$.
  After a union bound, the maximum over all $i,j$ is $O(1/n^2)$. This concludes \pref{eq:6}.

  In the $i = j$ case, since $\sum_{k} \iprod{w,A_k e_i \tensor A_k e_i}
  = \sum_{j,k} \iprod{e_j, A_k e_i}^2$ is a sum of $n^2$ independent square Gaussians,
  by a Bernstein inequality, $|\sum_k \iprod{w,A_k e_i \tensor A_k e_i} - n^2| \leq O(n \log^{1/2} n)$
  with probability $1 - n^{-\omega(1)}$.
  The same holds for the other tensored terms, and for $w^T AA^T e_{ii}$,
  so when $i = j$ we get that $|O(\lambda^2)\cL \|x\|^2 x_i^2 - 5| \leq O((\log^{1/2} n)/n)$
  with probability $1 - n^{-\omega(1)}$.
  Summing over all $i$, we find that $|O(\lambda^2)\cL \|x\|^4 - 5n| \leq O(\log^{1/2} n)$,
  so that $O(\lambda^2)|\cL \|x\|^2 x_i^2 - \tfrac{1}{n}\cL \|x\|^4| \leq O((\log^{1/2} n)/n)$
  with probability $1 - n^{-\omega(1)}$.
  A union bound over $i$ completes the argument.
\end{proof}

\begin{lemma}
  \label{lem:concentration-lb-4}
  Let $\bA$ be a random $4$-tensor with unit Gaussian entries.
  There is $\lambda^2 = O(n)$ so that
  when $\cL : \R[x]_4 \rightarrow \R$ is the linear operator whose matrix
  representation $M_{\cL}$ is given by
  $M_{\cL} \seteq \frac 1 {n^2 \lambda^2} \sum_{\pi \in \cS_4} A^\pi$,
  with probability $1 - O(n^{-50})$ the following events all occur for every $\pi \in \cS_4$.
\begin{align}
  \label{eq:4-1} -\lambda^2 \preceq \frac 1 2 (A^\pi + (A^\pi)^T)\\
  \label{eq:4-2}   \iprod {\bA, \sum_{\pi \in \cS_4} \bA^\pi } = \Omega(n^4) \\
  \label{eq:4-3}   \iprod{\symId, A^\pi } = O(\lambda^2 \sqrt{n}) \\
  \label{eq:4-4} n^2 \max_{i \neq j} |\cL \|x\|^2 x_i x_j | = O(1)\\
  \label{eq:4-5} n^{3/2} \max_i |\cL \|x\|^2 x_i^2| = O(1) \mper
\end{align}
\end{lemma}
\begin{proof}

  For \pref{eq:4-1}, we note that $\frac 1 2 A^\pi + (A^\pi)^T$ is an $n^2 \times n^2$ matrix with unit Gaussian entries.
  Thus, by \pref{lem:square-concentration}, we have $\frac 1 2 \|A^\pi + (A^\pi)^T \| \leq O(n) = O(\lambda)$.
  For \pref{eq:4-2} only syntactic changes are needed from the proof of \pref{eq:4}.
  For \pref{eq:4-3}, we observe that $\iprod{\symId, A^\pi}$ is a sum of $O(n^2)$ independent Gaussians, so is $O(n \log n) \le O(\lambda^2 \sqrt{n})$ with probability $1 - O(n^{-100})$.
  We turn finally to \pref{eq:4-4} and \pref{eq:4-5}.
  Unlike in the degree $3$ case, there is nothing special here about the diagonal so we will able to bound these cases together.
  Fix $i,j \leq n$.
  We expand $\cL \|x\|^2 x_i x_j$ as $\frac 1 {n^2 \lambda^2} \sum_{\pi \in \cS_4} w^TA^\pi (e_i \tensor e_j)$.
  The vector $A^\pi (e_i \tensor e_j)$ is a vector of unit Gaussians, so $w^T A^\pi (e_i \tensor e_j) = O(\sqrt n \log n)$
  with probability $1 - n^{-\omega(1)}$.
  Thus, also with probability $1 - n^{-\omega(1)}$, we get $n^2 \max_{i,j} |\cL \|x\|^2 x_i x_j | = O(1)$,
  which proves both \pref{eq:4-4} and \pref{eq:4-5}.
\end{proof}

\end{document}